\documentclass[preprint,12pt]{elsarticle}

\usepackage[utf8]{inputenc}
\usepackage[T1]{fontenc}

\biboptions{numbers,sort&compress}
\usepackage{amsmath}
\usepackage{amssymb}
\usepackage{amsfonts}
\usepackage{amsthm}
\usepackage{graphicx}
\usepackage{booktabs}
\usepackage{nicefrac}
\usepackage{xcolor}
\usepackage{mathtools}
\usepackage{algorithm}
\usepackage{algorithmic}
\usepackage{url}
\usepackage[hidelinks]{hyperref}

\usepackage{tabularx}
\usepackage{array}
\usepackage{ragged2e}

\newtheorem{theorem}{Theorem}
\newtheorem{proposition}[theorem]{Proposition}
\newtheorem{lemma}[theorem]{Lemma}
\newtheorem{corollary}[theorem]{Corollary}

\theoremstyle{remark}
\newtheorem{remark}[theorem]{Remark}
\theoremstyle{plain}
\newtheorem{assumption}[theorem]{Assumption}

\newcommand{\R}{\mathbb{R}}
\newcommand{\E}{\mathbb{E}}
\newcommand{\I}{\mathcal{I}}

\newcommand{\Hcal}{\mathcal{H}}
\newcommand{\Mcal}{\mathcal{M}}

\newcommand{\Scal}{\mathcal{S}}

\newcommand{\Pbb}{\mathbb{P}}
\newcommand{\bone}{\mathbf{1}}
\newcommand{\bx}{\mathbf{x}}
\newcommand{\by}{\mathbf{y}}

\newcommand{\bu}{\mathbf{u}}
\newcommand{\bv}{\mathbf{v}}
\newcommand{\Var}{\mathrm{Var}}
\newcommand{\Tr}{\mathrm{Tr}}

\newcommand{\opnorm}[1]{\left\|#1\right\|_{\mathrm{op}}}

\newif\ifsupplement

\journal{Artificial Intelligence}

\begin{document}

\begin{frontmatter}

\title{Learning Generated Controls under Fractured Geometry:
Projective Residualization and Variation-Allocation Frontiers}

\author[ustc]{Rui Wu\fnref{equal}}
\ead{wurui22@mail.ustc.edu.cn}
\author[ustc]{Zongyuan Chen\fnref{equal}}
\ead{chenzongyuan@mail.ustc.edu.cn}
\author[ustc]{Hong Xie\corref{corresponding}}
\ead{hongx87@ustc.edu.cn}
\author[ustc]{Defu Lian}
\ead{liandefu@ustc.edu.cn}
\author[ustc]{Enhong Chen}
\ead{cheneh@ustc.edu.cn}

\affiliation[ustc]{organization={School of Computer Science and Engineering, University of Science and Technology of China},
  city={Hefei},
  postcode={230027},
  country={China}}
\fntext[equal]{Rui Wu and Zongyuan Chen contributed equally.}
\cortext[corresponding]{Corresponding author.}

\begin{abstract}%
Many two-stage estimators assess the first-stage learner by prediction error,
even when the next stage uses its residual. In control-function instrumental
variables, that residual must preserve the latent control direction without
removing the treatment variation that identifies the structural response.
A scalar prediction score does not reveal how the learner allocates this
variation. Under piecewise-smooth graph geometry, interpolation can suppress the control,
whereas isotropic smoothing can leak systematic variation across boundaries.
We formulate this as a variation-allocation problem and introduce Adaptive
Anisotropic Instrumental Heat Flow (A-IHF). The method uses pilot treatment contrasts to adapt edge
conductance, takes the complement of a sparse graph resolvent as the generated
control, and selects candidates without consulting outcomes. For a linear
control-function regression, the generated control is identified only by its
span. Working in that projective geometry, we derive an exact finite-sample
fidelity--relevance frontier, spectral identities for remaining treatment
variation and coefficient distortion, and a lower bound for monotone
fixed-graph residual filters. A connected construction proves that adapting
conductance can remove the corresponding fixed-graph obstruction. In a
54-cell benchmark, the A-IHF family wins 32 cells; its guarded observational
variant lowers mean nonlinear response error by 8.3\%, with the largest gains
in fractured designs. Controlled rewiring explains when
the graph should be used, replaced by a fallback, or rejected. The resulting
lesson is task-specific: a first stage for generated controls should be judged
by control fidelity, downstream relevance, and graph compatibility together.
\end{abstract}

\begin{keyword}
graph signal processing, anisotropic diffusion, generated controls, instrumental variables, causal machine learning
\end{keyword}

\end{frontmatter}

\supplementfalse
\ifsupplement\else

\section{Introduction}
\label{sec:intro}

Many learning systems are built in stages. A fitted value, residual, or learned
representation produced upstream becomes an input to a later model. The two
stages are then coupled: a first-stage predictor may be accurate for its own
target yet discard variation that the downstream task needs. Control-function
instrumental variables (IV) make this issue particularly transparent. The
first stage estimates $g(Z)$ and passes the generated control
$V^\star=X-g(Z)$ to an outcome equation
\citep{rivers1988limited,newey1999nonparametric,blundell2003endogeneity,puli2020general}.
Treatment $X$ is therefore both the signal decomposed in the first stage and
the regressor whose variation, after conditioning on the generated control,
identifies the structural response
\citep{imbens1994identification,angrist1996identification,angrist2009mostly,wooldridge2010econometric}.

This coupling is easy to describe but easy to score incorrectly. On held-out
observations, prediction error and residual-recovery error are algebraically
the same. For any estimate $\hat g$,
\begin{equation}
    \hat V-V^\star
    =
    \{X-\hat g(Z)\}-\{X-g(Z)\}
    =
    g(Z)-\hat g(Z).
    \label{eq:prediction_residual_identity}
\end{equation}
Thus held-out prediction accuracy remains useful. What it does not report is
where the error occurs or whether the remaining treatment variation is usable
after the generated control has been partialled out. Same-sample
interpolation can absorb almost the entire residual and drive $\hat V$ toward
zero; cross-fitting mitigates that failure for inductive learners. Heavy
smoothing creates the opposite problem: $\hat V$ approaches $X$, so little
treatment variation remains. Between these endpoints lies a graph-specific
failure. If $g$ is smooth within regions but discontinuous between them, an
isotropic smoother can move systematic variation across a small set of
boundary edges while retaining a respectable average prediction score. The
resulting local contamination and the loss of downstream relevance are
different errors, and neither is summarized adequately by a single global
risk.

We ask whether a graph first stage can remove the locally smooth part of
treatment while preserving both the latent control direction and the
treatment variation needed by the outcome model. The problem has three
parts. The region boundaries are not observed; control fidelity and treatment
relevance need not improve together; and Euclidean residual error changes
under a rescaling that leaves a linear control-function coefficient exactly
unchanged. We refer to this as \emph{variation allocation under fractured
geometry}. Interpolation and cross-boundary smoothing fail for different
reasons, but each assigns variation to the wrong side of the decomposition.
Proposition \ref{prop:ntk} formalizes the interpolation endpoint. Sections
\ref{sec:theory} and \ref{sec:asymptotic_theory} study the remaining two
questions: what a graph residual preserves and when that preservation is
enough for the second stage.

\begin{figure}[t]
\centering
\includegraphics[width=0.94\linewidth]{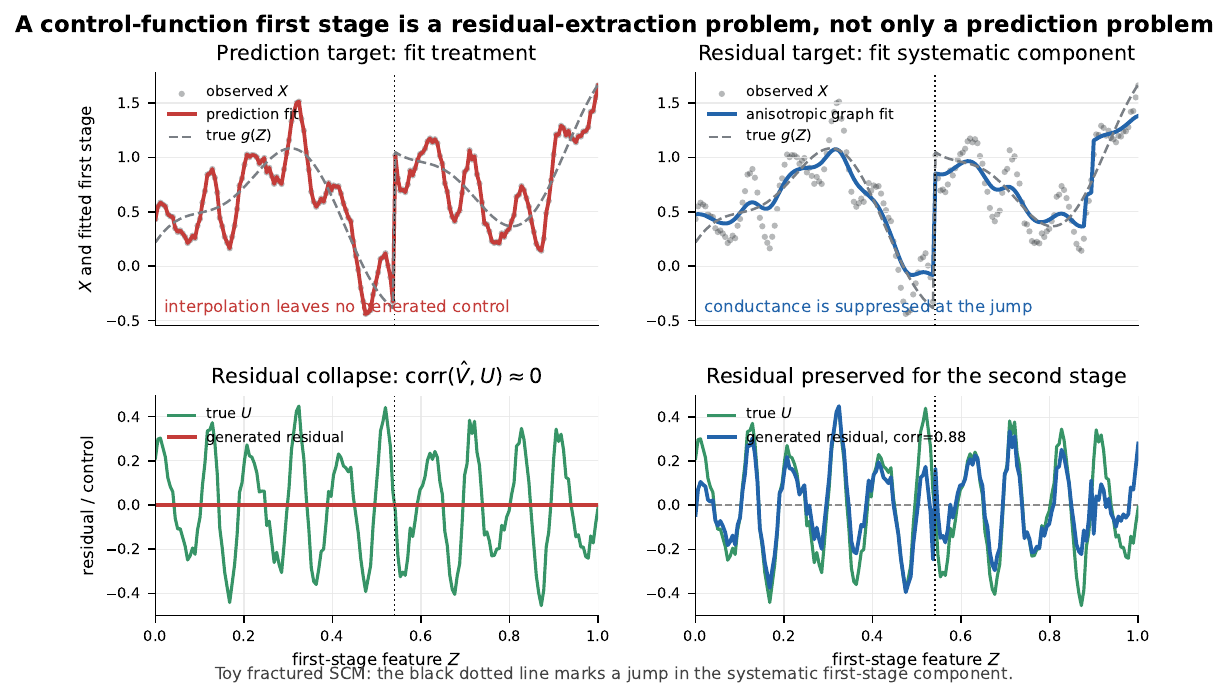}
\caption{\textbf{Variation allocation under fractured geometry.} Same-sample interpolation over-absorbs the generated residual, whereas isotropic smoothing reallocates the systematic component across a jump. A boundary-adaptive graph resolvent suppresses cross-boundary conductance while preserving a nondegenerate residual and downstream treatment variation. The figure is illustrative; Proposition \ref{prop:ntk} treats the interpolation endpoint and Section \ref{sec:theory} gives the graph-to-control-function certificate.}
\label{fig:toy_residual_extraction}
\end{figure}

Adaptive Anisotropic Instrumental Heat Flow (A-IHF) is designed around that
decomposition. It treats treatment as a signal on a graph of first-stage
features. A pilot resolvent estimates edgewise treatment contrasts; these
contrasts reduce conductance on edges for which pooling is poorly supported.
A second resolvent assigns the within-region component to the fitted first
stage, and its complement becomes the generated control. Outcomes are not
used in this construction. Candidate filters are rejected when they leave too
little residualized treatment variation, shatter the graph, or preserve local
treatment contrasts no better than unrelated pairs. The remaining candidates
are ranked by an observational score. Once the candidate path is fixed, the
algorithm consists of sparse symmetric linear solves.

The individual ingredients are not presented as inventions in isolation.
Anisotropic diffusion has long adapted conductance to observed contrasts
\citep{perona1990scale,weickert1998anisotropic}; graph trend filtering uses a
nonsmooth penalty to preserve jumps \citep{wang2015trend}; and first-stage
cross-fitting protects against interpolation. Those methods optimize a
reconstructed signal or an orthogonal score. A control function requires a
different object: the complement of the reconstructed signal becomes a
regressor, and the treatment must retain variation after that regressor is
removed. A-IHF's methodological addition is the coupling of boundary-adaptive
graph smoothing, complementary residual extraction, downstream relevance,
and an observable decision to use, replace, or reject the graph.

The analysis concerns additive control-function designs in which the
systematic first-stage component is piecewise smooth on an informative feature
graph. The setting is favorable when neighborhoods are reliable within
regions, boundary edges are sparse, the endogenous residual is not dominated
by low graph frequencies, and treatment retains variation after conditioning
on the generated control. A latent low-dimensional geometry is one way to
obtain such a graph, but it is not required by the algorithm. Within this
scope, the paper makes three claims, each tied to a result that can be checked.

\begin{itemize}
    \item \textbf{A generated-control operator with an observable use
    condition.} Existing graph denoisers decide how to reconstruct a signal;
    they do not determine whether the complementary residual is a valid input
    to a control-function regression. A-IHF adds that missing decision. Its
    selector enforces treatment relevance and tests whether local graph edges
    carry smaller treatment contrasts than unrelated pairs. The rewiring
    experiment makes this claim falsifiable: as graph compatibility is
    destroyed, the procedure must move from graph use to fallback and then to
    abstention.
    \item \textbf{Exact finite-sample geometry rather than a generic theory
    claim.} We do not claim that finite-sample control-function analysis is
    itself new. The new result is a projective characterization tailored to a
    generated residual. It gives an exact fidelity--relevance frontier,
    graph-spectral identities for relevance and coefficient distortion, and
    an isotonic-projection lower bound for monotone fixed-graph residual
    filters. A connected three-node construction supplies the complementary
    possibility result: treatment-dependent conductance can eliminate a
    zero-relevance obstruction that every shift-invariant filter on the
    original graph shares.
    \item \textbf{An empirical test of the proposed mechanism, including its
    reversals.} A matched 54-cell study measures projective fidelity,
    relevance, coefficient distortion, and frontier slack on the same runs as
    the downstream benchmark. The A-IHF family wins 32 cells, and its guarded
    observational variant lowers mean nonlinear response error by 8.3\%, with
    the clearest gains in fractured designs. The negative controls are equally
    informative: smooth and weak-instrument designs favor other learners, and
    graph spectral smoothing has lower aggregate linear coefficient error.
    Exact distortion correlates $0.998$ with realized linear coefficient
    error, while a separate scale intervention distinguishes the exact linear
    invariance from the scale sensitivity of finite neural training. The rate
    result retains the usual within-region graph-smoothing rate and states the
    unresolved adaptive-conductance term explicitly.
\end{itemize}

\section{Related Work}
\label{sec:related}

\paragraph{Control functions and nonparametric IV}
Control-function methods express endogeneity through a generated residual
from a first-stage equation. This construction appears in limited-information
models \citep{rivers1988limited}, nonparametric triangular systems
\citep{newey1999nonparametric,blundell2003endogeneity}, nonlinear
causal-effect models \citep{guo2016control}, and learned general control
functions \citep{puli2020general}. Nonparametric IV instead recovers a
structural function from conditional moments, typically through an ill-posed
inverse problem
\citep{newey2003instrumental,hall2005nonparametric,darolles2011nonparametric}.
Recent finite-sample results for Hausman- and Wald-type exogeneity tests
address a different question \citep{hahn2024finite}. We study recovery of the
generated residual when the first-stage signal has graph-aligned
discontinuities.

\paragraph{Machine learning IV and orthogonal estimation}
Modern IV estimators use kernels, neural networks, adversarial objectives, and
minimax conditional moments
\citep{hartford2017deep,singh2019kernel,muandet2020dual,bennett2019deep,dikkala2020minimax}.
Orthogonal procedures limit first-order sensitivity to nuisance estimation
through sample splitting and locally insensitive scores
\citep{chernozhukov2018double,syrgkanis2019machine}. A-IHF supplies a
structured generated control; it can be paired with an orthogonal second
stage under the conditions in Remark \ref{rem:orthogonal_inference}. The
nonlinear regressions used in our benchmark are point estimators, not an
inferential claim.

\paragraph{Graph regularization and diffusion}
The construction draws on manifold regularization, diffusion maps, and graph
signal processing
\citep{belkin2003laplacian,belkin2006manifold,coifman2006diffusion,vonluxburg2007tutorial,shuman2013emerging}.
Laplacian regression has minimax guarantees on suitable neighborhood graphs
\citep{green2021minimax}; graph trend filtering preserves piecewise structure
through an $\ell_1$ graph-difference penalty \citep{wang2015trend}; and graph
filters admit explicit stability bounds \citep{kenlay2021interpretable}.
Local graph operators also have established geometric limits
\citep{hein2005graphs,belkin2007convergence,vonluxburg2008consistency,ting2010analysis,garciatrillos2018variational}.
Our downstream target changes the design criterion: the smoother estimates
$g$, its complement becomes a regressor, and the pair must preserve treatment
relevance.

\begingroup
\setlength{\emergencystretch}{2em}
\paragraph{Regression geometry and graph-filter expressivity}
Partial-regression and partial-$R^2$ identities give a geometric description
of coefficient sensitivity \citep{cinelli2020making}. Graph uncertainty
principles describe vertex--spectral localization \citep{agaskar2013spectral},
while repeated Laplacian eigenvalues restrict polynomial spectral models
\citep{lu2024improving}. Weighted isotonic projection is classical
\citep{ayer1955empirical,barlow1972statistical}. Here these tools expose a
projective fidelity--relevance frontier and a fixed-graph restriction that
changes when conductance depends on treatment.
\par
\endgroup

\paragraph{Positioning}
Table \ref{tab:positioning} distinguishes four questions that are easily
conflated: how a method allocates variation, how it avoids empirical residual
collapse, whether it models unknown graph boundaries, and whether it connects
first-stage error to a downstream estimand. The comparison is not a ranking.
Cross-fitting addresses interpolation, graph trend filtering preserves
boundaries, and orthogonal IV offers stronger inference when its nuisance
conditions hold. None of these results implies the generated-control rule
studied here. The open combination is more specific: the complement of a
boundary-adaptive graph smoother must recover a control direction, preserve
treatment relevance, and expose an observable condition under which the graph
should not be used.

\begin{table}[!htbp]
\centering
\caption{\textbf{Positioning by the variation-allocation problem.} ``Boundary handling'' refers to unknown discontinuities of the first-stage systematic component on a feature graph. The table distinguishes methodological objectives rather than claiming uniform dominance.}
\label{tab:positioning}
\begingroup
\small
\setlength{\tabcolsep}{4.5pt}
\renewcommand{\arraystretch}{1.05}

\begin{tabularx}{\textwidth}{
@{}
>{\RaggedRight\arraybackslash}p{0.24\textwidth}
>{\RaggedRight\arraybackslash}X
>{\RaggedRight\arraybackslash}X
@{}
}
\toprule
\multicolumn{3}{@{}l}{\textit{Panel A: Variation-allocation mechanism and residual-collapse protection}}\\
\addlinespace[2pt]
\textbf{Method family}
& \textbf{Variation-allocation mechanism}
& \textbf{Residual-collapse protection}\\
\midrule

Classical/non-parametric CF
\citep{newey1999nonparametric,blundell2003endogeneity}
& Estimate $g(Z)$; retain $X-\hat g(Z)$
& Regularization within the assumed first-stage class\\

General control functions
\citep{puli2020general}
& Learn a control representation jointly
& Variational regularization and held-out training\\

Orthogonal ML-IV
\citep{chernozhukov2018double,syrgkanis2019machine}
& Allocate variation through score-specific nuisances
& Cross-fitting and orthogonal scores\\

Graph Laplacian regression
\citep{green2021minimax}
& Allocate low graph frequencies to the fitted signal
& Quadratic regularization\\

Graph trend filtering
\citep{wang2015trend}
& Allocate variation through sparse graph differences
& $\ell_1$ graph-difference regularization\\

\midrule
\textbf{A-IHF (this work)}
& \textbf{Within-region smoothing plus complementary residual}
& \textbf{Finite path, relevance floor, graph-compatibility guardrails}\\
\bottomrule
\end{tabularx}

\vspace{0.45em}

\begin{tabularx}{\textwidth}{
@{}
>{\RaggedRight\arraybackslash}p{0.24\textwidth}
>{\RaggedRight\arraybackslash}X
>{\RaggedRight\arraybackslash}X
@{}
}
\toprule
\multicolumn{3}{@{}l}{\textit{Panel B: Boundary handling and downstream guarantee}}\\
\addlinespace[2pt]
\textbf{Method family}
& \textbf{Boundary handling}
& \textbf{Downstream guarantee}\\
\midrule

Classical/non-parametric CF
\citep{newey1999nonparametric,blundell2003endogeneity}
& Kernel/series smoothness is usually global
& CF identification, consistency, and inference\\

General control functions
\citep{puli2020general}
& Implicit in the neural representation
& Identification and representation objectives\\

Orthogonal ML-IV
\citep{chernozhukov2018double,syrgkanis2019machine}
& Inherited from the nuisance learner
& Root-$n$ inference under product-rate conditions\\

Graph Laplacian regression
\citep{green2021minimax}
& No explicit unknown-boundary mechanism
& Prediction/estimation rates on neighborhood graphs\\

Graph trend filtering
\citep{wang2015trend}
& Explicit edge-preserving penalty
& Graph-signal estimation guarantees\\

\midrule
\textbf{A-IHF (this work)}
& \textbf{Pilot-guided treatment-dependent conductance}
& \textbf{Leakage--attenuation--noise--relevance certificate}\\
\bottomrule
\end{tabularx}

\endgroup
\end{table}

\section{Problem Formulation}
\label{sec:problem}

We consider data $\{(z_i,x_i,y_i)\}_{i=1}^n$ generated by
\begin{equation}
    \begin{aligned}
    X &= g(Z) + V^\star, \qquad V^\star=U+\eta_X,\\
    Y &= f_0(X) + \gamma U + \epsilon_Y ,
    \end{aligned}
    \label{eq:scm}
\end{equation}
Here $Z\in\R^{d_Z}$ is the first-stage feature vector, $X\in\R$ is treatment,
$V^\star=X-g(Z)$ is the first-stage residual, $U$ is the endogenous component
that enters the outcome, and $\eta_X$ is first-stage noise. In the simplest
case, $Z$ consists only of excluded instruments. When exogenous controls $W$
are present, $Z$ may include both instruments and controls, and the outcome
equation becomes $Y=f_0(X,W)+\gamma U+\epsilon_Y$. We absorb these first-stage
features into $Z$ and write $f_0(X)$ to avoid carrying $W$ through the
notation. Exclusion and relevance continue to concern the excluded-instrument
component.

If $\eta_X=0$, the first-stage residual is the outcome-relevant control:
$V^\star=U$. Our synthetic designs allow a small nonzero $\eta_X$ and generate
the nonlinear outcome using $U$; graph diagnostics can therefore be evaluated
against either $U$ or $V^\star$. A-IHF estimates only the first stage in
\eqref{eq:scm}. Its causal interpretation still rests on the usual IV
assumptions, and weak instruments remain a separate source of instability
\citep{staiger1997instrumental,stock2002survey}.

Throughout, $\|\cdot\|_2$ denotes the Euclidean norm on finite sample vectors. The scaled quantity $n^{-1/2}\|\cdot\|_2$ is the empirical $L_2$ norm. We write unscaled norms in finite-sample algebra and scaled norms in rate statements.

Let $\bx=(x_1,\ldots,x_n)^\top$, $\bv^\star=(v_1^\star,\ldots,v_n^\star)^\top$, $\bu=(u_1,\ldots,u_n)^\top$, and $\mathbf{g}=(g(z_1),\ldots,g(z_n))^\top$. A control-function first stage estimates $\mathbf{g}$ and forms
\begin{equation}
    \hat{\bv} = \bx - \hat{\mathbf{g}} .
\end{equation}
The outcome model conditions on $\hat{\bv}$ when estimating the structural
response. From the graph perspective, the target is $V^\star$; from the
control-function perspective, it is $U$. The two differ by $\eta_X$.

For a linear smoother $S$, the observed treatment is divided into the fitted
component $S\bx$ and the generated control $(I-S)\bx$. A useful division keeps
systematic signal out of the complement, avoids absorbing residual signal into
the fit, and preserves variation in $X$ after projection on the generated
control. No single in-sample fit statistic summarizes all three. Under
fractured geometry, the decision is also edge specific: pooling is desirable
within a smooth region but harmful across a jump.

\subsection{Selection Information Sets}

To distinguish implementable tuning from oracle evaluation, we record the
information available at each stage:
\begin{equation}
    \begin{aligned}
    \I_1
    &=\sigma\{(z_i,x_i):1\le i\le n\},\\
    \I_{\mathrm{obs}}
    &=\sigma\{(z_i,x_i,y_i):1\le i\le n\}.
    \end{aligned}
\end{equation}
A-IHF selects its first-stage parameters using $\I_1$; the outcome model may
use all of $\I_{\mathrm{obs}}$. Simulations additionally reveal
$\bv^\star$, $U$, $\eta_X$, and $f_0$, which we collect in
$\I_{\mathrm{oracle}}$.

We use four labels throughout:
\begin{itemize}
    \item \emph{fixed}: parameters are specified before evaluation;
    \item \emph{observational}: first-stage parameters are selected using $\I_1$, with $Y$ used only in the second stage or in evaluation;
    \item \emph{graph-admissibility-filtered observational}: observational selection after applying first-stage graph diagnostics, still using only $\I_1$;
    \item \emph{oracle}: parameters are selected using hidden simulator quantities.
\end{itemize}
The labels specify distinct estimators. Rows marked ``oracle'' are included
only as reference benchmarks.

\subsection{Prediction and Residual Extraction}

By \eqref{eq:prediction_residual_identity}, out-of-sample prediction
consistency implies out-of-sample residual consistency. The issue considered
here is more specific: if the generated control is formed on the same
observations used to train an interpolation-oriented learner, its empirical
variation can collapse. Proposition \ref{prop:ntk} records this endpoint in
the neural-tangent-kernel regime \citep{jacot2018neural}; spectral bias
provides complementary intuition about which modes disappear first
\citep{rahaman2019spectral}. The proposition does not argue against flexible
first stages, for which cross-fitting is the appropriate remedy. It identifies
the degenerate endpoint that a same-sample residualizer must avoid.

\begin{proposition}[Residual collapse under kernel gradient flow]
\label{prop:ntk}
Let $f_t$ denote the prediction of an over-parameterized neural network trained on squared loss for the first-stage regression $X$ on $Z$. In the neural tangent kernel limit, suppose the empirical kernel matrix $\Sigma$ is positive definite with eigenpairs $(\mu_j,\phi_j)$, $\mu_j>0$. If $\hat{\bv}_{NN}(t)=\bx-f_t(Z)$, then
\begin{equation}
    \langle \hat{\bv}_{NN}(t),\phi_j\rangle
    =
    e^{-\mu_j t}
    \langle \hat{\bv}_{NN}(0),\phi_j\rangle .
\end{equation}
Thus every positive-kernel empirical mode of the generated residual is exponentially damped, and $\hat{\bv}_{NN}(t)\to 0$ as $t\to\infty$.
\end{proposition}

The proof appears in \ref{app:ntk_collapse}. In our experiments,
inductive supervised baselines use out-of-fold predictions, which prevents a
training interpolator from mechanically zeroing its own residuals. A-IHF is
instead a transductive smoother on the observed graph. Its GCV term discourages
the identity-smoother endpoint, and its relevance floor excludes controls that
are nearly collinear with treatment. We report this protocol difference
throughout the empirical section.

\section{A-IHF: Task-Directed Graph Residualization}
\label{sec:method}

A-IHF has three coupled components. A feature graph represents first-stage
neighborhoods; treatment-guided conductance estimates where pooling is
appropriate; and a pair of complementary resolvents produces both
$\hat{\mathbf g}$ and the generated control. Adaptive conductance distinguishes
the method from isotropic graph ridge. Outcome-free tuning, a relevance floor,
and compatibility checks for connectivity, local degree, and treatment
contrast complete the first-stage rule.

\subsection{Graph Construction and Pilot Diffusion}

We begin with a graph on the first-stage features and regard treatment as a
signal on its vertices. Neighbors within a smooth region should exchange
information, whereas an edge that crosses a large jump should carry little
conductance. Given $z_1,\ldots,z_n$, we construct a symmetric
$K$-nearest-neighbor graph with nonnegative affinity matrix $A$. Distances are
converted to radial-basis weights using the median nonzero edge distance
before symmetrization. This construction follows standard practice in
manifold learning, spectral methods, and graph signal processing
\citep{belkin2003laplacian,coifman2006diffusion,vonluxburg2007tutorial,shuman2013emerging}.

For any symmetric weight matrix $W$, let $D_W$ be the diagonal degree matrix and let
\begin{equation}
    L(W) = \frac{D_W-W}{\bar d_W}, \qquad
    \bar d_W = \frac{1}{n}\Tr(D_W).
    \label{eq:scaled_laplacian}
\end{equation}
This scaled combinatorial Laplacian is symmetric positive semidefinite and has the constant vector in its nullspace on each connected component. The scaling keeps the diffusion parameter comparable across graphs with different mean degrees. The resolvent form is the graph analogue of Tikhonov-type smoothing \citep{tikhonov1977solutions,wahba1990spline}.

A pilot smoother is computed as
\begin{equation}
    \tilde{\bx} = (I+\tau L(A))^{-1}\bx .
    \label{eq:pilot}
\end{equation}
The pilot serves only to identify edges whose treatment contrast is unusually
large.

\subsection{Anisotropic Conductance}

For each edge $(i,j)$ in the graph, define
\begin{equation}
    C_{ij}
    =
    \exp\left\{
      -\frac{(\tilde x_i-\tilde x_j)^2}{\gamma}
    \right\}.
    \label{eq:conductance}
\end{equation}
To match the implementation also in the degenerate case, let
$\gamma_{\mathrm{raw}}$ be the $p$-th percentile of squared pilot differences
larger than $10^{-12}$ on graph edges. If no such difference exists, set
$\gamma_{\mathrm{raw}}=1$; otherwise use the empirical percentile. The
conductance scale is
$\gamma=\max\{\gamma_{\mathrm{raw}},10^{-12}\}$. The final weights are
\begin{equation}
    W_{ij}=A_{ij}C_{ij},
    \label{eq:weights}
\end{equation}
with an optional cutoff for numerical sparsity. As in anisotropic diffusion
\citep{perona1990scale,weickert1998anisotropic}, small pilot differences
preserve conductance and large differences reduce it.

\subsection{Resolvent Residual}

For a parameter tuple $h=(K,\tau,p,\lambda)$, let $W_h$ be the anisotropic weight matrix and define
\begin{equation}
    S_h = (I+\lambda L(W_h))^{-1}.
    \label{eq:smoother}
\end{equation}
The fitted first-stage systematic component and the generated control are
\begin{equation}
    \hat{\mathbf{g}}_h = S_h\bx, \qquad
    \hat{\bv}_h = (I-S_h)\bx .
    \label{eq:generated_control}
\end{equation}
Both quantities are obtained from sparse linear solves; the first stage uses
no gradient descent.

\subsection{Observational First-Stage Selection}
\label{sec:obs_selection}

A-IHF can be used with fixed defaults or with an observational rule over a finite candidate family $\Hcal_n$. The rule uses only $(Z,X)$. For each $h\in\Hcal_n$, let
\begin{equation}
    Q_h=[\bone,\hat{\bv}_h],
    \qquad
    M_h=I-Q_h(Q_h^\top Q_h)^\dagger Q_h^\top,
\end{equation}
and define the residualized treatment variation
\begin{equation}
    \kappa_n(h)=\frac{1}{n}\bx^\top M_h\bx .
    \label{eq:kappa}
\end{equation}
The observational score is
\begin{equation}
    \widehat Q_{\mathrm{obs}}(h)
    =
    \frac{\|(I-S_h)\bx\|_2^2/n}
    {\{1-\Tr(S_h)/n\}^2}
    +
    \alpha
    \frac{(S_h\bx)^\top L(W_h)(S_h\bx)}
    {\|\bx\|_2^2/n+\varepsilon}.
    \label{eq:qobs}
\end{equation}
The first term is graph generalized cross-validation
\citep{craven1979smoothing,golub1979generalized}, motivated by the
fixed-smoother identity in Proposition \ref{prop:gcv_identity}. The second
penalizes roughness in the fitted systematic component. Stochastic probes
provide an efficient estimate of the trace
\citep{hutchinson1989stochastic}. Within the chosen candidate family, the
score balances residual size, effective degrees of freedom, and fitted-signal
smoothness. Classical GCV assumes a fixed linear smoother and homoskedastic
noise; neither assumption is automatic here because the graph is adaptive and
the residual may be heteroskedastic or graph correlated. We examine this
limitation in the correlated-residual experiments.

For graph-admissibility filtering, let $\ell_n(h)$ be the largest post-cut
connected-component fraction of the final graph induced by $W_h$, and let
$d_{\min,n}(h)$ be its minimum final weighted degree after the numerical
conductance cutoff. Connectivity does not detect a densely but incorrectly
rewired graph, so we also compare local treatment contrast with its global
pairwise scale. For the base affinity $A_h$, define
\begin{equation}
    a_n(h)
    =
    \frac{\sum_{i<j}(A_h)_{ij}(x_i-x_j)^2}
    {2s_{X,n}^2\sum_{i<j}(A_h)_{ij}},
    \qquad
    s_{X,n}^2=\frac{1}{n}\|\bx-\bar x\bone\|_2^2 .
    \label{eq:edge_compatibility}
\end{equation}
The denominator equals the empirical squared contrast of an independently
drawn pair. Hence $a_n(h)\le1$ requires neighboring vertices to be at least as
similar in treatment as a generic pair, without consulting $Y$. Given thresholds
$\omega_n\in[0,1]$, $d_{0,n}\ge0$, and $a_{0,n}>0$, define
\begin{align}
    \Hcal_n^\kappa
    &=
    \left\{
    h\in\Hcal_n:
    \kappa_n(h)\ge \frac{c_\kappa}{n}
    \|\bx-\bar x\bone\|_2^2
    \right\}, \\
    \Hcal_n^{\mathrm{adm}}
    &=
    \left\{
    h\in\Hcal_n^\kappa:
    \ell_n(h)\ge \omega_n,
    \right. \nonumber\\[-2pt]
    &\hspace{2.7em}\left.
    d_{\min,n}(h)\ge d_{0,n},\quad
    a_n(h)\le a_{0,n}
    \right\}.
    \label{eq:admissible_family}
\end{align}
The unfiltered observational rule selects over $\Hcal_n^\kappa$,
\begin{equation}
    \hat h_{\mathrm{obs}}
    \in
    \arg\min_{h\in\Hcal_n^\kappa}
    \widehat Q_{\mathrm{obs}}(h).
    \label{eq:obs_rule}
\end{equation}
The graph-admissibility-filtered observational rule selects
\begin{equation}
    \hat h_{\mathrm{gobs}}
    \in
    \arg\min_{h\in\Hcal_n^{\mathrm{adm}}}
    \widehat Q_{\mathrm{obs}}(h).
    \label{eq:gobs_rule}
\end{equation}
We fix $\alpha\ge0$, $\varepsilon>0$, $c_\kappa>0$, $\omega_n$,
$d_{0,n}$, and $a_{0,n}$ as fixed before evaluation. These constants have
scale meanings rather than oracle meanings. The term $\varepsilon$ is a
numerical stabilizer. The relevance floor $c_\kappa$ requires a minimum
fraction of centered treatment variation to remain after projection on the
generated control. The connectivity thresholds require the final graph to
retain a macroscopic connected region and nondegenerate local degrees, while
\eqref{eq:edge_compatibility} rejects a connected but globally incompatible
representation. The roughness weight $\alpha$ fixes the scale of the
graph-energy penalty within the pre-specified family $\Hcal_n$; it is not
selected using $Y$, $\bv^\star$, or the structural response.

The candidate family and the filters jointly define the procedure. The
largest-component screen plays a role similar to a giant-component diagnostic
in random-graph and percolation theory
\citep{bollobas2001random,grimmett1999percolation}: a selected graph should
retain a macroscopic region rather than shatter into isolated fragments.
Section \ref{sec:diagnostics} shows that very sparse candidates and sharp cuts
can produce exactly this failure. We therefore regard $\Hcal_n$ and
$\Hcal_n^{\mathrm{adm}}$ as prespecified admissible families, not as
unrestricted hyperparameter searches.

Equation \eqref{eq:qobs} defines a finite-path selection rule. Extending it to
a minimax statement would require a continuum of graph scales and uniform
risk control for adaptive smoothers. Section \ref{sec:obs_theory} instead
provides a split-sample calibration inequality for a fixed finite family.

\begin{algorithm}[ht]
\caption{Adaptive Anisotropic Instrumental Heat Flow}
\label{alg:aihf}
\begin{algorithmic}[1]
\STATE \textbf{Input:} first-stage features $z_1,\ldots,z_n$, treatment vector $\bx$, and parameters $(K,\tau,p,\lambda)$.
\STATE Construct a symmetric $K$-nearest-neighbor affinity matrix $A$.
\STATE Form $L(A)$ and compute $\tilde{\bx}=(I+\tau L(A))^{-1}\bx$.
\STATE Let $\gamma_{\mathrm{raw}}$ be the $p$-th percentile of edge differences $(\tilde x_i-\tilde x_j)^2>10^{-12}$; if this set is empty, set $\gamma_{\mathrm{raw}}=1$; set $\gamma=\max\{\gamma_{\mathrm{raw}},10^{-12}\}$.
\STATE For graph edges, define $C_{ij}=\exp\{-(\tilde x_i-\tilde x_j)^2/\gamma\}$ and $W_{ij}=A_{ij}C_{ij}$; set $W_{ij}=0$ otherwise.
\STATE Form $S=(I+\lambda L(W))^{-1}$.
\STATE \textbf{Output:} $\hat{\mathbf g}=S\bx$ and $\hat{\bv}=(I-S)\bx$.
\end{algorithmic}
\end{algorithm}

The observational versions apply Algorithm \ref{alg:aihf} to each $h\in\Hcal_n$ and select $h$ by either \eqref{eq:obs_rule} or the graph-admissibility-filtered rule \eqref{eq:gobs_rule}.

\begin{figure}[ht]
\centering
\includegraphics[width=0.92\linewidth]{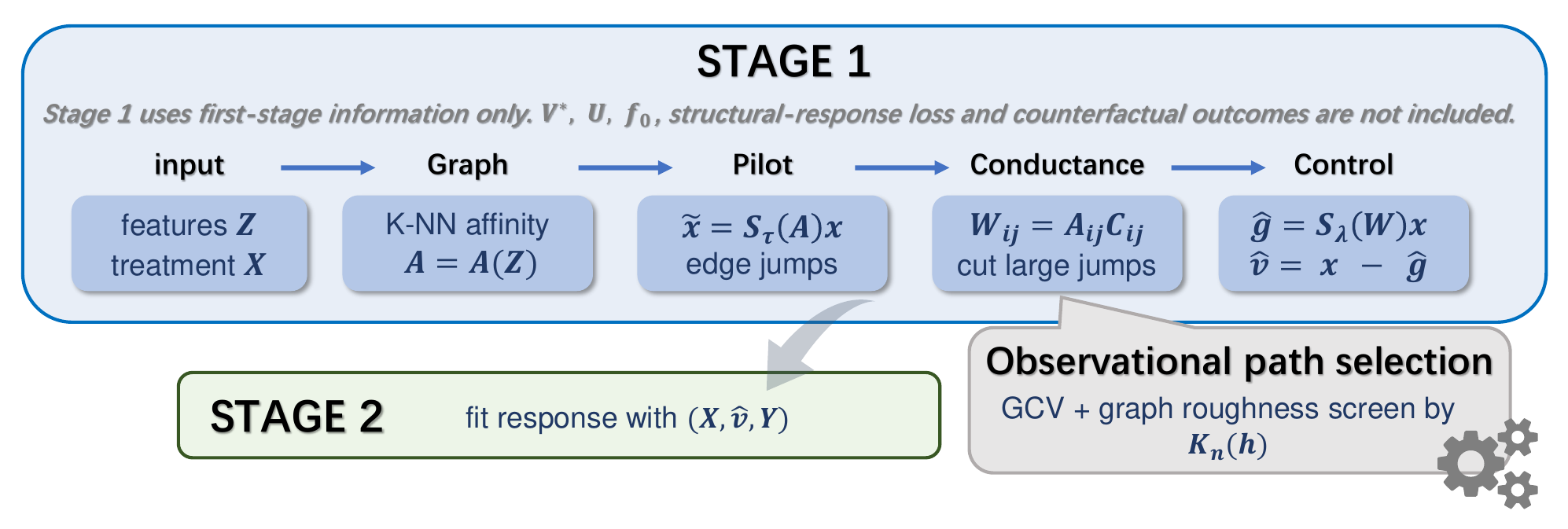}
\caption{\textbf{A-IHF mechanism.} The first stage constructs a graph on first-stage features, computes a pilot diffusion, attenuates conductance across large pilot treatment jumps, and forms a generated control by a final graph resolvent. The observational tuning rules use only $(Z,X)$, relevance, and graph-admissibility diagnostics; hidden simulator quantities are not part of Stage 1 selection.}
\label{fig:aihf_mechanism}
\end{figure}

\section{A Graph-to-Control-Function Certificate}
\label{sec:theory}

The graph operation matters downstream through the control it leaves behind.
Conditional on a realized smoother, generated-control error separates into
systematic signal left in the residual and residual signal absorbed by the
fit. The coefficient bound also accounts for the difference between
$V^\star$ and $U$, and scales the total error by the treatment variation left
after conditioning on the generated control. These finite-sample statements
require no graph limit. Section \ref{sec:asymptotic_theory} later gives
geometric conditions under which their terms vanish.

\subsection{Standing Assumptions and Targets}

\begin{assumption}[Additive first stage]
\label{ass:additive}
On the observed sample,
\begin{equation}
    \bx=\mathbf{g}+\bv^\star,
    \qquad
    \bv^\star=\bu+\boldsymbol{\eta}_X,
\end{equation}
where $\mathbf{g}=(g(z_1),\ldots,g(z_n))^\top$ is the first-stage systematic treatment component, $\bv^\star$ is the first-stage residual, $\bu$ is the outcome-relevant endogenous control, and $\boldsymbol{\eta}_X$ is first-stage noise. The graph residual error is evaluated against $\bv^\star$. The second-stage perturbation depends on the distance from $\hat{\bv}$ to $\bu$.
\end{assumption}

\begin{assumption}[Linear control-function second stage]
\label{ass:second_stage}
For the theoretical perturbation bound,
\begin{equation}
    y_i=\beta_0 x_i+\gamma_0 u_i+\epsilon_i,
    \qquad
    \E[\epsilon_i\mid x_i,u_i]=0 .
\end{equation}
This holds for $i=1,\ldots,n$.
Equivalently, $\by=\beta_0\bx+\gamma_0\bu+\boldsymbol{\epsilon}$ in vector notation.
This condition is used only for the finite-sample coefficient bound. The nonlinear response experiments are empirical evaluations of the same generated controls.
\end{assumption}

All finite-sample bounds condition on the realized smoother $S_h$. Assumption
\ref{ass:geometry}, introduced in Section \ref{sec:concrete_admissibility},
later supplies sufficient conditions for residual recovery. Although the
individual linear-smoother identities are elementary, their combination
locates misallocated variation and traces its effect into the
control-function coefficient. Proofs omitted from this section are collected
in \ref{app:auxiliary_proofs}.

\subsection{Fixed-Smoother Selection Identities}

\begin{proposition}[Fixed-smoother risk identity]
\label{prop:gcv_identity}
Assume $\bx=\mathbf{g}+\boldsymbol{\eta}$, where
$\E[\boldsymbol{\eta}\mid Z]=0$ and
$\E[\boldsymbol{\eta}\boldsymbol{\eta}^\top\mid Z]=\sigma^2 I$.
Let $S$ be a symmetric smoother treated as fixed conditional on $Z$. Then
\begin{equation}
    \E\left[\|(I-S)\bx\|_2^2\mid Z\right]
    =
    \|(I-S)\mathbf{g}\|_2^2
    +
    \sigma^2 \Tr\{(I-S)^2\}.
\end{equation}
\end{proposition}

Proposition \ref{prop:gcv_identity} supplies the idealized rationale for
ranking smoothers by residual size and effective degrees of freedom. Equation
\eqref{eq:qobs} adapts that rationale to a finite family of data-dependent
graphs, although IV residuals need not satisfy the proposition's scalar
conditional-variance assumption.

\subsection{Oracle Selection Is Not Observational Selection}

\begin{proposition}[Information-set separation]
\label{prop:info}
Let $P$ and $P'$ be two mechanisms that induce the same law for the observed sample $(Z,X,Y)$. Suppose that, under a coupling with the same observed sample, their oracle selectors over $\Hcal_n$ satisfy
\begin{equation}
    \Pbb\{h_P^\star\neq h_{P'}^\star\}>0,
\end{equation}
where $h_P^\star$ and $h_{P'}^\star$ may depend on hidden residual targets. Then no $\I_{\mathrm{obs}}$-measurable map can equal both oracle selectors with probability one.
\end{proposition}

The distinction is consequential in the experiments: oracle rows may use the
hidden residual target, whereas observational first-stage rules must be
functions of $\I_1$.

\subsection{Residual Error Decomposition}

\begin{proposition}[Exact graph residual decomposition]
\label{prop:decomp}
Under Assumption \ref{ass:additive}, let $\hat{\bv}_h=(I-S_h)\bx$. Then
\begin{equation}
    \hat{\bv}_h-\bv^\star
    =
    \underbrace{(I-S_h)\mathbf{g}}_{\text{structural leakage}}
    -
    \underbrace{S_h\bv^\star}_{\text{residual attenuation}} .
    \label{eq:decomp}
\end{equation}
\end{proposition}

The two terms in \eqref{eq:decomp} pull in different directions. Smoothing
should make $(I-S_h)\mathbf{g}$ small without making
$S_h\bv^\star$ large. Discontinuities in $\mathbf{g}$ make the first term
problematic for an isotropic graph; attenuating edges across large pilot jumps
is intended to reduce that leakage. By contrast, low-frequency structure in
$\bv^\star$ enlarges the second term even on a correctly specified graph. We
refer to this latter case as spectral leakage.

\begin{proposition}[Spectral leakage for correlated residuals]
\label{prop:leakage}
Let $S=(I+\lambda L)^{-1}$ for a fixed graph Laplacian $L$ with eigenpairs $(\mu_j,\phi_j)$. If $\E[\bv^\star(\bv^\star)^\top]=\Sigma_V$, then
\begin{equation}
    \E\|S\bv^\star\|_2^2
    =
    \sum_{j=1}^n
    \frac{\phi_j^\top \Sigma_V \phi_j}
    {(1+\lambda\mu_j)^2}.
    \label{eq:leakage}
\end{equation}
\end{proposition}

Low-frequency residual energy receives little attenuation in
\eqref{eq:leakage}, because the corresponding denominator is close to one.
The smoother cannot distinguish this component from other low-frequency graph
structure.

\subsection{Generated-Control Coefficient Error}

Let $\hat\beta(\hat{\bv})$ be the least-squares coefficient on $\bx$ when regressing $\by$ on $(\bx,\hat{\bv})$ and an intercept.

\begin{proposition}[Projective generated-control decomposition]
\label{prop:beta}
Under Assumption \ref{ass:second_stage}, let $\hat{\bv}$ be any generated control.
Let $M_{\hat{\bv}}$ be the residual-maker after projecting onto $[\bone,\hat{\bv}]$, and suppose
\begin{equation}
    \kappa_n(\hat{\bv})=\frac{1}{n}\bx^\top M_{\hat{\bv}}\bx>0.
\end{equation}
Then
\begin{equation}
    \hat\beta(\hat{\bv})-\beta_0
    =
    \gamma_0
    \frac{\bx^\top M_{\hat{\bv}}\bu}
    {\bx^\top M_{\hat{\bv}}\bx}
    +
    \frac{\bx^\top M_{\hat{\bv}}\boldsymbol{\epsilon}}
    {\bx^\top M_{\hat{\bv}}\bx}.
    \label{eq:beta_decomposition}
\end{equation}
Consequently,
\begin{equation}
    \begin{aligned}
    \left|
        \hat\beta(\hat{\bv})-\beta_0
        -
        \frac{\bx^\top M_{\hat{\bv}}\boldsymbol{\epsilon}}
        {\bx^\top M_{\hat{\bv}}\bx}
    \right|
    &\le
    \frac{|\gamma_0|\|M_{\hat{\bv}}\bu\|_2}
    {\sqrt{n\kappa_n(\hat{\bv})}}\\
    &\le
    \frac{|\gamma_0|\|\hat{\bv}-\bu\|_2}
    {\sqrt{n\kappa_n(\hat{\bv})}} .
    \end{aligned}
    \label{eq:beta_bound}
\end{equation}
\end{proposition}

The direct residual norm obscures an important invariance. Replacing
$\hat{\bv}$ by $c\hat{\bv}$ for any $c\ne0$ changes neither
$M_{\hat{\bv}}$ nor the second-stage coefficient. Thus, when
$\hat{\bv}=c\bu$, the projective error $\|M_{\hat{\bv}}\bu\|_2$ and the
associated coefficient distortion are both zero, even though
$\|\hat{\bv}-\bu\|_2$ may be arbitrarily large. Direct error is still useful
for nonlinear second stages and as a sufficient condition for rates, but it
does not measure the intrinsic discrepancy in a linear control-function
regression.

\begin{theorem}[Projective fidelity--relevance frontier]
\label{thm:projective_frontier}
Let $H_n=I-n^{-1}\bone\bone^\top$,
$\tilde{\bx}=H_n\bx$, and $\tilde{\bu}=H_n\bu$.  Suppose
$\tilde{\bx}$ and $\tilde{\bu}$ are nonzero and the centered sample space has
dimension at least two.  For a nonconstant generated control $\hat{\bv}$,
define
\begin{equation}
    q_n(\hat{\bv})
    =
    \frac{n\kappa_n(\hat{\bv})}{\|\tilde{\bx}\|_2^2},
    \qquad
    p_n(\hat{\bv})
    =
    \frac{\|M_{\hat{\bv}}\bu\|_2^2}{\|\tilde{\bu}\|_2^2},
\end{equation}
and let
\begin{equation}
    \rho_n
    =
    \frac{|\tilde{\bx}^\top\tilde{\bu}|}
    {\|\tilde{\bx}\|_2\|\tilde{\bu}\|_2}.
\end{equation}
Then, for every required relevance $q_0\in[0,1]$,
\begin{equation}
    \inf_{\hat{\bv}:\ q_n(\hat{\bv})\ge q_0}
    p_n(\hat{\bv})
    =
    \left[
        \rho_n\sqrt{q_0}
        -
        \sqrt{1-\rho_n^2}\sqrt{1-q_0}
    \right]_+^2.
    \label{eq:projective_frontier}
\end{equation}
\end{theorem}

Theorem \ref{thm:projective_frontier} characterizes the feasible region over
all one-dimensional generated controls in the observed sample. The true
control direction attains $q_n(\bu)=1-\rho_n^2$ with zero projective error.
Any relevance requirement below that value is therefore compatible with
perfect fidelity; a larger requirement necessarily introduces positive
projective error. Restricting the control to a graph-filter family may raise
the attainable frontier further. \ref{app:projective_proofs} proves
the result with a two-dimensional angle construction.

\begin{theorem}[Exact spectral variation allocation]
\label{thm:spectral_allocation}
Let $L$ be a symmetric graph Laplacian with an orthonormal eigenbasis
$\{(\mu_j,\phi_j)\}_{j=1}^n$, and let
$R=r(L)$ be a symmetric residual filter with $r(0)=0$.  Set
$\hat{\bv}=R\tilde{\bx}$,
$x_j=\phi_j^\top\tilde{\bx}$,
$u_j=\phi_j^\top\tilde{\bu}$,
$w_j=x_j^2$, and $t_j=u_j/x_j$ on
$\mathcal A=\{j:x_j\ne0\}$.  Write $r_j=r(\mu_j)$ and
$C_r=\sum_{j\in\mathcal A}w_jr_j^2>0$.  Then
\begin{equation}
    n\kappa_n(\hat{\bv})
    =
    \frac{
        \displaystyle\sum_{\substack{i<j\\i,j\in\mathcal A}}
        w_iw_j(r_i-r_j)^2
    }{C_r}.
    \label{eq:spectral_relevance_identity}
\end{equation}
Equivalently, for $\pi_j=w_j/\sum_{k\in\mathcal A}w_k$,
\begin{equation}
    q_n(\hat{\bv})
    =
    \frac{\Var_\pi(r)}{\E_\pi[r^2]}.
    \label{eq:spectral_relevance_variance}
\end{equation}
Under Assumption \ref{ass:second_stage}, if
$\kappa_n(\hat{\bv})>0$, then the generated-control part of the coefficient
error is exactly
\begin{equation}
    \hat\beta(\hat{\bv})-\beta_0-
    \mathcal E_{\epsilon,\hat{\bv}}
    =
    \gamma_0
    \frac{
        \displaystyle\sum_{\substack{i<j\\i,j\in\mathcal A}}
        w_iw_j(r_j-r_i)(t_ir_j-t_jr_i)
    }{
        \displaystyle\sum_{\substack{i<j\\i,j\in\mathcal A}}
        w_iw_j(r_i-r_j)^2
    },
    \label{eq:spectral_bias_identity}
\end{equation}
where
$\mathcal E_{\epsilon,\hat{\bv}}
=\bx^\top M_{\hat{\bv}}\boldsymbol{\epsilon}/
(\bx^\top M_{\hat{\bv}}\bx)$.
\end{theorem}

Equation \eqref{eq:spectral_relevance_variance} gives a simple interpretation:
downstream relevance comes from variation in filter gains across
treatment-active graph frequencies. A common gain on the active spectrum
makes $\hat{\bv}$ proportional to $\tilde{\bx}$ and leaves zero residualized
treatment variation. The distortion identity in
\eqref{eq:spectral_bias_identity} is correspondingly more specific than a
global residual-MSE bound. It compares the relative spectral shape of the
true control, $t_j=u_j/x_j$, with the residual response $r_j$.

\begin{proposition}[Monotone graph-filter barrier]
\label{prop:monotone_barrier}
Order the active graph frequencies and require equal gains within a repeated
eigenspace.  Let $\mathcal K$ be the cone of nonnegative nondecreasing
residual gains, with weighted norm
$\|a\|_w^2=\sum_{j\in\mathcal A}w_ja_j^2$, and let
$\Pi_{\mathcal K}$ be weighted projection onto this cone.  If
$\bu_\perp=\sum_{j\notin\mathcal A}u_j\phi_j$, then
\begin{equation}
    \begin{aligned}
    \inf_{r\in\mathcal K\setminus\{0\}}
    \|M_{R\tilde{\bx}}\bu\|_2^2
    &=
    \|\bu_\perp\|_2^2+\|t\|_w^2\\
    &\quad-
    \max\{
        \|\Pi_{\mathcal K}t\|_w^2,
        \|\Pi_{\mathcal K}(-t)\|_w^2
    \}.
    \end{aligned}
    \label{eq:monotone_barrier}
\end{equation}
The right side is therefore a lower bound for every complementary graph
resolvent $R=I-(I+\lambda L)^{-1}$.
\end{proposition}

The projection in Proposition \ref{prop:monotone_barrier} is the classical
weighted isotonic regression problem
\citep{ayer1955empirical,barlow1972statistical}. Here it quantifies a
downstream expressivity limit. The bound vanishes exactly when, up to sign,
the active part of the true control lies in the monotone response cone and
$\bu_\perp=0$. Any treatment-inactive component of the control remains
irreducible.
\ref{app:projective_proofs} proves the result and gives an explicit
three-node conductance construction that changes this barrier from
unidentified to zero projective distortion.

\begin{proposition}[Conductance escape from a repeated-eigenvalue obstruction]
\label{prop:conductance_escape}
Let
\begin{equation}
    \phi_g=\frac{1}{\sqrt6}(1,-2,1)^\top,\qquad
    \phi_u=\frac{1}{\sqrt2}(1,0,-1)^\top,
\end{equation}
and set $\tilde{\bx}=a\phi_g+b\phi_u$ and
$\tilde{\bu}=b\phi_u$, where $ab\ne0$.  On the unit-weight complete
three-node graph, every shift-invariant filter with $r(0)=0$ that produces a
nonzero residual gives $\hat{\bv}\propto\tilde{\bx}$ and hence
$\kappa_n(\hat{\bv})=0$.  If conductance adaptation deletes the two edges
incident to node 2 and retains edge $(1,3)$, then, for every $\lambda>0$,
\begin{equation}
    \begin{aligned}
    \{I-(I+\lambda L_1)^{-1}\}\tilde{\bx}
    &=
    \frac{3\lambda}{1+3\lambda}\tilde{\bu},\\
    n\kappa_n&=a^2,\qquad
    \|M_{\hat{\bv}}\bu\|_2=0,
    \end{aligned}
    \label{eq:conductance_escape}
\end{equation}
where
\begin{equation}
    L_1=
    \frac{3}{2}
    \begin{pmatrix}
        1&0&-1\\
        0&0&0\\
        -1&0&1
    \end{pmatrix}.
\end{equation}
Moreover, replacing each deleted edge by conductance $\delta>0$ gives a
connected graph whose projective error converges to zero and whose relevance
converges to $a^2/n$ as $\delta\downarrow0$.
\end{proposition}

The construction identifies an expressivity gap. On the original graph, a
scalar spectral response cannot separate the two active directions; after the
conductance change, it can. The zero-conductance case gives the exact
calculation, and the connected $\delta$-graph shows that the conclusion is
not an artifact of disconnection. The proposition concerns what the adaptive
operator can represent, not whether a data-driven conductance rule always
recovers the oracle cut.

The relevance screen addresses the opposite endpoint from interpolation. If
$S_h\approx0$, then $\hat{\bv}_h\approx\bx$, so projection on the generated
control removes nearly all treatment variation and $\kappa_n(h)$ is small.
When $S_h\approx I$, by contrast, interpolation makes
$\hat{\bv}_h\approx0$. GCV and the relevance screen are therefore needed for
different degeneracies.

\begin{corollary}[Graph-to-control-function perturbation]
\label{cor:graph_causal}
Under Assumptions \ref{ass:additive} and \ref{ass:second_stage}, let $\hat{\bv}_h=(I-S_h)\bx$ and suppose $\kappa_n(h)>0$. Define
\begin{equation}
    \mathcal E_{\epsilon,h}
    =
    \frac{\bx^\top M_{\hat{\bv}_h}\boldsymbol{\epsilon}}
    {\bx^\top M_{\hat{\bv}_h}\bx}.
    \label{eq:sampling_error_h}
\end{equation}
Then
\begin{equation}
    \begin{aligned}
    &|\hat\beta(\hat{\bv}_h)-\beta_0-\mathcal E_{\epsilon,h}|\\
    &\quad\le
    \frac{|\gamma_0|}{\sqrt{n\kappa_n(h)}}
    \|M_{\hat{\bv}_h}\bu\|_2\\
    &\quad\le
    \frac{|\gamma_0|}{\sqrt{n\kappa_n(h)}}
    \bigl\{
        \|(I-S_h)\mathbf{g}\|_2
        +\|S_h\bv^\star\|_2
        +\|\bv^\star-\bu\|_2
    \bigr\}.
    \end{aligned}
    \label{eq:graph_causal_bound}
\end{equation}
Equivalently,
\begin{equation}
    \begin{aligned}
    &|\hat\beta(\hat{\bv}_h)-\beta_0-\mathcal E_{\epsilon,h}|\\
    &\quad\le
    \frac{\sqrt{3}|\gamma_0|}{\sqrt{n\kappa_n(h)}}
    \bigl[
        \|(I-S_h)\mathbf{g}\|_2^2
        +\|S_h\bv^\star\|_2^2
        +\|\bv^\star-\bu\|_2^2
    \bigr]^{1/2}.
    \end{aligned}
\end{equation}
\end{corollary}

Corollary \ref{cor:graph_causal} carries the graph error into the outcome
regression. Its first bound has the intrinsic linear-CF form: projective
control error divided by remaining treatment variation. The looser
direct-norm bound separates three sources of error: leakage of $g(Z)$ into the
residual, attenuation of $V^\star$ by the smoother, and the mismatch between
$V^\star$ and $U$. The last term disappears in the noiseless control-function
model. This decomposition is diagnostically useful even when it is loose, as
it will be whenever the control direction is correct but its scale is not.

\subsection{Graph-Admissibility Certificate}
\label{sec:diagnostic_certificate}

For a realized smoother $S_h$, define
\begin{align}
    \mathsf{Leak}_n(h)
    &=
    n^{-1/2}\|(I-S_h)\mathbf{g}\|_2, \label{eq:cert_leak}\\
    \mathsf{Atten}_n(h)
    &=
    n^{-1/2}\|S_h\bv^\star\|_2, \label{eq:cert_atten}\\
    \mathsf{Noise}_n
    &=
    n^{-1/2}\|\bv^\star-\bu\|_2, \label{eq:cert_noise}\\
    \mathsf{Proj}_n(h)
    &=
    n^{-1/2}\|M_{\hat{\bv}_h}\bu\|_2, \label{eq:cert_proj}\\
    \mathsf{Rel}_n(h)
    &=
    \kappa_n(h). \label{eq:cert_rel}
\end{align}
The first quantity measures structural leakage. The second measures residual
attenuation. The third measures first-stage noise in the control target. The
fourth is scale-invariant projective control error. The fifth measures
residualized treatment variation after conditioning on the generated control.

\begin{proposition}[Graph-admissibility certificate]
\label{prop:diagnostic_certificate}
Under Assumptions \ref{ass:additive} and \ref{ass:second_stage}, it follows that if $\mathsf{Rel}_n(h)>0$, then
\begin{equation}
    n^{-1/2}\|\hat{\bv}_h-\bv^\star\|_2
    \le
    \mathsf{Leak}_n(h)+\mathsf{Atten}_n(h),
    \label{eq:cert_residual}
\end{equation}
and, with $\mathcal E_{\epsilon,h}$ defined in \eqref{eq:sampling_error_h},
\begin{equation}
    \begin{aligned}
    &|\hat\beta(\hat{\bv}_h)-\beta_0-\mathcal E_{\epsilon,h}|\\
    &\quad\le
    \frac{|\gamma_0|\mathsf{Proj}_n(h)}
    {\sqrt{\mathsf{Rel}_n(h)}}\\
    &\quad\le
    \frac{|\gamma_0|}{\sqrt{\mathsf{Rel}_n(h)}}
    \{\mathsf{Leak}_n(h)+\mathsf{Atten}_n(h)
      +\mathsf{Noise}_n\}.
    \end{aligned}
    \label{eq:cert_beta}
\end{equation}
\end{proposition}

In simulation, $\mathbf{g}$, $\bu$, and $\bv^\star$ are known, so
\eqref{eq:cert_leak}--\eqref{eq:cert_rel} can be reported directly.
$\mathsf{Proj}_n(h)$ is the primary linear-CF diagnostic; the other three
hidden terms provide a sufficient decomposition that also remains meaningful
for nonlinear second stages. In an observational application,
$\mathsf{Proj}_n(h)$, $\mathsf{Leak}_n(h)$, $\mathsf{Atten}_n(h)$, and
$\mathsf{Noise}_n$ are hidden. The implementable checks are graph roughness,
GCV, connectivity, local degree, edge compatibility, and
$\mathsf{Rel}_n(h)$. The experiments vary the corresponding failure modes in
turn: cross-boundary smoothing increases leakage, low-frequency residual
structure increases attenuation, first-stage noise separates $V^\star$ from
$U$, and weak instruments reduce relevance.

\fi

\ifsupplement

\section{Graph-Admissibility Rates and Inference}
\label{sec:asymptotic_details}

The asymptotic analysis begins with a reduction to leakage, attenuation,
noise mismatch, and relevance. We then derive sufficient conditions for that
reduction from a locally scaled graph on a latent piecewise-smooth geometry,
including the relation between the algorithmic parameter $\lambda_n$ and the
continuum smoothing scale.

\subsection{Consistency Under Graph-Admissible Sequences}
\label{sec:consistency}

The finite-sample certificate yields the following abstract reduction. We
state it for a generic sequence of graph smoothers so that the same argument
also covers alternative filters and approximate solves.

\begin{assumption}[Graph-admissible smoother sequence]
\label{ass:graph_admissible}
Consider a triangular array satisfying Assumptions \ref{ass:additive} and \ref{ass:second_stage}. Let $S_n=S_{h_n}$ be a sequence of A-IHF smoothers, where $h_n$ may be fixed, observationally selected, or otherwise measurable with respect to the first-stage sample. Let
\begin{equation}
    \hat{\bv}_n=(I-S_n)\bx_n .
\end{equation}
There exist deterministic sequences $a_{g,n}\to0$, $a_{v,n}\to0$, and $a_{x,n}\ge0$, and a constant $\kappa_0>0$, such that
\begin{align}
    n^{-1/2}\|(I-S_n)\mathbf{g}_n\|_2 &= O_p(a_{g,n}), \label{eq:ag}\\
    n^{-1/2}\|S_n\bv_n^\star\|_2 &= O_p(a_{v,n}), \label{eq:av}\\
    n^{-1/2}\|\bv_n^\star-\bu_n\|_2 &= O_p(a_{x,n}), \label{eq:ax}\\
    \Pr\{\kappa_n(\hat{\bv}_n)\ge \kappa_0\} &\to 1, \label{eq:kappa_asymp}\\
    \left|
    \frac{\bx_n^\top M_{\hat{\bv}_n}\boldsymbol{\epsilon}_n}
    {\bx_n^\top M_{\hat{\bv}_n}\bx_n}
    \right| &= O_p(n^{-1/2}). \label{eq:sampling_asymp}
\end{align}
\end{assumption}

Conditions \eqref{eq:ag} and \eqref{eq:av} control the two sides of the
first-stage decomposition: the smoother must preserve the systematic
component while removing the part of $V^\star$ assigned to the fit. Condition
\eqref{eq:ax} measures the difference between the graph residual target and
the outcome-relevant control. The remaining two conditions ensure,
respectively, second-stage relevance and the usual sampling behavior of the
outcome noise after partialling out the generated control.

\begin{proposition}[Consistency reduction to graph admissibility]
\label{prop:consistency_reduction}
Under Assumption \ref{ass:graph_admissible},
\begin{equation}
    n^{-1/2}\|\hat{\bv}_n-\bv_n^\star\|_2
    =
    O_p(a_{g,n}+a_{v,n}).
    \label{eq:residual_consistency}
\end{equation}
If $a_{g,n}+a_{v,n}\to0$, then the generated control is root-mean-square consistent for $V^\star$. If also $a_{x,n}\to0$, it is root-mean-square consistent for $U$. Under the linear second-stage model in Assumption \ref{ass:second_stage},
\begin{equation}
    |\hat\beta(\hat{\bv}_n)-\beta_0|
    =
    O_p(a_{g,n}+a_{v,n}+a_{x,n})+O_p(n^{-1/2}).
    \label{eq:beta_consistency_rate}
\end{equation}
In particular, if $a_{x,n}\to0$, then $\hat\beta(\hat{\bv}_n)\xrightarrow{p}\beta_0$.
\end{proposition}

\begin{proof}
By Proposition \ref{prop:decomp},
\begin{equation}
    \hat{\bv}_n-\bv_n^\star
    =
    (I-S_n)\mathbf{g}_n-S_n\bv_n^\star .
\end{equation}
Taking norms and applying \eqref{eq:ag}--\eqref{eq:av} gives \eqref{eq:residual_consistency}. For the second stage, use the exact decomposition in the proof of Proposition \ref{prop:beta}:
\begin{equation}
    \hat\beta(\hat{\bv}_n)-\beta_0
    =
    \gamma_0
    \frac{\bx_n^\top M_{\hat{\bv}_n}(\bu_n-\hat{\bv}_n)}
    {\bx_n^\top M_{\hat{\bv}_n}\bx_n}
    +
    \frac{\bx_n^\top M_{\hat{\bv}_n}\boldsymbol{\epsilon}_n}
    {\bx_n^\top M_{\hat{\bv}_n}\bx_n}.
\end{equation}
On the event $\kappa_n(\hat{\bv}_n)\ge\kappa_0$,
\begin{align}
    \left|
    \frac{\bx_n^\top M_{\hat{\bv}_n}(\bu_n-\hat{\bv}_n)}
    {\bx_n^\top M_{\hat{\bv}_n}\bx_n}
    \right|
    &\le
    \frac{\|M_{\hat{\bv}_n}\bx_n\|_2\,
    \|\bu_n-\hat{\bv}_n\|_2}
    {n\kappa_n(\hat{\bv}_n)} \\
    &=
    \frac{n^{-1/2}\|\bu_n-\hat{\bv}_n\|_2}
    {\sqrt{\kappa_n(\hat{\bv}_n)}} \\
    &=
    O_p(a_{g,n}+a_{v,n}+a_{x,n}).
\end{align}
The sampling term is controlled by \eqref{eq:sampling_asymp}. This proves \eqref{eq:beta_consistency_rate}.
\end{proof}

The proposition makes the sources of inconsistency explicit. Cross-boundary
smoothing can violate \eqref{eq:ag}; low-frequency graph structure in
$V^\star$ can violate \eqref{eq:av}; persistent first-stage noise prevents
\eqref{eq:ax} from vanishing; and a control nearly collinear with $X$ violates
\eqref{eq:kappa_asymp}. These four mechanisms determine the organization of
the empirical study.

\subsection{Concrete Graph-Admissibility Conditions}
\label{sec:concrete_admissibility}

We now turn to sufficient geometric conditions for residual recovery.
Relevance, the mismatch between $V^\star$ and $U$, and outcome-noise sampling
remain separate requirements. Let $T_i\in\Mcal$ be a latent coordinate on a
compact $m$-dimensional Riemannian manifold whose volume density is bounded
above and below. The graph may use $Z_i$ directly or a representation
$\psi(Z_i)$, provided that the representation preserves local neighborhoods.
Learning such a representation is outside the present analysis.

\begin{assumption}[Latent geometry and piecewise smoothness]
\label{ass:geometry}
The variables $T_i$ are i.i.d. on $\Mcal$. There is a representation $\psi$ and constants $0<c_\psi<C_\psi<\infty$ such that, with probability tending to one,
\begin{equation}
    c_\psi d_\Mcal(T_i,T_j)
    \le
    \|\psi(Z_i)-\psi(Z_j)\|_2
    \le
    C_\psi d_\Mcal(T_i,T_j)
\end{equation}
for all pairs with $d_\Mcal(T_i,T_j)\le c_0$. For
$\varepsilon_n=(K_n/n)^{1/m}$, suppose there is a constant $c_E<\infty$
such that every local graph edge has intrinsic length at most
$c_E\varepsilon_n$ with probability tending to one. The manifold is
partitioned into finitely many closed regions
\begin{equation}
    \Mcal=\bigcup_{\ell=1}^R \Mcal_\ell,
\end{equation}
with piecewise $C^1$ boundaries. On each region, $g$ belongs to a standard
H\"older class of order $s\in(0,2]$. Define
$\alpha_s=\min\{s,1\}$. We assume the corresponding edge-increment bound
\begin{equation}
    |g(T_i)-g(T_j)|
    \le L_gd_\Mcal(T_i,T_j)^{\alpha_s}
    \quad\text{for }T_i,T_j\in\Mcal_\ell .
    \label{eq:within_holder_increment}
\end{equation}
For $s>1$, this is the Lipschitz consequence of bounded first derivatives,
not an $O(d_\Mcal^s)$ increment claim. Each regional restriction has a
one-sided trace $g_\ell^\partial$ on its boundary. Whenever
$\Mcal_\ell$ and $\Mcal_k$ share a boundary across which a local edge may
pass, suppose
\begin{equation}
    \inf_{t\in\partial\Mcal_\ell\cap\partial\Mcal_k}
    |g_\ell^\partial(t)-g_k^\partial(t)|
    \ge \Delta
\end{equation}
for a constant $\Delta>0$. It follows that the actual local cross-region
edge set $E_{\mathrm{jump}}$ satisfies, with probability tending to one,
\begin{equation}
    \inf_{(i,j)\in E_{\mathrm{jump}}}|g(T_i)-g(T_j)|
    \ge
    \Delta-2L_g(c_E\varepsilon_n)^{\alpha_s}.
    \label{eq:jump_gap}
\end{equation}
\end{assumption}

\begin{remark}[Geometry and included controls]
\label{rem:geometry_controls}
Assumption \ref{ass:geometry} is a condition on the metric representation used by the graph. It is not a claim that raw excluded instruments, discrete controls, and fixed effects jointly form a smooth Riemannian manifold. Discrete included controls can be handled by stratification or localization before graph construction. In applications that concatenate standardized controls with excluded instruments, the graph is an empirical first-stage feature graph. The geometric rate theory should then be read conditionally on a representation that preserves the relevant local neighborhoods. The four real-IV applications in Section \ref{sec:real_iv_applications} are interpreted in this empirical sense, not as evidence for the manifold assumption.
\end{remark}

Let $K=K_n$ and define the intrinsic graph scale
\begin{equation}
    \varepsilon_n=\left(\frac{K_n}{n}\right)^{1/m}.
\end{equation}
We use the nearest-neighbor regime
\begin{equation}
    \frac{K_n}{\log n}\to\infty,\qquad
    \frac{K_n\log n}{n}\to 0.
    \label{eq:knn_regime}
\end{equation}
Under standard concentration conditions for random geometric graphs, graph edges have intrinsic length $O_p(\varepsilon_n)$ away from boundary effects \citep{belkin2007convergence,hein2005graphs,singer2006graph,ting2010analysis,garciatrillos2018variational}. The first condition is the usual connectivity-side condition for local graphs. The second keeps neighborhoods local. After anisotropic weighting, the final graph may disconnect across jump boundaries by construction; the connectivity requirement applies to the separated within-region graph.

The scaled Laplacian in \eqref{eq:scaled_laplacian} is not divided by $\varepsilon_n^2$. On smooth functions supported away from region boundaries, the low-frequency scaling has the form
\begin{equation}
    L_n^0 f
    =
    c_L\varepsilon_n^2(-\Delta_{\Mcal} f)
    +e_{n,f},
    \qquad
    \|e_{n,f}\|_{L_2(P_n)}\le \zeta_n\|f\|_{\Scal},
    \label{eq:laplacian_scaling}
\end{equation}
for a graph-dependent constant $c_L>0$ and a smoothness norm $\|\cdot\|_{\Scal}$. Hence the final resolvent $(I+\lambda_n L_n)^{-1}$ corresponds to a continuum heat scale
\begin{equation}
    t_n\asymp \lambda_n\varepsilon_n^2 ,
    \qquad
    t_{p,n}\asymp \tau_n\varepsilon_n^2 ,
    \label{eq:lambda_to_t}
\end{equation}
for the final and pilot diffusions. Constants such as $c_L$ are absorbed into $t_n$. The rate statements below use $t_n$ as shorthand for the algorithmic scale $\lambda_n\varepsilon_n^2$.

\begin{assumption}[Pilot accuracy and threshold separation]
\label{ass:pilot_threshold}
Let $\tilde{\bx}$ be the pilot diffusion in \eqref{eq:pilot}. There is a sequence $r_{p,n}\to0$ such that
\begin{equation}
    \|\tilde{\bx}-\mathbf{g}\|_\infty\le r_{p,n}
    \label{eq:pilot_sup}
\end{equation}
with probability tending to one. Write
$\rho_{g,n}=L_g(c_E\varepsilon_n)^{\alpha_s}$. The conductance scale
$\gamma_n$ satisfies
\begin{equation}
    \frac{(\rho_{g,n}+2r_{p,n})^2}{\gamma_n}\to0,
    \qquad
    \frac{\gamma_n}{(\Delta-2\rho_{g,n}-2r_{p,n})^2}\to0 .
    \label{eq:threshold_separation}
\end{equation}
\end{assumption}

Assumption \ref{ass:pilot_threshold} is a sufficient condition for analyzing
jump separation by the pilot. It is neither an observational diagnostic nor a
claim about arbitrary noisy first stages.

\begin{lemma}[Jump separation by anisotropic conductance]
\label{lem:jump_separation}
Suppose Assumptions \ref{ass:geometry} and \ref{ass:pilot_threshold} hold. For
local graph edges contained in a single smooth region,
\begin{equation}
    \min_{(i,j)\in E_{\mathrm{in}}} C_{ij}\to 1
\end{equation}
in probability. For the actual local graph edges that cross a boundary with
the one-sided trace gap in Assumption \ref{ass:geometry},
\begin{equation}
    \max_{(i,j)\in E_{\mathrm{jump}}} C_{ij}\to 0
\end{equation}
in probability.
\end{lemma}

\begin{lemma}[Percentile threshold]
\label{lem:percentile_threshold}
Let $D_{ij}=(\tilde x_i-\tilde x_j)^2$ on graph edges and let
$\hat\gamma_n$ be the implemented scale: the empirical $\theta$-quantile,
$\theta=p/100$, of edge values exceeding $10^{-12}$, followed by the
$10^{-12}$ floor, with value one before flooring if the positive-difference
set is empty. Suppose there is a deterministic sequence $b_n>0$ and a
constant $c_\gamma<\infty$ such that, with probability tending to one,
\begin{align}
    \max_{(i,j)\in E_{\mathrm{in}}}D_{ij} &\le c_\gamma \hat\gamma_n, \label{eq:percentile_within}\\
    \hat\gamma_n/b_n &\to0, \label{eq:percentile_gap}\\
    \min_{(i,j)\in E_{\mathrm{jump}}}D_{ij} &\ge b_n . \label{eq:percentile_jump}
\end{align}
Then the percentile rule yields
\begin{equation}
    \min_{(i,j)\in E_{\mathrm{in}}} C_{ij}\ge e^{-c_\gamma}+o_p(1),
    \qquad
    \max_{(i,j)\in E_{\mathrm{jump}}} C_{ij}\to0 .
\end{equation}
\end{lemma}

Lemma \ref{lem:jump_separation} gives uniform within-region conductance for a
deterministic threshold. The percentile implementation needs the weaker
conclusion in Lemma \ref{lem:percentile_threshold}: within-region weights stay
comparable while cross-jump weights vanish. Both lemmas motivate comparison
with a separated graph. Neither, however, is sufficient for the operator-rate
condition $\eta_n=o(1)$, which must be imposed separately.

\begin{assumption}[Bridge condition for the separated graph]
\label{ass:smoothing_rate}
Let $B_n$ be the set of observations within intrinsic distance $c\varepsilon_n$ of a region boundary, and define
\begin{equation}
    \xi_n=\left(\frac{|B_n|}{n}\right)^{1/2}.
\end{equation}
Let $W_n^0$ be an oracle separated weight matrix that retains within-region
edges, including within-region edges incident to $B_n$, and deletes
cross-region edges. Outside $B_n$, its retained weights are comparable to the
original affinities. Behavior inside $B_n$ is accounted for by the explicit
boundary term below. Let
\begin{equation}
    L_n^0=L(W_n^0),\qquad
    S_n^0=(I+\lambda_n L_n^0)^{-1}.
\end{equation}
Let $t_n\asymp\lambda_n\varepsilon_n^2$ be the continuum diffusion scale induced by the algorithmic parameter $\lambda_n$, and let $\zeta_n$ be the fixed-graph approximation error for the separated graph. The bias term below is the heat-semigroup truncation bias associated with the graph-to-manifold scaling in \eqref{eq:laplacian_scaling}; $t_n$ is not an additional algorithmic input. Assume $\zeta_n\to0$ and
\begin{align}
    n^{-1/2}\|(I-S_n^0)\mathbf{g}_n\|_2
    &\le
    C\{t_n^{s/2}+\zeta_n+\xi_n\},
    \label{eq:oracle_graph_bias}\\
    \Tr\{(S_n^0)^2\}
    &\le
    C t_n^{-m/2}.
    \label{eq:effective_dimension}
\end{align}
Here $S_n$ is the A-IHF smoother after the anisotropic conductance step. The term $t_n^{s/2}$ is the heat-smoothing bias on an $m$-dimensional geometry. The term $\zeta_n$ records fixed-graph approximation. The term $\xi_n$ records the boundary band. For piecewise $C^1$ boundaries and bounded density, $\xi_n=O_p(\varepsilon_n^{1/2})$.
\end{assumption}

Graph-Laplacian and spectral approximations for fixed, nonadaptive local
graphs are well understood under geometric regularity
\citep{belkin2007convergence,vonluxburg2008consistency,garciatrillos2018variational}.
Assumption \ref{ass:smoothing_rate} separately imposes the heat-semigroup bias
and effective-dimension bounds used below; those bounds are not inferred from
the cited spectral results alone. Because A-IHF chooses conductance from the
data, our theorem stops short of a continuum limit for the percentile rule.
Instead, it summarizes the adaptive-graph discrepancy through $\eta_n$.

\begin{lemma}[Resolvent perturbation by leaked graph mass]
\label{lem:resolvent_leakage}
Let $S_n=(I+\lambda_n L_n)^{-1}$ and $S_n^0=(I+\lambda_n L_n^0)^{-1}$, where $L_n$ and $L_n^0$ are symmetric positive semidefinite. Define
\begin{equation}
    \eta_n=\lambda_n\opnorm{L_n-L_n^0}.
    \label{eq:eta_operator}
\end{equation}
If $n^{-1/2}\|\bx_n\|_2=O_p(1)$, then
\begin{equation}
    n^{-1/2}\|(S_n-S_n^0)\bx_n\|_2
    =
    O_p(\eta_n).
    \label{eq:resolvent_leakage_bound}
\end{equation}
If
\begin{equation}
    \begin{aligned}
    0<c_d&\le \bar d_n^0,\bar d_n\le C_d<\infty,\\
    \opnorm{L_n^0}\vee\opnorm{L_n}&\le C_L,\\
    \frac{|\bar d_n-\bar d_n^0|}{\bar d_n^0}&\le C\ell_n,
    \end{aligned}
    \label{eq:mean_degree_scaling}
\end{equation}
and
\begin{equation}
    \ell_n=
    \frac{1}{\bar d_n^0}
    \max_i
    \sum_j |W_{n,ij}-W^0_{n,ij}|,
    \label{eq:leaked_degree}
\end{equation}
then $\opnorm{L_n-L_n^0}\le C\ell_n$. Thus $\eta_n\le C\lambda_n\ell_n$.
\end{lemma}

Because $L(W)$ is invariant to multiplying all weights by a common positive
constant, the bounded-mean-degree display in
\eqref{eq:mean_degree_scaling} can be imposed after a common deterministic
rescaling of $W_n$ and $W_n^0$. It is not a bounded-$K_n$ assumption.

\begin{assumption}[Residual spectral regularity]
\label{ass:residual_spectral}
Conditional on the latent geometry and on the separated graph, $\bv_n^\star$ is mean zero and sub-Gaussian with covariance $\sigma_V^2 I_n$. More generally, the analysis uses the following operative condition:
\begin{equation}
    n^{-1/2}\|S_n^0\bv_n^\star\|_2
    =
    O_p\left(\sqrt{\frac{t_n^{-m/2}}{n}}\right).
    \label{eq:residual_spectral_rate}
\end{equation}
\end{assumption}

The isotropic sub-Gaussian case implies \eqref{eq:residual_spectral_rate}: conditionally on the separated graph,
\begin{equation}
    \E\{\|S_n^0\bv_n^\star\|_2^2\mid Z\}
    =
    \sigma_V^2\Tr\{(S_n^0)^2\}
    \le C\sigma_V^2 t_n^{-m/2}.
\end{equation}
Markov's inequality gives the stated order. Assumption
\ref{ass:residual_spectral} separates the graph-smooth systematic component
from the residual. Graph-correlated residuals may violate this condition;
Proposition \ref{prop:leakage} gives the corresponding finite-sample
expression.

With this distinction in place, the rate theorem is conditional on the
adaptive graph satisfying the perturbation control encoded by $\eta_n$.

\begin{theorem}[Residual recovery rate under piecewise graph smoothness]
\label{thm:aihf_rate}
Suppose Assumptions \ref{ass:additive} and \ref{ass:geometry}--\ref{ass:residual_spectral} hold, and suppose $n^{-1/2}\|\bx_n\|_2=O_p(1)$. Let $\eta_n$ be defined by \eqref{eq:eta_operator}, let $t_n\asymp\lambda_n\varepsilon_n^2$, and let
\begin{equation}
    r_n=
    t_n^{s/2}
    +\sqrt{\frac{t_n^{-m/2}}{n}}
    +\zeta_n
    +\xi_n
    +\eta_n .
    \label{eq:rn_rate}
\end{equation}
Then the A-IHF generated control satisfies
\begin{equation}
    n^{-1/2}\|\hat{\bv}_n-\bv_n^\star\|_2
    =
    O_p(r_n).
    \label{eq:explicit_residual_rate}
\end{equation}
If $t_n\to0$, $n t_n^{m/2}\to\infty$, $\zeta_n\to0$, $\xi_n\to0$, and $\eta_n\to0$, then $\hat{\bv}_n$ is root-mean-square consistent for $\bv_n^\star$. If $\zeta_n$, $\xi_n$, and $\eta_n$ are of smaller order and $t_n\asymp n^{-2/(2s+m)}$, equivalently
\begin{equation}
    \lambda_n
    \asymp
    n^{-2/(2s+m)}\varepsilon_n^{-2},
    \label{eq:lambda_optimal_scaling}
\end{equation}
then
\begin{equation}
    n^{-1/2}\|\hat{\bv}_n-\bv_n^\star\|_2
    =
    O_p\left(n^{-s/(2s+m)}\right).
    \label{eq:nonparametric_rate}
\end{equation}
\end{theorem}

The five terms in \eqref{eq:rn_rate} correspond to heat-smoothing bias,
effective residual dimension, fixed-graph approximation, the boundary band,
and cross-jump perturbation. When $s=1$ and the dimension lies in the regimes
studied by \citet{green2021minimax}, the leading exponent agrees with the
neighborhood-graph Laplacian rate. For general $s$, it follows from the bias
and effective-dimension bounds in Assumption \ref{ass:smoothing_rate}, not
from \citet{green2021minimax} alone. We make no rate-improvement claim.
Rather, the additional terms $\xi_n+\eta_n$ specify when adaptive edge
suppression preserves the within-region rate, and Corollary
\ref{cor:explicit_beta_rate} then propagates that rate to the
control-function coefficient. Accordingly, Theorem \ref{thm:aihf_rate} is a
reduction from an adaptive graph to a separated-graph approximation, not a
continuum-limit theorem for the percentile rule. Establishing
$\eta_n=o(1)$ for that fully adaptive rule under primitive noise conditions
remains open.

\begin{corollary}[Explicit second-stage rate]
\label{cor:explicit_beta_rate}
Under the conditions of Theorem \ref{thm:aihf_rate} and Assumption \ref{ass:second_stage}, suppose also that
\begin{equation}
    \Pr\{\kappa_n(\hat{\bv}_n)\ge \kappa_0\}\to1
\end{equation}
for some $\kappa_0>0$. Define
\begin{equation}
    \mathcal E_{\epsilon,n}
    =
    \frac{\bx_n^\top M_{\hat{\bv}_n}\boldsymbol{\epsilon}_n}
    {\bx_n^\top M_{\hat{\bv}_n}\bx_n}.
    \label{eq:sampling_error_n}
\end{equation}
Then
\begin{equation}
    \begin{aligned}
    |\hat\beta(\hat{\bv}_n)-\beta_0-\mathcal E_{\epsilon,n}|
    &=O_p(r_n+\delta_{x,n}),\\
    \delta_{x,n}
    &=n^{-1/2}\|\bv_n^\star-\bu_n\|_2 .
    \end{aligned}
    \label{eq:explicit_beta_centered_rate}
\end{equation}
If $\mathcal E_{\epsilon,n}=O_p(n^{-1/2})$, then
\begin{equation}
    |\hat\beta(\hat{\bv}_n)-\beta_0|
    =
    O_p(r_n+\delta_{x,n})+O_p(n^{-1/2}),
    \label{eq:explicit_beta_rate}
\end{equation}
\end{corollary}

\begin{corollary}[Root-n linear inference under negligible generated-control error]
\label{cor:nonorthogonal_rootn}
In the linear second-stage model of Assumption \ref{ass:second_stage}, suppose
\begin{equation}
    \|\hat{\bv}_n-\bu_n\|_2=o_p(1),
    \qquad
    \kappa_n(\hat{\bv}_n)\xrightarrow{p}\kappa_0>0 .
    \label{eq:negligible_generated_control}
\end{equation}
Suppose further that
\begin{equation}
    \frac{1}{\sqrt n}\bx_n^\top M_{\hat{\bv}_n}\boldsymbol{\epsilon}_n
    \rightsquigarrow N(0,\Omega),
    \qquad
    \frac{1}{n}\bx_n^\top M_{\hat{\bv}_n}\bx_n
    \xrightarrow{p} Q>0 .
    \label{eq:linear_sampling_clt}
\end{equation}
Then
\begin{equation}
    \sqrt n\{\hat\beta(\hat{\bv}_n)-\beta_0\}
    \rightsquigarrow
    N(0,\Omega/Q^2).
    \label{eq:nonorthogonal_rootn}
\end{equation}
\end{corollary}

\begin{proof}
Use the decomposition in the proof of Proposition \ref{prop:beta}. The generated-control term is bounded by
\begin{equation}
    \sqrt n
    \frac{|\gamma_0|}{\sqrt{n\kappa_n(\hat{\bv}_n)}}
    \|\hat{\bv}_n-\bu_n\|_2
    =
    O_p\!\left(\|\hat{\bv}_n-\bu_n\|_2\right)
    =
    o_p(1).
\end{equation}
The sampling term converges by \eqref{eq:linear_sampling_clt} and Slutsky's theorem.
\end{proof}

The requirement in \eqref{eq:negligible_generated_control} is substantially
stronger than the rate needed for point estimation. Theorem
\ref{thm:aihf_rate} controls the normalized distance to $\bv_n^\star$,
whereas the corollary assumes that the unnormalized distance to $\bu_n$
vanishes. Both graph residual error and first-stage noise mismatch must
therefore be negligible. Corollary \ref{cor:nonorthogonal_rootn} should be
read as a benchmark for this stringent regime; the orthogonal construction
below accommodates slower nuisance rates.

\subsection{Selection Calibration}
\label{sec:obs_theory}

Because pilot-smoothed treatment determines conductance,
\eqref{eq:qobs} is data adaptive. A clean finite-sample calibration is
available for an idealized split version: one subsample constructs a finite
family of smoother matrices on an evaluation sample, for example through a
prespecified out-of-sample graph extension. Conditional on that construction
split, $\{S_h:h\in\Hcal_n\}$ is fixed. The implemented A-IHF procedure is
transductive instead; it selects parameters on the observed graph and refits
the graph on the full sample. Proposition \ref{prop:selector_oracle} therefore
calibrates the split, fixed-smoother surrogate rather than the implemented
selector.

For fixed candidate smoothers, define the sample residual-recovery loss
\begin{equation}
    R_n(h)=\frac{1}{n}\|S_h\bx-\mathbf{g}\|_2^2
    =
    \frac{1}{n}\|\hat{\bv}_h-\bv^\star\|_2^2,
\end{equation}
where the equality follows because $S_h\bx-\mathbf{g}=-(\hat{\bv}_h-\bv^\star)$.
Define also the roughness penalty
\begin{equation}
    J_n(h)=
    \frac{(S_h\bx)^\top L(W_h)(S_h\bx)}
    {\|\bx\|_2^2/n+\varepsilon}.
\end{equation}
Since $L(W_h)$ is positive semidefinite and $\varepsilon>0$, $J_n(h)\ge0$.
Let $\Hcal_n^\kappa$ be the candidates satisfying the relevance screen in \eqref{eq:obs_rule}.

\begin{proposition}[Calibration inequality for the split observational selector]
\label{prop:selector_oracle}
Assume $\alpha\ge0$, $\varepsilon>0$, $\Hcal_n^\kappa$ is nonempty, $\Hcal_n$ is finite, each candidate smoother satisfies $0\preceq S_h\preceq I$, and
\begin{equation}
    d_h=\frac{1}{n}\Tr(S_h),
    \qquad
    1-d_h\ge c_{\mathrm{df}}>0
    \label{eq:df_lower_bound}
\end{equation}
for all $h\in\Hcal_n$. Suppose that, conditional on the graph-construction split,
\begin{equation}
    \sup_{h\in\Hcal_n^\kappa}
    \left|
    \widehat Q_{\mathrm{obs}}(h)
    -
    \{R_n(h)+\alpha J_n(h)+\sigma_V^2\}
    \right|
    \le
    \Delta_n
    \label{eq:selector_uniform}
\end{equation}
with probability at least $1-\delta_n$. Then the split selector satisfies
\begin{equation}
    R_n(\hat h_{\mathrm{obs}})
    \le
    \inf_{h\in\Hcal_n^\kappa}
    \{R_n(h)+\alpha J_n(h)\}
    +2\Delta_n
    \label{eq:selector_oracle}
\end{equation}
with probability at least $1-\delta_n$.

In addition, suppose $\bx=\mathbf{g}+\boldsymbol{\eta}$ on the evaluation split, the entries of $\boldsymbol{\eta}$ are conditionally independent, mean zero, sub-Gaussian with variance proxy bounded by $\sigma_\eta^2$, $\E[\boldsymbol{\eta}\boldsymbol{\eta}^\top\mid Z]=\sigma_V^2 I$, and $n^{-1}\|\mathbf{g}\|_2^2\le C_g$. Define
\begin{equation}
    b_{\mathrm{sel},n}=
    \sup_{h\in\Hcal_n^\kappa}
    \left|
    \E\!\left[
        \widehat Q_{\mathrm{obs}}(h)
        -
        \{R_n(h)+\alpha J_n(h)+\sigma_V^2\}
        \,\middle|\, Z
    \right]
    \right|.
    \label{eq:selector_bias}
\end{equation}
Then, for all $0<\delta<1$, there is a constant $C$ depending only on $c_{\mathrm{df}}$, $\sigma_\eta$, and $C_g$ such that
\begin{equation}
    \Delta_n=
    b_{\mathrm{sel},n}+
    C\left\{
    \sqrt{\frac{\log(2|\Hcal_n|/\delta)}{n}}
    +
    \frac{\log(2|\Hcal_n|/\delta)}{n}
    \right\}
    \label{eq:selector_delta_explicit}
\end{equation}
is valid with conditional probability at least $1-\delta$. If $\Tr(S_h)$ is
replaced by an estimator $\widehat{\Tr}(S_h)$ satisfying
$\sup_h|\widehat d_h-d_h|\le e_{\mathrm{tr},n}<c_{\mathrm{df}}/2$ and the
evaluation sample obeys $n^{-1}\|\bx\|_2^2\le C_x$, the same bound gains an
additive term $C C_x e_{\mathrm{tr},n}$. Under the preceding sub-Gaussian
conditions, $n^{-1}\|\bx\|_2^2=O_p(1)$, so the corresponding unconditional
trace contribution is $O_p(e_{\mathrm{tr},n})$.
\end{proposition}

For every A-IHF candidate, $0\preceq S_h\preceq I$ because $L(W_h)$ is
symmetric positive semidefinite. The term $b_{\mathrm{sel},n}$ measures the
gap between GCV and residual-recovery risk; adaptive graphs,
heteroskedasticity, and graph-correlated residuals may keep this gap from
vanishing. The proposition is therefore a finite-family calibration result.
It has an oracle interpretation only when both the selection-bias and
trace-estimation terms are small.

\begin{remark}[Asymptotic behavior of the selection-bias term]
\label{rem:selection_bias}
The idealized case in which $b_{\mathrm{sel},n}$ is expected to vanish is narrow but informative. If the candidate smoothers are fixed conditional on $Z$, the first-stage residual $V^\star$ is conditionally homoskedastic and mean independent of the graph features, and its conditional covariance has no systematic low-frequency alignment with the candidate graph eigenspaces, then the classical GCV identity approximates the residual-recovery risk up to sampling and trace-estimation errors. Under a finite or slowly growing candidate family and a uniform law of large numbers for the quadratic forms $\{n^{-1}\|(I-S_h)\bx\|_2^2:h\in\Hcal_n^\kappa\}$, this gives $b_{\mathrm{sel},n}=o_p(1)$. These conditions are deliberately stronger than the empirical setting. Their role is to identify what the observational score is trying to approximate; when heteroskedasticity or graph-correlated residual variation violates them, $b_{\mathrm{sel},n}$ becomes a selection-bias term rather than a vanishing remainder. The correlated-residual design is included for exactly this failure mode. Establishing $b_{\mathrm{sel},n}\to0$ for fully adaptive graph GCV remains a separate uniform risk problem.
\end{remark}

Thus the comparison is with the best admissible candidate for the penalized
residual-recovery target under the idealized split construction. We do not use
it as a guarantee for the transductive implementation, whose candidate family
and relevance screen remain part of the estimator definition.

\subsection{Orthogonal Score Implication}
\label{sec:orthogonal_inference}

Proposition \ref{prop:beta} controls point-estimation error for the OLS
control-function coefficient. Root-$n$ inference needs either a stronger
residual rate or an orthogonal score. A standard cross-fitted construction
solves
\begin{equation}
    \frac{1}{n}\sum_{i=1}^n
    \psi(W_i;\beta,\hat\eta_{-k(i)})=0,
    \label{eq:orthogonal_score}
\end{equation}
where $W_i=(Y_i,X_i,Z_i)$, the nuisance $\eta$ contains the generated control and second-stage regressions, and $k(i)$ denotes the fold not containing observation $i$.

\begin{remark}[Orthogonal-score normality under product-rate conditions]
\label{rem:orthogonal_inference}
Let parameter $\beta_0$ satisfy $\E[\psi(W;\beta_0,\eta_0)]=0$. Suppose the score is Neyman orthogonal at $\eta_0$, has nonsingular Jacobian $G=\partial_\beta \E[\psi(W;\beta,\eta_0)]|_{\beta=\beta_0}$, and satisfies a central limit theorem with variance
\begin{equation}
    \Omega=\E[\psi(W;\beta_0,\eta_0)\psi(W;\beta_0,\eta_0)^\top].
\end{equation}
If the cross-fitted nuisance estimates obey the usual product-rate condition, in particular if the A-IHF control enters with $L_2$ rate $\rho_n=o(n^{-1/4})$ relative to the outcome-relevant control and the remaining nuisance errors have compatible rates, then the solution $\tilde\beta$ of \eqref{eq:orthogonal_score} satisfies
\begin{equation}
    \sqrt n(\tilde\beta-\beta_0)
    \rightsquigarrow
    N\left(0,G^{-1}\Omega G^{-\top}\right).
    \label{eq:orthogonal_clt}
\end{equation}
A plug-in sandwich estimator using cross-fitted scores is consistent under the same conditions.
\end{remark}

Remark \ref{rem:orthogonal_inference} records how the generated-control rate
enters the standard orthogonal-score expansion
\citep{chernozhukov2018double}. It does not cover the neural response learner
used in our structural-response benchmark. An analogous nonparametric result
would need stability conditions on the second-stage learner---for example,
Lipschitz or entropy control---or an orthogonal nonlinear score with compatible
nuisance rates. We therefore treat the nonlinear benchmark as an empirical
point-estimation study. Formal error propagation is confined to the linear
control-function and orthogonal-score results, and the real-IV applications
use bootstrap intervals only to describe estimate stability.

\begin{remark}[Intrinsic dimension and orthogonal inference]
\label{rem:dimension_inference}
The rate condition $\rho_n=o(n^{-1/4})$ is restrictive. Under the rate in Theorem \ref{thm:aihf_rate}, ignoring fixed-graph, boundary, and leakage terms gives
\begin{equation}
    \rho_n\asymp n^{-s/(2s+m)} .
\end{equation}
The product-rate condition for root-$n$ orthogonal inference then requires
\begin{equation}
    \frac{s}{2s+m}>\frac{1}{4},
    \qquad\text{equivalently}\qquad
    m<2s .
    \label{eq:dimension_inference_condition}
\end{equation}
Since the graph-Laplacian rate stated here uses a second-order smoother with $s\le2$, this condition supports root-$n$ orthogonal inference only below intrinsic dimension four, unless additional structure or stronger nuisance estimation is available. In higher intrinsic dimensions, generated-control error can dominate the score expansion. Conventional score intervals may then undercover. The finite-sample undercoverage in \ref{sec:orthogonal_diag} is consistent with this rate restriction.
\end{remark}

\subsection{Approximate Resolvent Solves}
\label{sec:approx_solve}

All preceding statements use the exact resolvent
$S_h=(I+\lambda L(W_h))^{-1}$. On large graphs it may be replaced by an
iterative solve or a polynomial approximation. The resulting error enters the
certificate additively. We study conjugate gradients because the relevant
linear systems are symmetric positive definite
\citep{shewchuk1994introduction,saad2003iterative}.

\begin{proposition}[Perturbation from approximate graph solves]
\label{prop:approx_solve}
Let $\tilde S_h$ be a numerical approximation to $S_h$ and define
\begin{equation}
    \tilde{\bv}_h=(I-\tilde S_h)\bx .
\end{equation}
If
\begin{equation}
    n^{-1/2}\|(\tilde S_h-S_h)\bx\|_2\le \delta_n ,
    \label{eq:solver_delta}
\end{equation}
then
\begin{equation}
    n^{-1/2}\|\tilde{\bv}_h-\bv^\star\|_2
    \le
    n^{-1/2}\|\hat{\bv}_h-\bv^\star\|_2+\delta_n .
    \label{eq:approx_residual_bound}
\end{equation}
For a triangular array, suppose the exact smoother $S_n$ satisfies
\begin{align}
    n^{-1/2}\|(I-S_n)\mathbf{g}_n\|_2 &= O_p(a_{g,n}),\\
    n^{-1/2}\|S_n\bv_n^\star\|_2 &= O_p(a_{v,n}),\\
    n^{-1/2}\|\bv_n^\star-\bu_n\|_2 &= O_p(a_{x,n}).
\end{align}
Suppose also that
\begin{equation}
    n^{-1/2}\|(\tilde S_n-S_n)\bx_n\|_2=O_p(\delta_n),
\end{equation}
and that the relevance and sampling conditions in Assumption \ref{ass:graph_admissible} hold for $\tilde{\bv}_n=(I-\tilde S_n)\bx_n$. Then
\begin{equation}
    n^{-1/2}\|\tilde{\bv}_n-\bv_n^\star\|_2
    =
    O_p(a_{g,n}+a_{v,n}+\delta_n),
    \label{eq:approx_residual_rate}
\end{equation}
and
\begin{equation}
    |\hat\beta(\tilde{\bv}_n)-\beta_0|
    =
    O_p(a_{g,n}+a_{v,n}+a_{x,n}+\delta_n)+O_p(n^{-1/2}).
    \label{eq:approx_beta_bound}
\end{equation}
If only the centered coefficient is needed, the same statement holds with the sampling term subtracted as in Proposition \ref{prop:beta}.
\end{proposition}

\begin{proof}
Since
\begin{equation}
    \tilde{\bv}_h-\hat{\bv}_h
    =
    (S_h-\tilde S_h)\bx ,
\end{equation}
the triangle inequality gives \eqref{eq:approx_residual_bound}. The exact-smoother decomposition gives
\begin{equation}
    \hat{\bv}_n-\bv_n^\star
    =
    (I-S_n)\mathbf{g}_n-S_n\bv_n^\star .
\end{equation}
Combining this identity with the approximate-solve error gives \eqref{eq:approx_residual_rate}. The coefficient statement follows from Proposition \ref{prop:beta} applied to $\tilde{\bv}_n$, the relevance condition, and the sampling condition.
\end{proof}

Consequently, conjugate-gradient, preconditioned, and polynomial
implementations retain the reduction result whenever their empirical solve
error $\delta_n$ is small. The required control is in the norm of the
generated residual, rather than the stronger operator norm.

\else

\section{Graph-Admissibility Rates and Inference}
\label{sec:asymptotic_theory}

The finite-sample certificate yields a rate once the graph separates smooth
regions with vanishing error. Let $m$ denote intrinsic dimension, $s$
within-region smoothness, and $t_n$ the continuum scale associated with the
graph resolvent. Three additional terms track the departures from the ideal
within-region smoother: $\zeta_n$ is fixed-graph approximation error,
$\xi_n$ is boundary-band error, and $\eta_n$ is the perturbation introduced by
adaptive conductance.

\begin{theorem}[Rate on a graph-admissible sequence]
\label{thm:rate_main}
Suppose the systematic first-stage component is piecewise $s$-smooth on a
compact $m$-dimensional geometry, the separated neighborhood graph satisfies
the standard heat-smoothing approximation, and the residual has no persistent
low-frequency alignment with that graph. Then
\begin{equation}
 n^{-1/2}\|\hat{\bv}_n-\bv_n^\star\|_2
 =
 O_p\!\left(
 t_n^{s/2}
 +\sqrt{\frac{t_n^{-m/2}}{n}}
 +\zeta_n+\xi_n+\eta_n
 \right).
\label{eq:rate_main}
\end{equation}
If the last three terms are smaller and
$t_n\asymp n^{-2/(2s+m)}$, the leading rate is
$n^{-s/(2s+m)}$.
\end{theorem}

The supplement provides the proof, sufficient graph conditions, and the
mapping from $t_n$ to the resolvent parameter. The leading term in
\eqref{eq:rate_main} is the standard within-region nonparametric rate. The
point of the theorem is to expose the additional cost of a boundary band and
data-adaptive conductance, not to improve the generic smoothing rate. In
particular, it assumes $\eta_n=o(1)$; a primitive proof of this condition for
the implemented percentile rule remains open.

When residualized treatment variation is bounded away from zero, the
finite-sample projective bound turns \eqref{eq:rate_main} into a coefficient
rate. A conventional linear score additionally needs negligible
generated-control error at root-$n$ scale, whereas a cross-fitted orthogonal
score replaces that requirement with a nuisance product-rate condition. The
supplement develops both routes, along with a finite-family calibration
inequality for the split selector and a bound for approximate resolvent
solves. Its selector comparison is relative to the best candidate in a fixed
admissible family, not to an unrestricted adaptive-graph oracle.

\fi

\ifsupplement

\section{Empirical Evidence and Operating Regimes}
\label{sec:experiments_supp}

\subsection{Protocol}

The experiments are organized around the terms of the
variation-allocation certificate. Fractured and multi-fracture designs create
unknown boundaries and hence structural leakage. Correlated residuals stress
attenuation when the residual itself is graph smooth, and weak instruments
stress the remaining-variation denominator. Representation perturbations and
empirical covariate clouds test whether the graph is admissible. Smooth
designs serve as negative controls, since boundary adaptation has no inherent
advantage there.

We evaluate both recovery of the first-stage residual and estimation of the
structural response. Hidden quantities enter only evaluation and explicitly
labeled oracle references. In the nonlinear synthetic tables, control
correlation is computed against the outcome-relevant $U$; the certificate
tables use the graph target $V^\star=X-g(Z)$ and report its discrepancy from
$U$ when relevant.

When the structural response $f_0$ is known, we report
\begin{equation}
    \mathrm{MSE}_{\mathrm{resp}}
    =
    \frac{1}{|\mathcal X_{\mathrm{test}}|}
    \sum_{x\in\mathcal X_{\mathrm{test}}}
    \{\hat f(x)-f_0(x)\}^2 .
    \label{eq:response_mse}
\end{equation}
We refer to \eqref{eq:response_mse} as structural-response MSE. In the
nonlinear experiments, $\hat f$ is the response component learned after the
generated control has been fixed. The metric is available when the response is
simulated, but not in the four real-IV applications. Accordingly, the
nonlinear tables are point-estimation benchmarks that are sensitive to
boundary errors in the control. The finite-sample bound in Section
\ref{sec:theory}, by contrast, concerns a linear control-function coefficient.
Uniform propagation through an arbitrary neural second stage would require
additional stability assumptions on both the response class and its training
algorithm.

The evidence proceeds from aggregate performance to mechanism and then to
external geometry. First, a common response learner compares generated
controls across the full synthetic grid. Next, targeted experiments examine
leakage, attenuation, noise mismatch, relevance, and the behavior of the
outcome-free selector under graph shattering or coherent miswiring. Finally,
semi-synthetic known-truth designs and four real-IV applications use empirical
feature geometries. The real applications describe estimate stability and
guard behavior; they do not test the latent-manifold assumption.

The benchmark comprises six designs: fractured, smooth, multi-fracture,
weak-instrument, correlated-residual, and high-dimensional nuisance. Each
embeds a one-dimensional latent coordinate in an observed feature vector of
dimension $d_Z$. Thus $d_Z=50$ stresses the ambient representation; it does
not assert an intrinsic dimension of 50. The latent dimension is $m=1$, within
the regime discussed in Remark \ref{rem:dimension_inference}. Combining
$n\in\{800,1500,3000\}$, $d_Z\in\{5,20,50\}$, and ten seeds yields 540 runs.
The main fractured cell uses $n=800$ and $d_Z=50$.

Targeted studies at that same size examine first-stage cross-validation, split
selection, mechanism ablations, graph certificates, and boundary stress. The
rewiring calibration uses $n=300$, 50 replications, and ten corruption levels.
\ref{app:exp_details} adds a linear-inference diagnostic,
meta-parameter sensitivity, component timings, and approximate solves. The
semi-synthetic real-$Z$ experiment retains observed tabular covariates but
simulates $X$, $Y$, $U$, and $V^\star$. Card schooling, Mroz labor supply,
Cigarette Demand, and Social Insurance contain no known structural response,
so we use them only to study estimate stability and guard actions.

Fixed A-IHF uses $K=15$, $\tau=2$, $\lambda=30$, and $p=80$. The
observational rule searches the finite family
\begin{equation}
\begin{aligned}
K&\in\{10,15,20\}, & \tau&\in\{1,2\},\\
\lambda&\in\{10,30,50\}, & p&\in\{70,80,90\}.
\end{aligned}
\end{equation}
There are 54 candidates per run. The graph-admissibility-filtered
observational variant uses \eqref{eq:gobs_rule} with minimum final weighted
degree $d_{0,n}=10^{-4}$, $\omega_n=0.5$, and
$a_{0,n}=1$. Thus every retained vertex must meet the degree floor, the
largest post-cut component must contain at least half of the sample, and local
treatment contrast cannot exceed the global-pair scale. All 540 selected
graphs in the locked main benchmark pass the added compatibility screen
(maximum $a_n=0.937$), so it does not alter any reported benchmark estimate.
These diagnostics use only $(Z,X)$. The asymptotic theory uses a local-graph
sequence $K_n$; the experiments use a finite candidate family of $K$ values
as part of the finite-path estimator. The trace in the GCV term is estimated
by Hutchinson probes. In the main nonlinear benchmark, the second stage is
the same additive neural regressor for all control-function methods. The
gradient-free claim concerns only the construction of $\hat{\bv}$ in Stage 1.
The linear alignment experiment in Section \ref{sec:linear_alignment} uses
ordinary least squares.

The matched projective audit reruns guarded A-IHF, graph ridge GCV, and graph
spectral GCV on every cell and seed, producing 540 paired controls per method.
For each run it records $p_n$, $q_n$, the exact distortion in
\eqref{eq:beta_decomposition}, its Cauchy--Schwarz bound, the ratio between
them, and the slack above the universal frontier. These simulator-only
quantities are computed after selection and never used for tuning. A separate
scale intervention uses the six $n=800$, $d_Z=50$ cells. Each estimated
control, as well as oracle $U$, is multiplied by
$c\in\{0.01,0.1,1,10,100\}$ before refitting the common additive ELU model.
We run the experiment with both raw controls and centered, unit-variance
controls. All positive factors give the same canonicalized input up to
roundoff, so the latter row is an implementation check rather than five
distinct inputs.

The baseline suite includes tuned graph ridge, graph spectral regression,
series regression, kernel ridge, random forest, histogram gradient boosting,
XGBoost, and neural control functions. Every tuning criterion uses only
$(Z,X)$, through first-stage cross-validation or graph GCV. The residual
protocols necessarily differ. Supervised inductive learners produce
out-of-fold residuals; A-IHF and the other graph methods are transductive
full-sample smoothers whose effective degrees of freedom are controlled by
GCV. For supervised rows, a single five-fold partition both ranks
hyperparameters and produces the final out-of-fold predictions. This is
CV-tuned OOF prediction, not nested cross-fitting: an observation is excluded
from its final fit but still contributes to global hyperparameter selection.
The released metadata records this distinction. OOF fitting prevents
mechanical interpolation of a training residual, whereas transductive fitting
uses the full feature graph and may retain same-sample bias. Neither protocol
is an inferential sample-splitting guarantee.
\ref{sec:baseline_protocol} gives the full protocol and seed variation;
linear IV and DeepIV appear only in separate diagnostics.

The suite also contains a neural control function selected by held-out
prediction error for $X$ given $Z$. When $g$ is the conditional mean,
\eqref{eq:prediction_residual_identity} makes this the held-out recovery risk
for $V^\star$ as well. It does not, however, separately measure boundary-local leakage or
the relevance denominator in \eqref{eq:graph_causal_bound}.

\subsection{Synthetic Benchmark and First-Stage Selection}

\subsubsection{Main Fractured Design}

\begin{table}[ht]
\centering
\caption{\textbf{Main fractured design.} Selected methods from the tuned baseline suite, averaged over 10 seeds for $n=800$, $d_Z=50$. First-stage selection uses only $(Z,X)$.}
\label{tab:main}
\begingroup
\setlength{\tabcolsep}{3pt}
\renewcommand{\arraystretch}{1.05}
\resizebox{\linewidth}{!}{%
\begin{tabular}{llcc}
\toprule
\textbf{Method} & \textbf{Selection} & \textbf{Control corr.} $\uparrow$ & \textbf{Structural-response MSE} $\downarrow$ \\
\midrule
A-IHF & observational with guardrail & 0.948 & \textbf{1.732} \\
A-IHF & observational & 0.948 & \textbf{1.732} \\
A-IHF & fixed & \textbf{0.953} & 1.810 \\
Graph ridge CF & graph GCV & 0.853 & 2.163 \\
Graph spectral CF & graph GCV & 0.858 & 2.179 \\
Random forest CF & first-stage CV & 0.827 & 2.277 \\
XGBoost CF & first-stage CV & 0.793 & 2.603 \\
Kernel ridge CF & first-stage CV & 0.750 & 2.883 \\
CV-tuned Deep CF & first-stage CV & 0.764 & 3.153 \\
Deep ensemble CF & fixed ensemble & 0.715 & 3.589 \\
\bottomrule
\end{tabular}
}
\endgroup
\end{table}

In the main fractured cell, both observational rules select the same effective
candidate and improve on the fixed default in structural-response MSE. Their
selection uses no simulator-only quantities. Among the alternatives, graph
GCV and the out-of-fold tree methods are the closest competitors.

\begin{figure}[ht]
\centering
\includegraphics[width=0.82\linewidth]{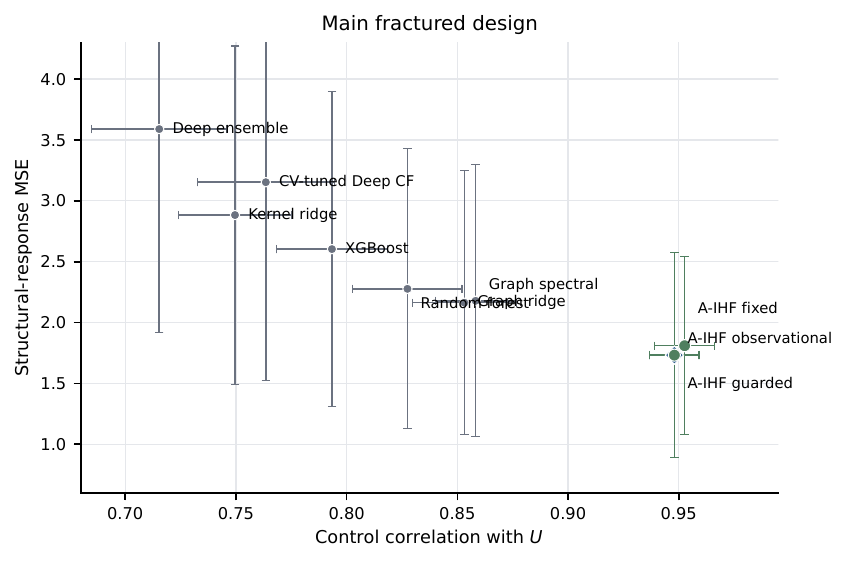}
\caption{\textbf{Residual recovery and downstream error in the main fractured design.} Guarded observational A-IHF uses only $(Z,X)$ for first-stage selection and lies in the high-correlation, low-error region.}
\label{fig:main_tradeoff}
\end{figure}

\subsubsection{Observational Neural Cross-Validation}

\noindent
To focus on first-stage tuning, we compare A-IHF with a neural
control function chosen by held-out treatment prediction. The same tuned
neural baseline appears in the full benchmark; Table
\ref{tab:cv_deep_cf} focuses on the three
$n=800$, $d_Z=50$ designs that make the contrast easiest to see.

\begin{table}[ht]
\centering
\small
\caption{\textbf{First-stage cross-validation.} Entries report mean structural-response MSE, with mean correlation against $U$ in parentheses, over 10 seeds for $n=800$, $d_Z=50$. The CV-tuned neural control function selects architecture, weight decay, learning rate, and stopping epoch by held-out prediction error for $X$ given $Z$.}
\label{tab:cv_deep_cf}
\begingroup
\setlength{\tabcolsep}{4pt}
\renewcommand{\arraystretch}{1.05}
\resizebox{\linewidth}{!}{%
\begin{tabular}{lcccc}
\toprule
\textbf{Design} & \textbf{A-IHF obs.} & \textbf{A-IHF fixed} & \textbf{CV-tuned Deep CF} & \textbf{Unreg. Deep CF} \\
\midrule
Fractured & \textbf{1.668} (0.948) & 1.779 (0.953) & 3.098 (0.756) & 4.433 (0.652) \\
High-dimensional nuisance & \textbf{2.270} (0.767) & 2.789 (0.637) & 4.881 (0.647) & 10.867 (0.093) \\
Smooth & 1.035 (0.971) & 0.822 (0.977) & \textbf{0.787} (0.971) & 9.395 (0.528) \\
\bottomrule
\end{tabular}
}
\endgroup
\end{table}

Cross-validation markedly improves the neural first stage over unregularized
training. Its target, however, is held-out prediction error for $X$ given
$Z$. For a linear second stage, Corollary \ref{cor:graph_causal} instead
identifies the pair
$\{\|M_{\hat{\bv}}\bu\|_2,\kappa_n(\hat{\bv})\}$ as intrinsic; direct
residual error is sufficient but may be loose. Predictive tuning does not
close the gap to A-IHF in the fractured or high-dimensional nuisance designs.
On the smooth design, where boundary adaptation should matter less, the
cross-validated neural model is competitive with fixed A-IHF.

\subsubsection{Linear Second-Stage Alignment}
\label{sec:linear_alignment}

To evaluate the setting covered directly by Corollary
\ref{cor:graph_causal}, we retain the same first-stage designs and replace the
nonlinear response by
\begin{equation}
    Y_{\mathrm{lin}}=\beta_0 X+\gamma_0 U+\epsilon,
    \qquad
    \beta_0=1,\quad \gamma_0=2.5.
\end{equation}
For each generated control $\hat V$, we estimate $\beta_0$ by ordinary least squares on $(1,X,\hat V)$. Table \ref{tab:linear_stage2} reports the mean absolute coefficient error over 10 seeds.

\begin{table}[ht]
\centering
\small
\caption{\textbf{Linear second-stage alignment.} Entries are mean absolute errors $|\hat\beta-\beta_0|$ for $Y_{\mathrm{lin}}=\beta_0X+\gamma_0U+\epsilon$, $n=800$, $d_Z=50$.}
\label{tab:linear_stage2}
\begingroup
\setlength{\tabcolsep}{4pt}
\renewcommand{\arraystretch}{1.05}
\resizebox{\linewidth}{!}{%
\begin{tabular}{lccccc}
\toprule
\textbf{Design} & \textbf{2SLS} & \textbf{CV Deep CF} & \textbf{Regularized Deep CF} & \textbf{A-IHF fixed} & \textbf{A-IHF obs.} \\
\midrule
Fractured & 0.048 & 0.085 & 0.061 & 0.054 & \textbf{0.031} \\
High-dimensional nuisance & 0.096 & 0.172 & \textbf{0.056} & 0.419 & 0.253 \\
Smooth & 0.262 & 0.146 & \textbf{0.065} & 0.162 & 0.085 \\
Correlated residual & 0.447 & 0.421 & 0.485 & \textbf{0.398} & 0.409 \\
\bottomrule
\end{tabular}
}
\endgroup
\end{table}

Observational A-IHF has the smallest coefficient error in the fractured
design, together with correlation $0.953$ with $U$ and
$\kappa_n=5.555$. The picture changes in the high-dimensional nuisance
design. Observational tuning improves on fixed A-IHF, but regularized Deep CF
and 2SLS have smaller coefficient errors. Correlation alone does not explain
the difference: A-IHF correlates more strongly with $U$ than CV-tuned Deep CF
($0.771$ versus $0.651$), yet leaves less residualized treatment variation
($1.950$ versus $2.864$). That variation is precisely the denominator in
\eqref{eq:graph_causal_bound}. All methods struggle when residuals are graph
correlated, as the spectral-leakage analysis predicts. We report the broader
54-cell coefficient comparison next and a finite-sample orthogonal-score
diagnostic in \ref{sec:orthogonal_diag}.

\subsubsection{Full Benchmark}

\begin{table}[ht]
\centering
\caption{\textbf{Aggregate results over 54 benchmark cells.} Each cell is averaged over 10 seeds; the table reports summaries of those cell averages. Structural-response MSE is the main nonlinear benchmark metric. The last column is a linear coefficient diagnostic.}
\label{tab:aggregate}
\begingroup
\setlength{\tabcolsep}{4pt}
\renewcommand{\arraystretch}{1.05}
\resizebox{\linewidth}{!}{%
\begin{tabular}{lcccc}
\toprule
\textbf{Method} & \textbf{Mean response MSE} $\downarrow$ & \textbf{Median response MSE} $\downarrow$ & \textbf{Median corr.} $\uparrow$ & \textbf{Mean linear abs. error} $\downarrow$ \\
\midrule
A-IHF, guarded observational & \textbf{2.671} & \textbf{2.533} & \textbf{0.904} & 0.204 \\
A-IHF, observational & 2.688 & \textbf{2.533} & \textbf{0.904} & 0.204 \\
A-IHF, fixed & 2.772 & 2.710 & 0.872 & 0.266 \\
Graph ridge CF & 2.912 & 2.756 & 0.839 & 0.142 \\
Random forest CF & 2.973 & 2.825 & 0.826 & 0.148 \\
Graph spectral CF & 3.049 & 3.168 & 0.850 & \textbf{0.115} \\
Kernel ridge CF & 3.148 & 3.034 & 0.781 & 0.138 \\
XGBoost CF & 3.219 & 3.429 & 0.798 & 0.120 \\
Deep ensemble CF & 3.385 & 3.545 & 0.740 & 0.120 \\
Series CF & 3.417 & 3.622 & 0.741 & 0.121 \\
CV-tuned Deep CF & 3.444 & 3.472 & 0.762 & 0.134 \\
Fixed Deep CF & 3.537 & 3.719 & 0.728 & 0.136 \\
HistGBDT CF & 3.606 & 3.566 & 0.779 & 0.186 \\
Graph ridge CF, fixed & 3.822 & 3.187 & 0.729 & 0.390 \\
\bottomrule
\end{tabular}
}
\endgroup
\end{table}

Guarded observational A-IHF has the lowest mean structural-response MSE over
the 54 cells. It ranks first in 17 cells and among the top three in 36;
observational and fixed A-IHF lead another 8 and 7 cells, respectively.
Counting the family rather than a single variant, A-IHF beats the best
non-A-IHF method in 32 cells. The linear diagnostic gives a different
ordering: graph spectral, Deep ensemble, XGBoost, and series CF all have lower
mean coefficient error.

The two outcomes should be kept separate. Structural-response MSE evaluates a
nonlinear point-estimation task under a common response learner. The final
column evaluates a low-capacity linear coefficient, the setting addressed
directly by Corollary \ref{cor:graph_causal}. That coefficient depends on
projective error and remaining treatment variation, not simply on the global
$L_2$ distance between $\hat V$ and $U$. Table \ref{tab:aggregate} therefore
reveals a tradeoff rather than uniform dominance. The matched audit in Table
\ref{tab:projective_empirical} identifies the source of the different
rankings.

\begin{table}[ht]
\centering
\caption{\textbf{Matched projective audit on the 54-cell grid.} Values average
540 paired runs per method.  Exact distortion is
$|\gamma_0\bx^\top M_{\hat{\bv}}\bu/
(\bx^\top M_{\hat{\bv}}\bx)|$; CF bound is the Cauchy--Schwarz bound in
\eqref{eq:beta_bound}; alignment is their per-run ratio; frontier slack is
$p_n-F_{\rho_n}(q_n)$.  Lower is better except for $q_n$.}
\label{tab:projective_empirical}
\begingroup
\setlength{\tabcolsep}{3.2pt}
\renewcommand{\arraystretch}{1.05}
\resizebox{\linewidth}{!}{%
\begin{tabular}{lrrrrrrr}
\toprule
\textbf{Method} & \textbf{Linear error} & \textbf{Exact distortion} &
$\boldsymbol{p_n}$ & $\boldsymbol{q_n}$ & \textbf{CF bound} &
\textbf{Alignment} & \textbf{Frontier slack} \\
\midrule
A-IHF, guarded & 0.204 & 0.205 & \textbf{0.357} & 0.548 & 1.018 & 0.211 & \textbf{0.328} \\
Graph ridge GCV & 0.142 & 0.142 & 0.398 & 0.551 & 0.978 & 0.159 & 0.384 \\
Graph spectral GCV & \textbf{0.115} & \textbf{0.115} & 0.387 & \textbf{0.619} & \textbf{0.879} & \textbf{0.136} & 0.366 \\
\bottomrule
\end{tabular}
}
\endgroup
\end{table}

Graph spectral wins the linear column even though A-IHF has stronger average
control fidelity. A-IHF has
the lowest $p_n$, the lowest direct control RMSE ($0.590$, versus $0.630$ for
graph ridge and $0.696$ for graph spectral), and the smallest frontier slack.
Graph spectral, however, leaves more normalized treatment variation and its
unexplained control component is less aligned with the residualized treatment.
Its exact distortion is therefore $0.115$, compared with $0.205$ for A-IHF.
Across all 1,620 method-runs, exact distortion correlates $0.998$
with realized linear absolute error. In matched comparisons, A-IHF has
lower $p_n$ than graph spectral in 51.9\% of runs and lower direct RMSE in
66.1\%, but lower exact distortion in only 34.3\%. Projective fidelity alone
therefore does not rank linear coefficients; relevance and the orientation of
the remaining error also matter.

\begin{table}[ht]
\centering
\caption{\textbf{Positive generated-control scale intervention.} Results use
the six $n=800$, $d_Z=50$ cells and ten seeds.  ``MSE range'' is the mean
within-run range over $c\in\{0.01,0.1,1,10,100\}$.  Canonicalization maps all
positive scales to the same input by construction.}
\label{tab:scale_intervention}
\begingroup
\setlength{\tabcolsep}{3.2pt}
\renewcommand{\arraystretch}{1.05}
\resizebox{\linewidth}{!}{%
\begin{tabular}{lrrrrr}
\toprule
\textbf{Method} & \textbf{Raw MSE, $c=1$} & \textbf{Canonical MSE} &
\textbf{Raw MSE range} & \textbf{Canonical range} &
\textbf{Max.\ OLS deviation} \\
\midrule
A-IHF, guarded & \textbf{2.076} & \textbf{2.077} & 1.223 & 0 & $2.08\times10^{-12}$ \\
Graph ridge GCV & 2.247 & 2.254 & 1.728 & 0 & $1.85\times10^{-12}$ \\
Graph spectral GCV & 2.356 & 2.360 & 1.344 & 0 & $2.09\times10^{-12}$ \\
Oracle $U$ & 1.101 & 1.102 & 1.733 & 0 & $2.12\times10^{-12}$ \\
\bottomrule
\end{tabular}
}
\endgroup
\end{table}

The scale intervention separates an algebraic invariance from the behavior of
a finite training pipeline. Across 2,400 intervention rows, the largest deviation
from the $c=1$ OLS coefficient is $2.12\times10^{-12}$; the largest changes in
$p_n$ and $q_n$ are $2.22\times10^{-16}$ and $3.33\times10^{-16}$.
The raw ELU learner is not comparably invariant under the fixed 500-epoch,
Adam, and weight-decay pipeline: its mean within-run MSE range is between
$1.223$ and $1.728$, including $1.733$ for the oracle control. Because the
same sensitivity occurs for $U$ itself, it is an optimization and
regularization calibration effect, not evidence against the control
direction. Canonicalization removes the arbitrary positive scale and
preserves the $c=1$ ranking: A-IHF remains below graph ridge and graph
spectral on this 60-run slice. We therefore use projective quantities only
for the linear coefficient theory and treat the nonlinear rows as empirical
learner evaluations.

The nonlinear benchmark consequently has a conditional stability
interpretation, not an algebraic scale-invariance theorem. If the response map
is locally Lipschitz in a calibrated control coordinate, replacing $U$ with
$\hat V$ perturbs the oracle response in proportion to
$\|\hat{\bv}-\bu\|_2$ on the evaluation distribution. A global norm can still
hide errors localized near a boundary: a smoother may estimate a linear
coefficient well while mixing observations across a jump in $g(Z)$. That
contamination is the structural-leakage term in Proposition
\ref{prop:diagnostic_certificate}. A-IHF trades some global smoothness for
boundary-sensitive recovery; it does not guarantee a uniformly smaller
residual norm. Subject to stability of the second-stage learner, the same
residual error carries into the learned response, in addition to the learner's
own estimation error. We therefore use the certificate terms as diagnostics
and hold the nonlinear architecture and training schedule fixed across
methods.

\subsubsection{Performance by Design}

\begin{table}[ht]
\centering
\caption{\textbf{Mean structural-response MSE by data-generating design.} Values average over $n$, $d_Z$, and seeds.}
\label{tab:dgp}
\begingroup
\setlength{\tabcolsep}{4pt}
\renewcommand{\arraystretch}{1.05}
\resizebox{\linewidth}{!}{%
\begin{tabular}{lcccc}
\toprule
\textbf{Design} & \textbf{A-IHF guardrail} & \textbf{A-IHF fixed} & \textbf{Best non-A-IHF} & \textbf{Best non-A-IHF method} \\
\midrule
Fractured & \textbf{3.526} & 3.718 & 4.228 & Graph spectral CF \\
Multi-fracture & \textbf{2.045} & 2.299 & 2.107 & Graph ridge CF \\
High-dimensional nuisance & \textbf{3.253} & 3.470 & 3.868 & RF CF \\
Smooth & 1.167 & 1.198 & \textbf{1.000} & Deep ensemble CF \\
Weak instrument & 0.803 & 0.738 & \textbf{0.619} & Graph ridge CF \\
Correlated residual & 5.231 & \textbf{5.212} & 5.450 & Graph spectral CF \\
\bottomrule
\end{tabular}
}
\endgroup
\end{table}

A-IHF leads on average in the fractured, multi-fracture, high-dimensional
nuisance, and correlated-residual groups, with the clearest margins in the
fractured and high-dimensional nuisance designs. Smooth designs favor standard
smoothers or neural baselines, and graph ridge performs best under weak
instruments. The correlated-residual group remains difficult for every graph
method because $V^\star$ has low-frequency graph energy, as anticipated by
\eqref{eq:leakage}; fixed A-IHF is slightly better than its guarded
observational variant in that group.

\begin{figure}[ht]
\centering
\includegraphics[width=0.86\linewidth]{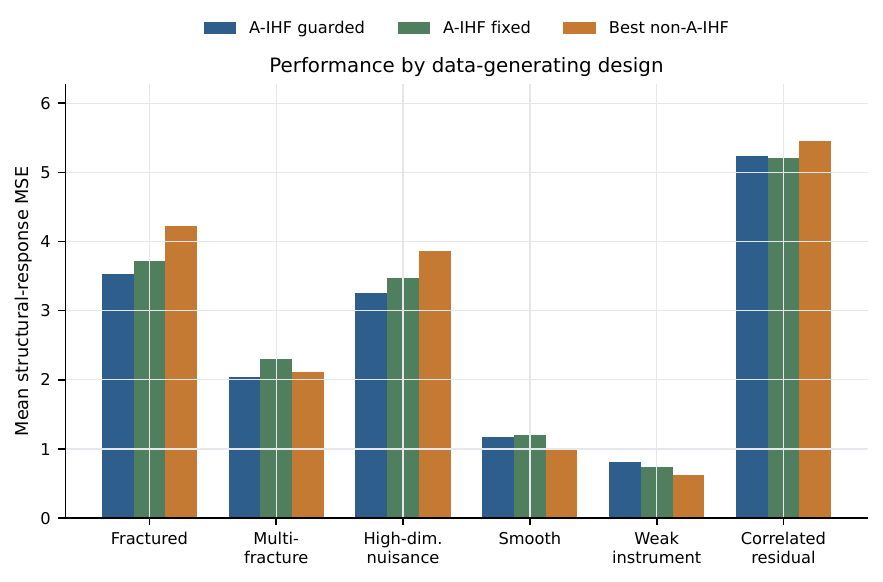}
\caption{\textbf{Mean structural-response MSE by data-generating design.} The comparison separates guarded observational A-IHF, fixed A-IHF, and the best non-A-IHF baseline in each design. A-IHF has its largest relative gains in fractured and high-dimensional nuisance designs. It is not uniformly best in smooth or weak-instrument settings.}
\label{fig:dgp_mse_map}
\end{figure}

\subsection{Mechanism and Failure-Mode Summary}

Aggregate performance alone does not reveal whether gains come from the
conductance update, the candidate family, or the admissibility filter. We
therefore examine graph construction and selection separately. On a fixed 54-cell grid, an unrestricted
observational search over-fragments the graph and raises mean response MSE
from 3.791 for fixed A-IHF to 5.878. Minimum-degree and
largest-component constraints change the selected graph and lower the mean to
4.521. Split selection is close to full-sample observational selection in the
fractured, high-dimensional nuisance, and correlated-residual designs, but is
less reliable in the smooth and weak-instrument controls.

The resulting mechanism claim is deliberately narrow. Removing anisotropy is
most damaging in fractured designs, whereas the pilot and roughness terms
matter only on parts of the grid. Representation perturbations show that the
method depends on a graph that respects first-stage geometry. Controlled
rewiring further distinguishes connectivity from compatibility:
topology-only checks accept coherently miswired graphs long after the
treatment-contrast screen begins to trigger fallback. Semi-synthetic
real-covariate experiments and real-IV applications then provide known-truth
and estimate-stability checks, respectively. Neither makes arbitrary
Euclidean neighborhoods in tabular data admissible. Detailed selection paths,
ablations, certificates, representation and rewiring studies, empirical
covariate checks, numerical identity tests, runtimes, and solver diagnostics
appear in the supplement.

\fi

\ifsupplement\else

\section{Empirical Evidence and Operating Regimes}
\label{sec:experiments}

\subsection{Protocol}

The experiments ask whether boundary adaptation improves the generated
control where the graph is informative, and whether the proposed diagnostics
flag settings in which it is not. Six synthetic designs isolate fractured and
smooth first stages, multiple boundaries, weak instruments, graph-correlated
residuals, and high-dimensional nuisance coordinates. We cross three sample
sizes with three ambient dimensions and average each of the resulting 54
cells over ten seeds. Structural-response MSE, computed against the known
response function, is the primary nonlinear metric. A separate linear outcome
is used for the coefficient diagnostics covered by the theory.

All first-stage tuning uses only $(Z,X)$. Supervised baselines use
cross-validated out-of-fold residuals; graph methods are transductive
smoothers tuned by graph GCV or the observational A-IHF rule. The baseline
suite includes graph ridge and spectral filters, series and kernel regression,
tree ensembles, boosting, and neural control functions. Every generated
control is passed to the same additive neural response learner. The full
candidate path, training protocol, seed variation, and local experiments are
reported in Section \ref{sec:experiments_supp} of the supplement.

\subsection{Aggregate Performance}

\begin{table}[ht]
\centering
\caption{\textbf{Aggregate results over 54 benchmark cells.} Entries summarize
cell averages; structural-response MSE is the main nonlinear metric. The last
column is a linear coefficient diagnostic.}
\label{tab:aggregate_main}
\begingroup
\setlength{\tabcolsep}{4pt}
\renewcommand{\arraystretch}{1.05}
\resizebox{\linewidth}{!}{%
\begin{tabular}{lcccc}
\toprule
\textbf{Method} & \textbf{Mean MSE} $\downarrow$ & \textbf{Median MSE} $\downarrow$ & \textbf{Median corr.} $\uparrow$ & \textbf{Linear error} $\downarrow$ \\
\midrule
A-IHF, guarded observational & \textbf{2.671} & \textbf{2.533} & \textbf{0.904} & 0.204 \\
A-IHF, observational & 2.688 & \textbf{2.533} & \textbf{0.904} & 0.204 \\
A-IHF, fixed & 2.772 & 2.710 & 0.872 & 0.266 \\
Graph ridge CF & 2.912 & 2.756 & 0.839 & 0.142 \\
Random forest CF & 2.973 & 2.825 & 0.826 & 0.148 \\
Graph spectral CF & 3.049 & 3.168 & 0.850 & \textbf{0.115} \\
Kernel ridge CF & 3.148 & 3.034 & 0.781 & 0.138 \\
XGBoost CF & 3.219 & 3.429 & 0.798 & 0.120 \\
Deep ensemble CF & 3.385 & 3.545 & 0.740 & 0.120 \\
Series CF & 3.417 & 3.622 & 0.741 & 0.121 \\
CV-tuned Deep CF & 3.444 & 3.472 & 0.762 & 0.134 \\
Fixed Deep CF & 3.537 & 3.719 & 0.728 & 0.136 \\
HistGBDT CF & 3.606 & 3.566 & 0.779 & 0.186 \\
Graph ridge CF, fixed & 3.822 & 3.187 & 0.729 & 0.390 \\
\bottomrule
\end{tabular}}
\endgroup
\end{table}

Guarded observational A-IHF has the lowest mean structural-response MSE in
Table \ref{tab:aggregate_main}. It ranks first in 17 cells and among the top
three in 36; counting all three variants, the A-IHF family wins 32 cells. In
the main fractured design, guarded A-IHF reduces MSE from 2.163 for graph
ridge, the closest listed competitor, to 1.732. The advantage is not uniform.
Graph spectral CF is markedly better on the aggregate linear diagnostic, and
standard smoothers lead in several smooth and weak-instrument cells.

\begin{figure}[ht]
\centering
\includegraphics[width=0.82\linewidth]{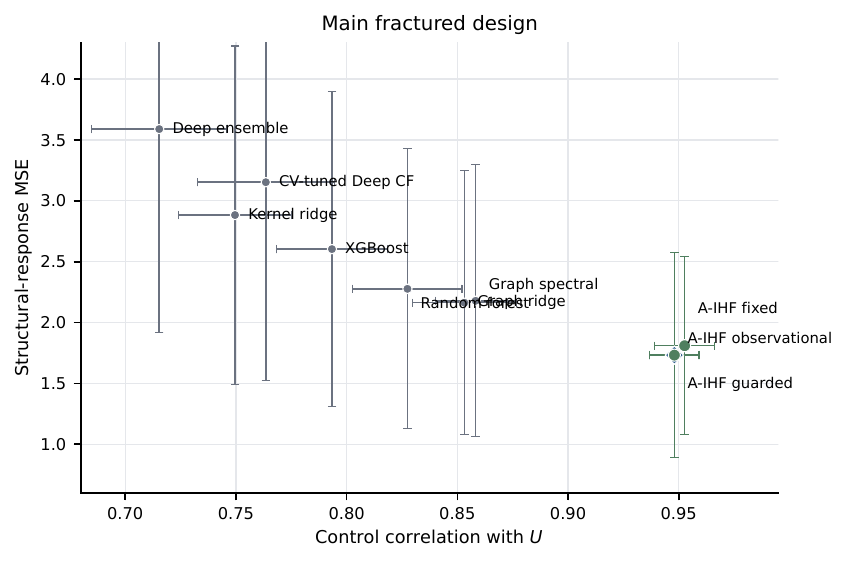}
\caption{\textbf{Main fractured design.} Guarded observational A-IHF uses no
outcomes for first-stage selection and lies in the high-correlation,
low-response-error region.}
\label{fig:main_tradeoff_main}
\end{figure}

\subsection{Projective Audit and Operating Regimes}

\begin{table}[ht]
\centering
\caption{\textbf{Matched projective audit on the 54-cell grid.} Values average
540 paired runs per method. Lower is better except for $q_n$.}
\label{tab:projective_empirical_main}
\begingroup
\setlength{\tabcolsep}{3.2pt}
\renewcommand{\arraystretch}{1.05}
\resizebox{\linewidth}{!}{%
\begin{tabular}{lrrrrrrr}
\toprule
\textbf{Method} & \textbf{Linear error} & \textbf{Exact distortion} &
$\boldsymbol{p_n}$ & $\boldsymbol{q_n}$ & \textbf{CF bound} &
\textbf{Alignment} & \textbf{Frontier slack} \\
\midrule
A-IHF, guarded & 0.204 & 0.205 & \textbf{0.357} & 0.548 & 1.018 & 0.211 & \textbf{0.328} \\
Graph ridge GCV & 0.142 & 0.142 & 0.398 & 0.551 & 0.978 & 0.159 & 0.384 \\
Graph spectral GCV & \textbf{0.115} & \textbf{0.115} & 0.387 & \textbf{0.619} & \textbf{0.879} & \textbf{0.136} & 0.366 \\
\bottomrule
\end{tabular}}
\endgroup
\end{table}

The audit explains the linear ranking. A-IHF gives the smallest projective
error $p_n$ and frontier slack, but graph spectral leaves more residualized
treatment variation and better aligns the remaining error. Its exact
coefficient distortion is therefore lower. Across all 1,620 matched runs,
exact distortion correlates $0.998$ with realized linear absolute error. The
experiment supports the variation-allocation account: fidelity alone does not
rank linear coefficients; relevance and error orientation also enter.

\begin{table}[ht]
\centering
\caption{\textbf{Mean structural-response MSE by design.} Values average over
sample size, ambient dimension, and seeds.}
\label{tab:dgp_main}
\begingroup
\setlength{\tabcolsep}{4pt}
\renewcommand{\arraystretch}{1.05}
\resizebox{\linewidth}{!}{%
\begin{tabular}{lcccc}
\toprule
\textbf{Design} & \textbf{A-IHF guardrail} & \textbf{A-IHF fixed} & \textbf{Best alternative} & \textbf{Method} \\
\midrule
Fractured & \textbf{3.526} & 3.718 & 4.228 & Graph spectral \\
Multi-fracture & \textbf{2.045} & 2.299 & 2.107 & Graph ridge \\
High-dimensional nuisance & \textbf{3.253} & 3.470 & 3.868 & Random forest \\
Smooth & 1.167 & 1.198 & \textbf{1.000} & Deep ensemble \\
Weak instrument & 0.803 & 0.738 & \textbf{0.619} & Graph ridge \\
Correlated residual & 5.231 & \textbf{5.212} & 5.450 & Graph spectral \\
\bottomrule
\end{tabular}}
\endgroup
\end{table}

The largest gains occur in fractured and high-dimensional nuisance designs;
the multi-fracture improvement is smaller. Smooth designs favor standard
smoothers, weak instruments favor graph ridge, and graph-correlated residuals
remain difficult for every graph method. These reversals match the theory:
anisotropy targets boundary leakage, not low-frequency residual attenuation
or a weak relevance denominator.

\begin{table}[ht]
\centering
\caption{\textbf{Positive generated-control scale intervention.} ``Range'' is
the mean MSE range over $c\in\{0.01,0.1,1,10,100\}$. Canonicalization centers
and standardizes the generated control.}
\label{tab:scale_intervention_main}
\begingroup
\setlength{\tabcolsep}{3.2pt}
\renewcommand{\arraystretch}{1.05}
\resizebox{\linewidth}{!}{%
\begin{tabular}{lrrrrr}
\toprule
\textbf{Method} & \textbf{Raw MSE} & \textbf{Canonical MSE} &
\textbf{Raw range} & \textbf{Canonical range} & \textbf{Max. OLS deviation} \\
\midrule
A-IHF, guarded & \textbf{2.076} & \textbf{2.077} & 1.223 & 0 & $2.08\times10^{-12}$ \\
Graph ridge GCV & 2.247 & 2.254 & 1.728 & 0 & $1.85\times10^{-12}$ \\
Graph spectral GCV & 2.356 & 2.360 & 1.344 & 0 & $2.09\times10^{-12}$ \\
Oracle $U$ & 1.101 & 1.102 & 1.733 & 0 & $2.12\times10^{-12}$ \\
\bottomrule
\end{tabular}}
\endgroup
\end{table}

Table \ref{tab:scale_intervention_main} separates the exact linear result from
the behavior of a trained nonlinear learner. Positive rescaling changes the
OLS coefficient and the projective quantities only at roundoff. The ELU
learner is scale-sensitive under a fixed optimization budget, including when
it receives oracle $U$; standardization removes the arbitrary scale. We use
projective quantities for the linear theory and treat nonlinear response MSE
as an empirical learner metric.

Controlled edge rewiring shows that connectivity alone can accept coherently
miswired graphs. The compatibility screen instead moves from graph use to
ridge fallback and, under severe corruption, abstention. Four real-IV
applications exhibit all three actions. Since their structural effects are
not observed, we interpret them as stability and guardrail checks, not
validation of identification. Section \ref{sec:experiments_supp} of the supplement reports the
rewiring calibration, real-data results, ablations, representation stress
tests, and full scale-intervention protocol.

\begin{figure*}[t]
\centering
\includegraphics[width=0.94\textwidth]{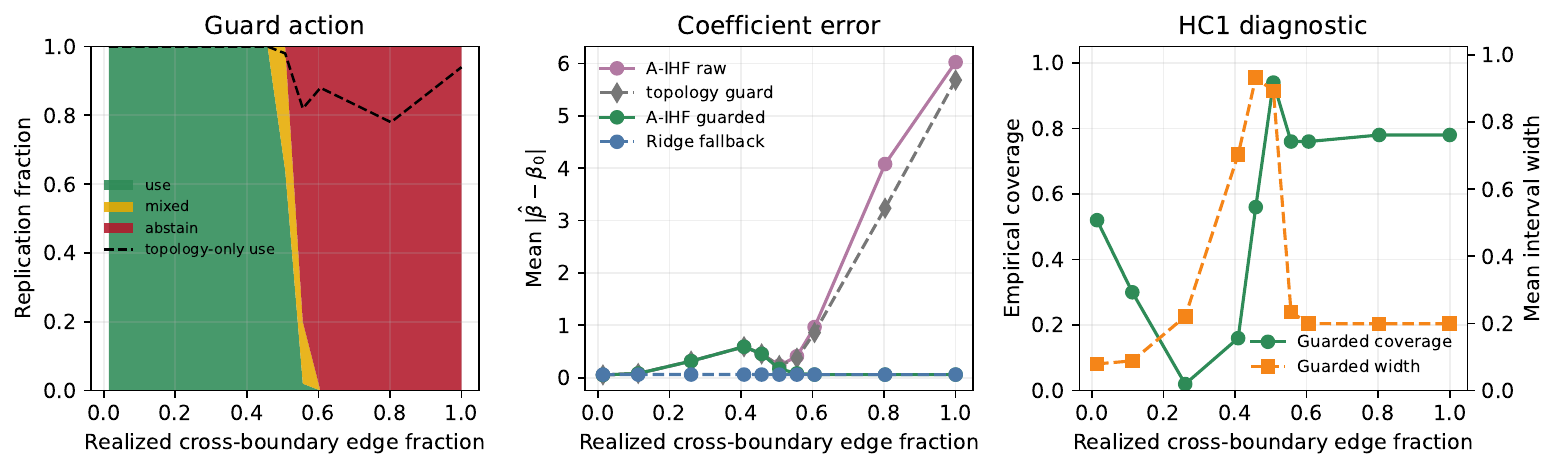}
\caption{\textbf{Guardrail calibration under controlled graph corruption.}
Connectivity remains high after coherent cross-boundary rewiring, while the
treatment-contrast screen increasingly triggers ridge fallback and eventually
abstention. The experiment calibrates an observable failure boundary; it is
not a uniform risk guarantee.}
\label{fig:guardrail_corruption_main}
\end{figure*}

\fi

\ifsupplement

\section{Additional Mechanism, Selection, and Failure-Mode Experiments}

\subsubsection{Selection and Graph-Admissibility Checks}
\label{sec:diagnostics}

All observational selections in the 540-run benchmark pass the relevance
screen in \eqref{eq:obs_rule}. Without the graph filter, the rule chooses
$p=70$ in every run and $K=10$ in 529 runs. The minimum-degree condition
changes this concentration: the guarded rule chooses $K=10$, $15$, and $20$
in 468, 40, and 32 runs, and $p=70$, $80$, and $90$ in 466, 28, and 46 runs.
The smallest selected weighted degree is $1.01\times10^{-4}$, just above the
prespecified floor, and the median largest-component fraction after cutting is
1.0. Thus the score tends to prefer sharp graphs, while the admissibility
filter removes its most fragmented choices.

An extended-grid experiment uses
\begin{equation}
K\in\{5,10,15,20\},\qquad p\in\{50,60,70,80,90\},
\end{equation}
and a fixed response evaluation grid $X_{\mathrm{test}}\in[-5,5]$. Over this
broader grid, the rule selects $p=50$ in all 540 runs and $K=5$ in 527. Mean
structural-response MSE rises from 3.791 for fixed A-IHF to 5.878. The
expanded family has therefore exposed an over-fragmentation failure rather
than a useful finer scale.

The extended-grid experiment is repeated with graph-admissibility constraints: the final graph must have mean degree at least 4, and the number of post-cut connected components must be no larger than $\sqrt n$. This guardrail moves the selected graph away from the sparsest candidate: $K=15$ is selected in 488 of 540 runs, the median number of post-cut components is 1, and the median largest-component fraction is 1.0. The average structural-response MSE of observational A-IHF decreases from 5.878 to 4.521 under the same fixed-grid evaluation. Table \ref{tab:guardrail_diag} reports the comparison.

Anisotropy is intended to weaken edges across structural jumps while retaining
a macroscopic neighborhood graph within each smooth region. Once a candidate
breaks that graph into many small components, it is more plausibly isolating
finite-sample noise than recovering structural boundaries. The degree and
component conditions formalize this distinction: they prevent shattering
without forbidding local attenuation across a boundary.

The main benchmark uses the graph-admissibility filter defined in Section
\ref{sec:obs_selection}, with largest-component threshold $\omega_n=0.5$
and edge-compatibility threshold $a_{0,n}=1$. The component value encodes a
benchmark prior that the tested fractured designs contain a dominant
macroscopic geometry after legitimate boundary attenuation. It is not a
statistical constant and is not used in the graph-limit theory. In
applications with several comparable disconnected regimes, such as three
equally sized subpopulations, this screen should be replaced by a
domain-appropriate component-size rule or removed and reported as a
graph-certificate diagnostic. The estimator is therefore the candidate
family together with its admissibility filters, not an unrestricted search
over arbitrarily shattered or globally incompatible graphs.

\begin{table}[ht]
\centering
\caption{\textbf{Fixed-grid graph-admissibility check.} Values average over the same 54 benchmark cells and use $X_{\mathrm{test}}\in[-5,5]$.}
\label{tab:guardrail_diag}
\begingroup
\small
\setlength{\tabcolsep}{3pt}
\renewcommand{\arraystretch}{1.08}
\resizebox{\linewidth}{!}{%
\begin{tabular}{p{0.34\linewidth}cp{0.40\linewidth}}
\toprule
\textbf{Procedure} & \textbf{Mean structural-response MSE} $\downarrow$ & \textbf{Selected graph summary} \\
\midrule
A-IHF, fixed & 3.791 & fixed $K=15$, $p=80$ \\
A-IHF, observational broad grid & 5.878 & $K=5$ in 527/540; $p=50$ in 540/540 \\
A-IHF, observational with graph guardrail & 4.521 & $K=15$ in 488/540; median components 1 \\
\bottomrule
\end{tabular}
}
\endgroup
\end{table}

The filter improves the broad-grid search, although fixed A-IHF still performs
better on this evaluation. Connectivity constraints therefore address one
failure of observational tuning, not every source of selection error. The
main results use a prespecified admissible family; broader searches must be
treated as part of the graph-admissibility problem.

\subsubsection{Edge-Corruption Calibration}
\label{sec:edge_corruption_calibration}

Connectivity detects shattering, but not a connected graph whose edges link
the wrong observations. We therefore calibrate the
compatibility screen in \eqref{eq:edge_compatibility} on a known-truth
piecewise-smooth first stage with $n=300$. Starting from a symmetric
15-nearest-neighbor graph, corruption level $\rho$ replaces a $\rho$ fraction
of within-region edges by equally weighted cross-boundary edges. Treatment,
outcome, and the non-graph fallback features are held fixed. The experiment
uses 50 Monte Carlo replications, ten pre-specified corruption levels, and
five-fold first-stage fitting, giving 500 run-level estimates. Selection is
repeated in each training fold and never reads held-out treatment or any
outcome.

\begin{table*}[t]
\centering
\small
\caption{\textbf{Guardrail calibration under edge rewiring.} Cross-edge is
the realized fraction of cross-boundary edges. Action columns are replication
fractions for the calibrated guard. Errors are mean
$|\hat\beta-\beta_0|$. The topology guard omits
\eqref{eq:edge_compatibility}; the calibrated guard uses $a_{0,n}=1$ and
falls back foldwise to ridge. Results use 50 replications per level.}
\label{tab:edge_corruption_calibration}
\begingroup
\setlength{\tabcolsep}{3.5pt}
\renewcommand{\arraystretch}{1.05}
\resizebox{\textwidth}{!}{%
\begin{tabular}{rrrrrrrrr}
\toprule
$\boldsymbol{\rho}$ & \textbf{Cross-edge} & \textbf{Topo.\ use} &
\textbf{Use} & \textbf{Mixed} & \textbf{Abstain} &
\textbf{Raw error} & \textbf{Topo.\ error} & \textbf{Guarded error} \\
\midrule
0.00 & 0.014 & 1.00 & 1.00 & 0.00 & 0.00 & 0.052 & 0.052 & 0.052 \\
0.25 & 0.261 & 1.00 & 1.00 & 0.00 & 0.00 & 0.316 & 0.316 & 0.316 \\
0.40 & 0.409 & 1.00 & 1.00 & 0.00 & 0.00 & 0.591 & 0.591 & 0.591 \\
0.50 & 0.507 & 0.98 & 0.64 & 0.36 & 0.00 & 0.230 & 0.232 & 0.168 \\
0.55 & 0.557 & 0.82 & 0.02 & 0.18 & 0.80 & 0.409 & 0.372 & 0.082 \\
0.60 & 0.606 & 0.88 & 0.00 & 0.00 & 1.00 & 0.966 & 0.861 & 0.059 \\
1.00 & 1.000 & 0.94 & 0.00 & 0.00 & 1.00 & 6.024 & 5.681 & 0.059 \\
\bottomrule
\end{tabular}}
\endgroup
\end{table*}

\begin{figure*}[t]
\centering
\includegraphics[width=\textwidth]{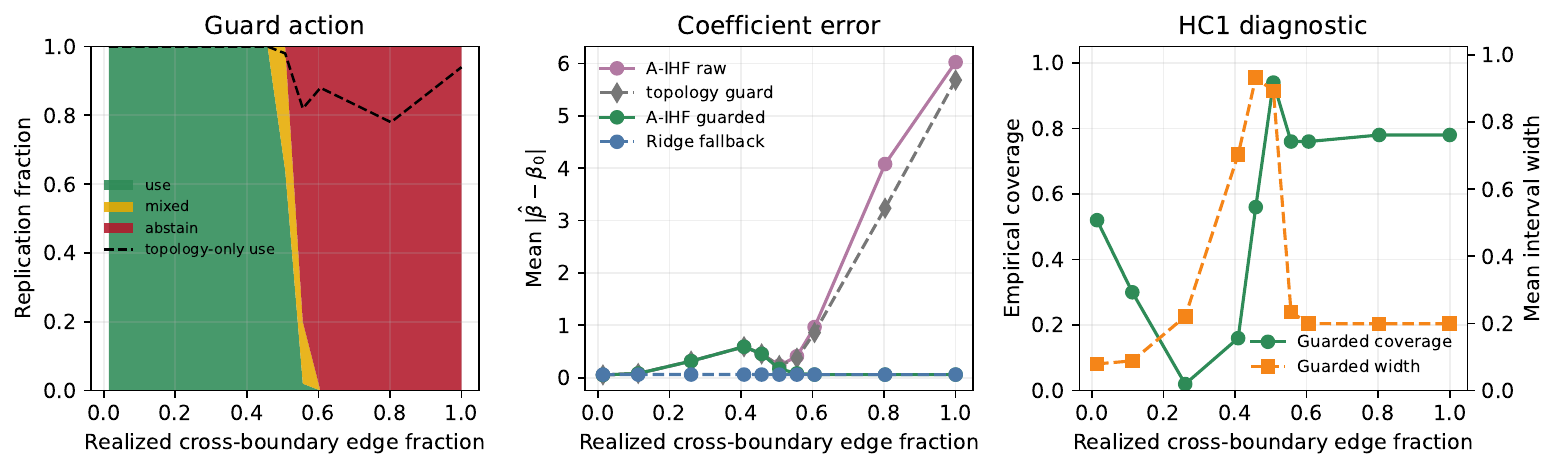}
\caption{\textbf{Calibration path from graph use to abstention.} The
topology-only guard frequently accepts a connected but globally miswired
graph. The edge-compatibility screen begins foldwise fallback near a 0.5
realized cross-boundary fraction and fully abstains from 0.6 onward. HC1
coverage and width are descriptive generated-control diagnostics, not a
formal coverage theorem.}
\label{fig:edge_corruption_calibration}
\end{figure*}

At a realized cross-edge fraction of $0.507$, 36\% of replications use a
mixture of A-IHF and ridge folds and the mean fallback fraction is $0.132$.
At $0.557$, 80\% fully abstain and the mean fallback fraction is $0.920$.
From $0.606$ onward every replication abstains, so guarded error equals ridge
error. In contrast, the topology-only guard still reports full use in 94\% of
replications at complete rewiring and its mean coefficient error is $5.681$.
Connectivity alone is therefore unable to detect coherent miswiring.

The screen does not provide uniform risk protection. Before $a_n$ reaches its
threshold, raw and guarded A-IHF coincide and may both be worse than ridge:
at cross-edge fraction $0.261$, for example, their mean error is $0.316$.
The experiment locates an observable abstention boundary rather than an oracle
performance boundary. As a compatibility check with the locked results, we
also compute $a_n$ for all 540 selected benchmark graphs and for every
$K\in\{10,15,20\}$ on 3,800 reconstructed real-IV training folds. All 11,400
real-IV candidate graphs and all benchmark selections pass, so the added
screen changes none of those estimates.

\subsubsection{Split-Selection Check}
\label{sec:split_selection_diag}

Proposition \ref{prop:selector_oracle} analyzes fixed candidate smoothers on an
evaluation split. Our implemented split experiment is more modest: one
subsample selects parameters from $(Z,X)$, after which the transductive graph
is refit on the full sample. It measures parameter-selection stability and
computational cost, not the oracle inequality itself.

\begin{table}[ht]
\centering
\small
\caption{\textbf{Split-selection check.} Entries report mean nonlinear structural-response MSE over 10 seeds for $n=800$, $d_Z=50$. The split-selected row selects parameters on a subsample and refits the graph smoother on the full sample. The last column is the mean runtime ratio relative to full observational selection.}
\label{tab:split_selection}
\begingroup
\setlength{\tabcolsep}{4pt}
\renewcommand{\arraystretch}{1.05}
\resizebox{\linewidth}{!}{%
\begin{tabular}{lcccc}
\toprule
\textbf{Design} & \textbf{A-IHF full obs.} & \textbf{A-IHF split-selected} & \textbf{Response-MSE difference} & \textbf{Runtime ratio} \\
\midrule
Fractured & 1.263 & 1.264 & 0.001 & 0.42 \\
High-dimensional nuisance & 2.059 & 2.067 & 0.008 & 0.29 \\
Smooth & 0.955 & 1.384 & 0.429 & 0.42 \\
Weak instrument & 0.719 & 1.114 & 0.395 & 0.41 \\
Correlated residual & 6.385 & 6.385 & 0.000 & 0.46 \\
\bottomrule
\end{tabular}
}
\endgroup
\end{table}

Split selection closely matches full observational selection in the fractured,
high-dimensional nuisance, and correlated-residual designs, but is less
reliable in the smooth and weak-instrument cases. It chooses $p=70$ in all 50
runs, $K=10$ in 44, $\lambda=10$ in 43, and $\tau=1$ in 28.

\subsubsection{Mechanism Ablation}
\label{sec:mechanism_ablation}

We compare the main graph operations at $n=800$, $d_Z=50$. The ablations use the same seeds and downstream learner as the split-selection experiment.

\begin{table}[ht]
\centering
\small
\caption{\textbf{A-IHF mechanism ablation.} Entries report mean nonlinear structural-response MSE over 10 seeds. The full observational rule uses the 54-candidate family.}
\label{tab:mechanism_ablation}
\begingroup
\setlength{\tabcolsep}{4pt}
\renewcommand{\arraystretch}{1.05}
\resizebox{\linewidth}{!}{%
\begin{tabular}{lccccc}
\toprule
\textbf{Design} & \textbf{Fixed A-IHF} & \textbf{Isotropic} & \textbf{No pilot} & \textbf{Obs. full} & \textbf{Obs. no roughness} \\
\midrule
Fractured & 1.308 & 2.583 & 1.766 & \textbf{1.263} & 1.266 \\
High-dimensional nuisance & 2.556 & 2.164 & 2.442 & 2.059 & \textbf{2.041} \\
Smooth & \textbf{0.846} & 0.963 & 0.877 & 0.955 & 1.367 \\
Weak instrument & \textbf{0.606} & 0.902 & 0.836 & 0.719 & 1.568 \\
Correlated residual & 5.990 & 9.197 & \textbf{5.952} & 6.385 & 6.385 \\
\bottomrule
\end{tabular}
}
\endgroup
\end{table}

The ablation clarifies which parts of A-IHF matter in each regime. In the
fractured design, isotropic smoothing reduces control correlation from $0.953$
to $0.727$ and increases response error; removing the pilot raises MSE from
$1.308$ to $1.766$. The hard cutoff has virtually no statistical effect on
this grid and should be viewed as a sparsification device. The roughness
penalty leaves the fractured and correlated-residual selections unchanged but
improves the smooth and weak-instrument designs.

\subsubsection{Representation Check}
\label{sec:representation_check}

The representation experiment varies the graph input and addresses Assumption \ref{ass:geometry}. Table \ref{tab:rep_stress} reports observational A-IHF under an oracle latent representation, the observed first-stage feature representation, principal-component representations, and observed features with additional irrelevant noise. \ref{sec:meta_sensitivity_app} reports the sensitivity of the observational score to $\alpha$ and $c_\kappa$.

\begin{table}[ht]
\centering
\small
\caption{\textbf{Representation stress test.} Mean structural-response MSE of observational A-IHF over 10 seeds for $n=800$, $d_Z=50$. The oracle latent representation is a simulation reference and is not used in the main benchmark.}
\label{tab:rep_stress}
\begingroup
\setlength{\tabcolsep}{4pt}
\renewcommand{\arraystretch}{1.05}
\resizebox{\linewidth}{!}{%
\begin{tabular}{lcccccc}
\toprule
\textbf{Design} & \textbf{Latent oracle} & \textbf{Observed $Z$} & \textbf{PCA-5} & \textbf{PCA-10} & \textbf{Noisy $Z$} & \textbf{Noisy PCA-5} \\
\midrule
Fractured & \textbf{1.689} & 1.767 & 2.070 & 3.507 & 3.190 & 3.301 \\
High-dimensional nuisance & \textbf{1.143} & 2.469 & 2.637 & 2.805 & 2.903 & 3.213 \\
\bottomrule
\end{tabular}
}
\endgroup
\end{table}

The latent representation gives the lowest response error, and the observed
features remain usable. Degraded representations can be much worse. In
particular, fixed A-IHF with PCA-10 has mean MSE $22.862$ in the fractured
design and $48.631$ in the high-dimensional nuisance design. Observational
selection mitigates, but does not eliminate, this failure. A-IHF should
therefore be understood as a residual extractor conditional on the graph
representation it receives.

\subsubsection{Graph Certificate and Boundary Stress}
\label{sec:boundary_stress}

Downstream response error combines several mechanisms. To separate them, we
report the certificate terms from Proposition
\ref{prop:diagnostic_certificate}. Each is a normalized root-mean-square
quantity for $V^\star=X-g(Z)$ and is available only in simulation, where
$\mathbf{g}$ and $V^\star$ are known.

\begin{table}[ht]
\centering
\small
\caption{\textbf{Graph-admissibility certificate.} Mean over 10 seeds for $n=800$, $d_Z=50$. Correlation and RMSE are computed against $V^\star=X-g(Z)$. The bound column is leakage plus residual attenuation for the graph residual target; the coefficient certificate also includes first-stage noise mismatch.}
\label{tab:graph_certificate}
\begingroup
\setlength{\tabcolsep}{3pt}
\renewcommand{\arraystretch}{1.05}
\resizebox{\linewidth}{!}{%
\begin{tabular}{llcccccc}
\toprule
\textbf{Design} & \textbf{Method} & \textbf{Corr.} & \textbf{RMSE} & \textbf{Leakage} & \textbf{Atten.} & \textbf{Bound} & $\boldsymbol{\kappa_n}$ \\
\midrule
Fractured & A-IHF obs. & 0.953 & 0.372 & 0.133 & 0.328 & 0.461 & 5.555 \\
Fractured & A-IHF fixed & 0.957 & 0.291 & 0.236 & 0.158 & 0.394 & 5.232 \\
Smooth & A-IHF obs. & 0.977 & 0.247 & 0.110 & 0.218 & 0.327 & 0.909 \\
Smooth & A-IHF fixed & 0.982 & 0.207 & 0.131 & 0.162 & 0.292 & 0.848 \\
Weak instrument & A-IHF obs. & 0.961 & 0.292 & 0.212 & 0.215 & 0.427 & 0.563 \\
Weak instrument & A-IHF fixed & 0.961 & 0.281 & 0.249 & 0.162 & 0.411 & 0.488 \\
Correlated residual & A-IHF obs. & 0.170 & 0.993 & 0.021 & 0.992 & 1.013 & 7.091 \\
Correlated residual & A-IHF fixed & 0.287 & 0.971 & 0.046 & 0.970 & 1.016 & 7.016 \\
High-dimensional nuisance & A-IHF obs. & 0.771 & 0.656 & 0.629 & 0.242 & 0.871 & 1.950 \\
High-dimensional nuisance & A-IHF fixed & 0.639 & 1.077 & 1.100 & 0.093 & 1.194 & 1.017 \\
\bottomrule
\end{tabular}
}
\endgroup
\end{table}

The certificate clarifies three distinct behaviors. Under correlated
residuals, leakage is small but attenuation is nearly one because the smoother
retains little of the low-frequency residual in the generated control. In the
high-dimensional nuisance design, observational selection lowers leakage and
raises relevance relative to the fixed default. Under weak instruments,
correlation remains high while $\kappa_n$ is small. Residual fidelity and
second-stage relevance must therefore be read together.

We also vary representation quality and first-stage strength. Table
\ref{tab:boundary_summary} averages over five designs and ten seeds and shows
the two endpoints of each path; for instrument strength, the left endpoint is
the weak-instrument setting.

\begin{table}[ht]
\centering
\small
\caption{\textbf{Boundary stress tests.} Values average over five designs and ten seeds at $n=800$, $d_Z=50$. Correlation and RMSE are computed against $V^\star=X-g(Z)$. Each cell reports the low-endpoint value followed by the high-endpoint value. For instrument strength, the endpoint path is $0.2\to1.0$; for noise dimensions, $0\to200$; for representation noise, $0\to2$.}
\label{tab:boundary_summary}
\begingroup
\setlength{\tabcolsep}{3pt}
\renewcommand{\arraystretch}{1.05}
\resizebox{\linewidth}{!}{%
\begin{tabular}{llcccc}
\toprule
\textbf{Stress path} & \textbf{Method} & \textbf{Corr.} & \textbf{RMSE} & \textbf{Bound} & $\boldsymbol{\kappa_n}$ \\
\midrule
Instrument strength & A-IHF obs. & $0.820\to0.766$ & $0.379\to0.512$ & $0.451\to0.620$ & $0.321\to3.214$ \\
Instrument strength & A-IHF fixed & $0.830\to0.765$ & $0.367\to0.566$ & $0.437\to0.661$ & $0.310\to2.920$ \\
Irrelevant dimensions & A-IHF obs. & $0.766\to0.545$ & $0.512\to1.129$ & $0.620\to1.400$ & $3.214\to0.684$ \\
Irrelevant dimensions & A-IHF fixed & $0.765\to0.564$ & $0.566\to1.383$ & $0.661\to1.543$ & $2.920\to0.174$ \\
Representation noise & A-IHF obs. & $0.766\to0.538$ & $0.512\to1.251$ & $0.620\to1.514$ & $3.214\to0.215$ \\
Representation noise & A-IHF fixed & $0.765\to0.556$ & $0.566\to1.493$ & $0.661\to1.628$ & $2.920\to0.028$ \\
\bottomrule
\end{tabular}
}
\endgroup
\end{table}

Irrelevant coordinates and isotropic representation noise increase the
certificate bound and reduce residualized treatment variation. Observational
selection often moderates the decline in RMSE and $\kappa_n$, but cannot
remove the dependence on a geometry-preserving representation. Weak
instruments fail differently: at strength $0.2$, the residual remains highly
correlated with its target while $\kappa_n$ is close to zero. The denominator
in Corollary \ref{cor:graph_causal} then becomes the limiting term.

\section{Empirical Covariate Geometry and Stability Checks}

\subsubsection{Semi-Synthetic Real-\texorpdfstring{$Z$}{Z} Experiment}
\label{sec:realz_diag}

To separate graph geometry from outcome uncertainty, we replace the synthetic
feature cloud with covariates from the \texttt{diabetes},
\texttt{breast\_cancer}, and \texttt{digits} data sets while continuing to
simulate treatment and outcomes. For each data set we generate fractured,
smooth, and weak-instrument designs with $n=400$ and ten seeds. This retains
known $U$, $V^\star$, and structural response while placing the graph on an
empirical covariate cloud. The nonlinear response uses $U$. The coefficient
column of Table \ref{tab:realz_overall} instead uses $V^\star$ in a linear
outcome, so it evaluates recovery of the graph residual target rather than
replicating the design of Section \ref{sec:linear_alignment}.

\begin{table}[ht]
\centering
\small
\caption{\textbf{Semi-synthetic real-$Z$ experiment.} Mean over three data sets, three simulated designs, and ten seeds. The covariates are real; $X$, $Y$, $U$, and $V^\star$ are simulated.}
\label{tab:realz_overall}
\begingroup
\setlength{\tabcolsep}{4pt}
\renewcommand{\arraystretch}{1.05}
\resizebox{\linewidth}{!}{%
\begin{tabular}{lccc}
\toprule
\textbf{Method} & \textbf{Structural-response MSE} $\downarrow$ & \textbf{$|\hat\beta-\beta_0|$} $\downarrow$ & \textbf{Control corr.} $\uparrow$ \\
\midrule
Kernel ridge CF & \textbf{6.094} & 0.248 & \textbf{0.884} \\
A-IHF, observational & 6.760 & 0.360 & 0.860 \\
Random forest CF & 7.469 & \textbf{0.236} & 0.860 \\
CV-tuned Deep CF & 9.866 & 0.455 & 0.853 \\
2SLS & 10.598 & 0.446 & 0.805 \\
Regularized Deep CF & 10.886 & 0.497 & 0.839 \\
Gradient boosting CF & 12.295 & 0.622 & 0.823 \\
A-IHF, fixed & 17.920 & 0.977 & 0.821 \\
Unregularized Deep CF & 30.043 & 1.394 & 0.140 \\
\bottomrule
\end{tabular}
}
\endgroup
\end{table}

These empirical covariate clouds are less favorable to A-IHF than the
synthetic graph. Observational A-IHF ranks second in response MSE and improves
the fixed default by $11.160$ points on average, but kernel ridge has the
lowest response MSE and highest control correlation, while random forest has
the smallest coefficient error. The unregularized neural first stage remains
unstable.

\begin{table}[ht]
\centering
\small
\caption{\textbf{Semi-synthetic real-$Z$ results by simulated design.} Entries are mean nonlinear structural-response MSE.}
\label{tab:realz_by_dgp}
\begingroup
\setlength{\tabcolsep}{4pt}
\renewcommand{\arraystretch}{1.05}
\resizebox{\linewidth}{!}{%
\begin{tabular}{lcccc}
\toprule
\textbf{Design} & \textbf{A-IHF obs.} & \textbf{KRR CF} & \textbf{RF CF} & \textbf{CV Deep CF} \\
\midrule
Real-$Z$ fractured & \textbf{1.968} & 2.198 & 2.601 & 3.293 \\
Real-$Z$ smooth & \textbf{4.863} & 5.420 & 8.087 & 8.051 \\
Real-$Z$ weak instrument & 13.448 & \textbf{10.665} & 11.719 & 18.253 \\
\bottomrule
\end{tabular}
}
\endgroup
\end{table}

Broken down by design, observational A-IHF leads in the fractured and smooth
cases but not under weak instruments. In the latter, its selected residual
correlates $0.913$ with the target yet leaves only $\kappa_n=0.115$ of
residualized treatment variation. This is the denominator effect in
\eqref{eq:graph_causal_bound}: high control fidelity cannot compensate for
the absence of treatment variation after conditioning.

The selection pattern resembles the synthetic benchmark. The observational rule selects $p=70$ in all 90 semi-synthetic runs, $K=10$ in 89 runs, $\lambda=10$ in 55 runs, and $\lambda=30$ in 35 runs. The admissible family remains part of the estimator.

\subsubsection{Real-IV Applications and Guard Actions}
\label{sec:real_iv_applications}

We consider four prespecified real-IV applications: Card schooling
\citep{card1995using,wooldridge2010econometric}, the standard wage--education
IV specification applied to the working-women subsample of the Mroz data
\citep{mroz1987sensitivity,wooldridge2020introductory}, the 1995 CigarettesSW
demand design \citep{stock2007introduction}, and the randomized-default Social
Insurance teaching extract \citep{cai2015social,huntingtonklein2022effect}. The
samples contain 3,010, 428, 48, and 1,378 complete observations,
respectively. All generated controls are out of fold. Each of 200 bootstrap
resamples repeats graph selection and the foldwise admissibility check; Social
Insurance resamples natural-village clusters, and the other applications
resample rows. Because the structural response is unobserved, these
experiments describe estimate stability and the action of the filter rather
than accuracy relative to a known causal effect.

\begin{table*}[t]
\centering
\small
\caption{\textbf{Four real-IV applications.} Each cell reports the
original-sample coefficient followed by the 2.5\% and 97.5\% quantiles from
200 full-procedure bootstrap resamples. Raw A-IHF is a diagnostic that ignores
abstention; guarded A-IHF is the reported estimator.}
\label{tab:real_iv}
\begingroup
\setlength{\tabcolsep}{3.8pt}
\renewcommand{\arraystretch}{1.08}
\resizebox{\textwidth}{!}{%
\begin{tabular}{lccccc}
\toprule
\textbf{Data set} & \textbf{2SLS} & \textbf{Ridge CF} &
\textbf{Graph ridge CF} & \textbf{A-IHF raw} &
\textbf{A-IHF guarded} \\
\midrule
Card &
$0.132\ [0.037,0.255]$ &
$0.133\ [0.034,0.217]$ &
$0.072\ [0.041,0.091]$ &
$0.071\ [0.050,0.087]$ &
$0.133\ [0.034,0.217]$ \\
Mroz &
$0.061\ [0.001,0.124]$ &
$0.062\ [-0.00005,0.122]$ &
$0.069\ [-0.003,0.148]$ &
$0.075\ [0.012,0.148]$ &
$0.075\ [0.012,0.148]$ \\
Cigarettes &
$-1.277\ [-1.793,-0.862]$ &
$-1.256\ [-1.774,-0.872]$ &
$-1.030\ [-3.031,1.153]$ &
$-1.501\ [-2.279,0.047]$ &
$-1.501\ [-2.279,0.047]$ \\
Social Insurance &
$0.791\ [0.126,1.979]$ &
$0.677\ [0.154,1.428]$ &
$0.171\ [-0.501,2.371]$ &
$0.274\ [-0.736,1.643]$ &
$0.444\ [-0.191,1.597]$ \\
\bottomrule
\end{tabular}}
\endgroup
\end{table*}

\begin{figure*}[t]
\centering
\includegraphics[width=0.94\textwidth]{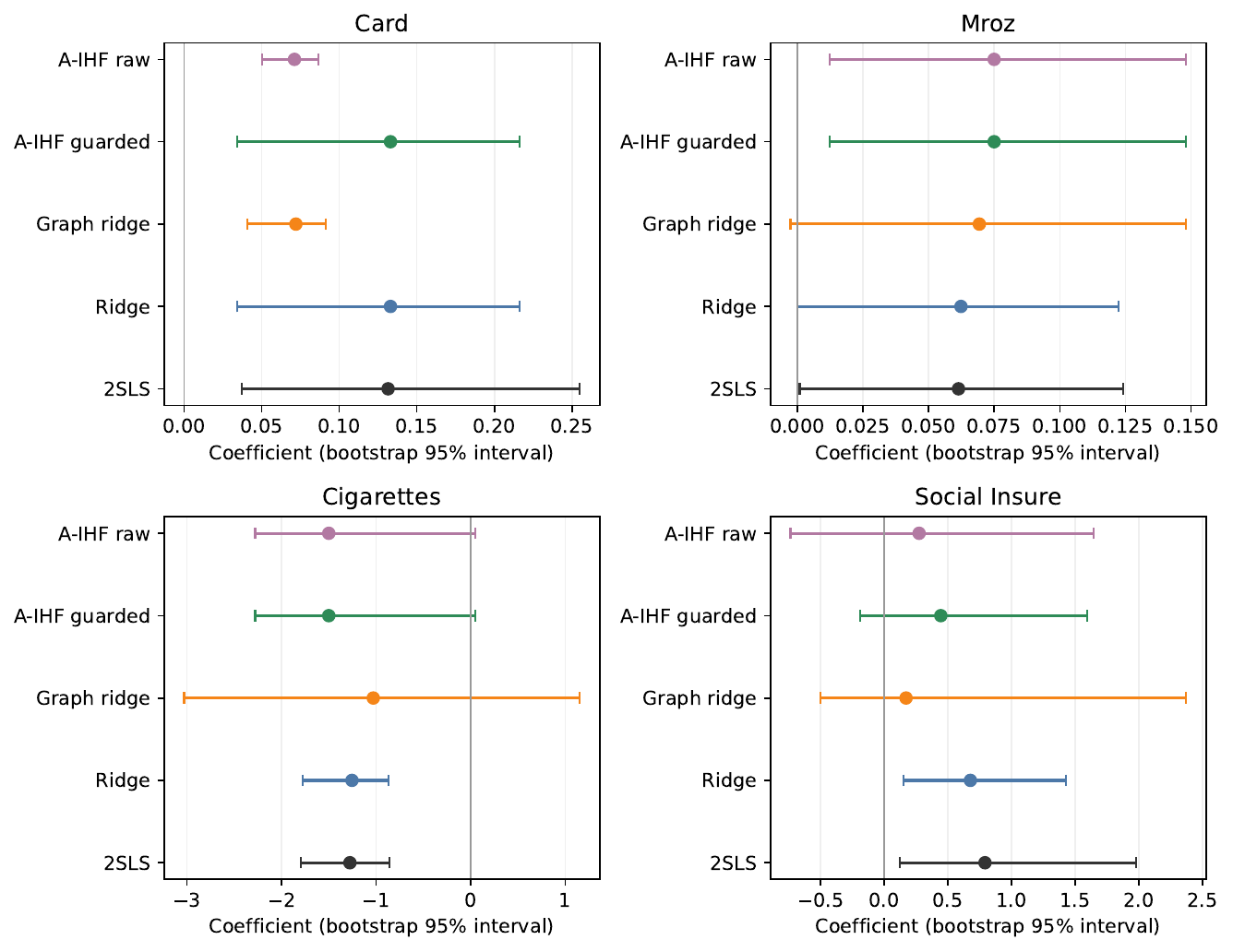}
\caption{\textbf{Real-IV coefficients and bootstrap intervals.} Raw A-IHF is
shown only to expose what the graph rule would return without abstention.
Card fully falls back to ridge, Mroz and Cigarettes use A-IHF, and Social
Insurance uses a foldwise mixture. Overlapping intervals preclude claims of
uniform superiority.}
\label{fig:real_iv_coefficients}
\end{figure*}

Here the guard actions are as important as the coefficient ranking. Card's
selected full graph has 425 components, a largest-component fraction of
$0.484$, and zero minimum degree; every bootstrap resample therefore fully
abstains, so guarded A-IHF coincides with ridge. Mroz and Cigarettes have
connected admissible graphs and use A-IHF. Of the Mroz bootstrap samples, 196
use A-IHF in every fold and four are mixed. Cigarettes uses A-IHF in all 200,
although its sample of 48 produces a wide interval that crosses zero. Social
Insurance has an inadmissible full graph and an original-sample fallback
fraction of $0.8$; 32 bootstrap samples use A-IHF throughout, 80 are mixed,
and 88 abstain completely. We report its wide interval without interpreting
it as evidence of a gain.

The partial first-stage $F$ statistics are $13.26$, $55.40$, $244.73$, and
$116.33$. In 200 instrument permutations, the corresponding 95th percentiles
are $3.37$, $3.19$, $3.78$, and $3.55$. These values support first-stage
relevance relative to the permutation reference; they say nothing about the
exclusion restriction.
The edge-compatibility statistic is below one in every original and
bootstrap training fold, so the new screen leaves all coefficients in
Table \ref{tab:real_iv} unchanged. The applications instead exercise the
degree and component safeguards.

\paragraph{Known-truth experiments on the same covariate geometries.}
Using the same feature clouds, we also simulate piecewise or smooth treatment
mechanisms, residuals, and linear outcomes. Twenty replications per
data-set--design cell yield 1,120 method-runs. HistGBDT has the lowest
coefficient error in six of eight cells; linear OOF leads on Mroz--smooth and
guarded A-IHF on Mroz--piecewise. Raw A-IHF improves on graph ridge in six
cells and on ridge in five, but on HistGBDT in only one. Boundary adaptation
can therefore improve a fixed graph smoother on these covariates, while
flexible non-graph learners remain stronger in several tabular geometries.

\section{Numerical Checks of the Exact Theory}
\label{sec:theory_verification}

Because the projective results are algebraic, they can be checked numerically
case by case rather than only through a simulation trend. A seeded verifier
generates centered regressions, signed spectral responses, feasible frontier
attainers, monotone gain vectors, and both versions of the three-node
construction. The two sides of each identity are computed independently, and
the program exits with a nonzero status if an identity or lower bound exceeds
its stated floating-point tolerance. Table \ref{tab:theory_verification}
reports the largest discrepancy observed.

\begin{table}[ht]
\centering
\small
\caption{\textbf{Numerical checks of the exact theory.} All
tests use seed 20260725 and pass. ``Excess'' is the computed violation of a
lower bound; values near machine precision may be positive because of
floating-point roundoff.}
\label{tab:theory_verification}
\begingroup
\setlength{\tabcolsep}{3.2pt}
\renewcommand{\arraystretch}{1.05}
\resizebox{\linewidth}{!}{%
\begin{tabular}{lrrr}
\toprule
\textbf{Claim tested} & \textbf{Cases} & \textbf{Max. discrepancy} &
\textbf{Tolerance} \\
\midrule
Projective coefficient identity & 5,000 & $3.11\times10^{-15}$ & $2\times10^{-9}$ \\
Nonzero scale invariance & 5,000 & $4.44\times10^{-16}$ & $2\times10^{-9}$ \\
Universal-frontier attainment & 2,121 & $4.79\times10^{-16}$ & $2\times10^{-12}$ \\
Universal-frontier excess & 25,000 & $8.55\times10^{-15}$ & $2\times10^{-12}$ \\
Spectral relevance identity & 19,998 & $4.97\times10^{-14}$ & $2\times10^{-9}$ \\
Spectral distortion identity & 19,998 & $1.90\times10^{-10}$ & $2\times10^{-9}$ \\
Monotone-cone excess & 100,000 & $0$ & $2\times10^{-10}$ \\
Exact topology escape & 81 & $3.54\times10^{-12}$ & $2\times10^{-9}$ \\
Connected-limit formulas & 138 & $2.66\times10^{-14}$ & $2\times10^{-9}$ \\
\bottomrule
\end{tabular}
}
\endgroup
\end{table}

The checks include gain rescalings from $10^{-9}$ to $10^9$, nonmonotone
target spectra, repeated-eigenvalue degeneracy, and connected conductances down
to $10^{-10}$. On the isotropic three-node graph, the
largest absolute relevance over 81 regularization values is
$6.00\times10^{-25}$; after the exact cut, the largest projective error is
$3.48\times10^{-25}$. The exact cut gain
$3\lambda/(1+3\lambda)$ also agrees with the mean-degree-scaled
implementation within $4.42\times10^{-13}$. For the
positive-conductance limit, the independently
computed closed forms agree with direct matrix projection within
$2.67\times10^{-14}$. These calculations complement rather than replace the
proofs in \ref{app:projective_proofs}; their purpose is to catch
sign, normalization, degeneracy, and pairwise-expansion mistakes.

\subsection{A Numerical Stress Test of the Rate Premise}
\label{sec:rate_stress}

The rate theorem assumes that the adaptive-graph perturbation
$\eta_n=\lambda_n\|L_n-L_n^0\|_{\mathrm{op}}$ tends to zero. We examine that
assumption directly, since improving prediction error does not imply it. In a
one-dimensional piecewise-$C^2$ first stage with one jump, we
set $K_n=\lceil0.65\log^2 n\rceil$,
$t_n=n^{-2/5}$, and
$\lambda_n=t_n/(K_n/n)^2$.  The oracle graph deletes edges crossing the known
partition.  For the implemented adaptive graph, $L_n^0$ retains its learned
within-region conductances and deletes only cross-region edges.  Table
\ref{tab:rate_stress} reports means over 20 independent samples.

\begin{table}[ht]
\centering
\small
\caption{\textbf{Stress test of the graph-admissibility premise.} RMSE is
$n^{-1/2}\|\hat V-V^\star\|_2$.  Cross mass is the fraction of learned
conductance on cross-partition edges.  The experiment evaluates a fixed
70th-percentile path and is not used to tune the reported benchmark.}
\label{tab:rate_stress}
\begingroup
\setlength{\tabcolsep}{3pt}
\renewcommand{\arraystretch}{1.05}
\begin{tabular}{rrrrr}
\toprule
$n$ & $K_n$ & Oracle RMSE & Adaptive RMSE & Cross mass / $\eta_n$ \\
\midrule
400  & 24 & 0.142 & 0.182 & $3.45{\times}10^{-3}$ / 9.28 \\
800  & 30 & 0.115 & 0.139 & $7.79{\times}10^{-4}$ / 11.02 \\
1600 & 36 & 0.092 & 0.098 & $2.57{\times}10^{-4}$ / 12.97 \\
3200 & 43 & 0.073 & 0.072 & $2.80{\times}10^{-5}$ / 9.79 \\
\bottomrule
\end{tabular}
\endgroup
\end{table}

Oracle and adaptive residual RMSE have log--log slopes $-0.320$ and
$-0.453$, respectively, and the learned cross-edge mass falls with slope
$-2.243$. Yet $\eta_n$ has slope $0.047$ on this path: the shrinking
cross mass is offset by the growing regularization scale in the operator
norm. The residual estimator improves, but the fixed-percentile path does not
satisfy the vanishing-perturbation premise on these sample sizes. The
experiment therefore reinforces the scope of Theorem \ref{thm:aihf_rate}: it
is a reduction conditional on graph admissibility, while a primitive analysis
of the percentile construction remains open. The artifact includes the
script and all 80 run-level records.

\section{Numerical and Evaluation Notes}

\ref{sec:runtime_diag} reports first-stage component timings. \ref{sec:solver_diag} reports conjugate-gradient checks for the final graph resolvent. The empirical solve error is below $5\times10^{-5}$ at $n=10000$ in the tested designs. \ref{sec:evaluation_grid_note} records the response-grid convention used in the benchmark.

\fi

\ifsupplement\else

\section{Limitations}
\label{sec:limitations}

A-IHF depends on the representation used to build the graph. If unrelated
observations become neighbors, conductance adaptation cannot reconstruct the
missing geometry. Irrelevant coordinates and measurement noise produce this
failure in the representation study. Even on a useful graph, the systematic
signal and residual must be separated spectrally: low-frequency residual
variation is absorbed by the smoother, while weak instruments can leave
$\kappa_n(\hat{\bv})$ too small for a stable second stage.

The projective results cover one generated control and a linear outcome
regression. Proposition \ref{prop:conductance_escape} establishes an
expressivity separation, not consistency of the implemented percentile rule.
The rate theorem likewise begins with a graph-admissible sequence and tracks
the perturbation caused by anisotropic conductance; a continuum limit for the
fully data-adaptive graph remains open. Several controls would require a
Grassmannian, rather than one-dimensional projective, analysis. Root-$n$
orthogonal inference also retains the usual product-rate restriction.

Selection is relative to a finite candidate path. The GCV term is a proxy for
first-stage smoothness, and Proposition \ref{prop:selector_oracle} retains an
explicit gap between that proxy and residual-recovery risk. Stable graphs and
roughly homoskedastic residuals make the approximation plausible;
heteroskedastic or graph-correlated residuals need not. An overly broad path
can also fragment the graph. The candidate family is therefore part of the
estimator, not an innocuous implementation choice.

The neural outcome experiments share one fixed learner but are not covered by
the linear coefficient theorem. Their scale intervention confirms that
finite optimization and regularization need not inherit the exact linear
invariance. The real-IV studies have a similarly limited role: because neither
the true control nor the structural effect is observed, their intervals show
stability under the stated instruments rather than validate identification.
The compatibility screen can also accept moderate graph corruption before it
triggers fallback. It is a diagnostic boundary, not a risk guarantee.

\section{Conclusion}
\label{sec:conclusion}

The first stage of a control-function estimator does more than predict
treatment: it decides which treatment variation becomes a control and which
remains available to identify the response. Framing that decision as
variation allocation separates three failures that a prediction score can
conflate---interpolation, leakage across a graph boundary, and loss of
residualized treatment variation. A-IHF addresses the second failure while
checking explicitly for the third.

For a linear second stage, the relevant discrepancy is projective rather than
Euclidean. The exact frontier and spectral identities show why control
fidelity alone cannot rank coefficient estimates. The monotone-filter lower
bound then identifies a restriction of fixed-graph residual filters, and the
connected construction shows precisely what treatment-dependent conductance
adds: it can change the graph eigenspaces and recover positive relevance where
every shift-invariant filter on the original graph fails.

The experiments support a conditional, not universal, recommendation.
Guarded observational A-IHF lowers mean structural-response MSE by 8.3\%
across 54 cells, and the A-IHF family wins 32 cells. Its gains concentrate in
fractured and high-dimensional nuisance designs. Smooth designs, weak
instruments, graph-correlated residuals, and poorly represented tabular
geometries narrow or reverse the comparison; graph spectral smoothing is also
better on the aggregate linear coefficient diagnostic. The broader point is
therefore a design rule for staged learning: when a learned residual becomes
a downstream representation, its selector should measure the information the
next stage retains and should be able to decline the representation on which
its structure depends.


\section*{Declaration of competing interest}

The authors declare that they have no known competing financial interests or
personal relationships that could have appeared to influence the work
reported in this paper.

\section*{Data and code availability}

The benchmark datasets used in this study are publicly available from the
original sources cited in the manuscript. The authors intend to make the
implementation code, experimental scripts, configuration files, and
run-level metadata available in a public repository at a later stage.

\section*{Declaration of generative AI and AI-assisted technologies in the manuscript preparation process}

In this work, generative AI tools were used to assist with drafting and
editing portions of the manuscript for clarity, grammar, and readability.
They were not used to formulate the scientific claims, develop theoretical
results or proofs, design the methodology or experiments, or interpret the
experimental results. All AI-assisted text was reviewed, verified, and
revised by the authors. We take full responsibility for the final content
of this work.

\fi

\ifsupplement

\section{Proofs}
\label{app:proofs}

\subsection{Auxiliary Finite-Sample Proofs}
\label{app:auxiliary_proofs}

\begin{proof}[Proof of Proposition \ref{prop:gcv_identity}]
Since $(I-S)\bx=(I-S)\mathbf{g}+(I-S)\boldsymbol{\eta}$,
\begin{align}
    \E[\|(I-S)\bx\|_2^2\mid Z]
    &=\|(I-S)\mathbf{g}\|_2^2
      +\E[\boldsymbol{\eta}^\top(I-S)^2\boldsymbol{\eta}\mid Z]\\
    &=\|(I-S)\mathbf{g}\|_2^2+\sigma^2\Tr\{(I-S)^2\}.
\end{align}
\end{proof}

\begin{proof}[Proof of Proposition \ref{prop:info}]
An $\I_{\mathrm{obs}}$-measurable selector takes the same value under the two
coupled mechanisms because their observed samples agree. On the positive-
probability event $\{h_P^\star\ne h_{P'}^\star\}$, that value cannot equal both
oracle selectors.
\end{proof}

\begin{proof}[Proof of Proposition \ref{prop:decomp}]
Substitute $\bx=\mathbf{g}+\bv^\star$ into
$\hat{\bv}_h=(I-S_h)\bx$ and subtract $\bv^\star$.
\end{proof}

\begin{proof}[Proof of Proposition \ref{prop:beta}]
Frisch--Waugh--Lovell gives
$\hat\beta(\hat{\bv})=(\bx^\top M_{\hat{\bv}}\by)/
(\bx^\top M_{\hat{\bv}}\bx)$. Substitution of Assumption
\ref{ass:second_stage} gives \eqref{eq:beta_decomposition}. Since
$M_{\hat{\bv}}$ is an orthogonal projector and
$M_{\hat{\bv}}\hat{\bv}=0$,
\begin{equation}
\begin{aligned}
|\bx^\top M_{\hat{\bv}}\bu|
&=|(M_{\hat{\bv}}\bx)^\top(M_{\hat{\bv}}\bu)|\\
&\le \sqrt{n\kappa_n(\hat{\bv})}\,\|M_{\hat{\bv}}\bu\|_2,\\
\|M_{\hat{\bv}}\bu\|_2
&=\|M_{\hat{\bv}}(\bu-\hat{\bv})\|_2
\le\|\bu-\hat{\bv}\|_2.
\end{aligned}
\end{equation}
These inequalities prove \eqref{eq:beta_bound}.
\end{proof}

\begin{proof}[Proof of Corollary \ref{cor:graph_causal}]
Proposition \ref{prop:beta} gives the first inequality. Moreover,
\begin{align}
\|M_{\hat{\bv}_h}\bu\|_2
&\le\|\bu-\hat{\bv}_h\|_2\\
&\le\|(I-S_h)\mathbf{g}\|_2+\|S_h\bv^\star\|_2
    +\|\bv^\star-\bu\|_2,
\end{align}
where the second line uses Proposition \ref{prop:decomp}. Substitution yields
\eqref{eq:graph_causal_bound}; the final display follows from
$(a+b+c)^2\le3(a^2+b^2+c^2)$.
\end{proof}

\begin{proof}[Proof of Proposition \ref{prop:diagnostic_certificate}]
Equation \eqref{eq:cert_residual} is Proposition \ref{prop:decomp} divided by
$\sqrt n$. Also,
\begin{align}
\mathsf{Proj}_n(h)
&\le n^{-1/2}\|\hat{\bv}_h-\bu\|_2 \\
&\le \mathsf{Leak}_n(h)+\mathsf{Atten}_n(h)+\mathsf{Noise}_n.
\end{align}
Equation \eqref{eq:cert_beta} follows from Proposition \ref{prop:beta}.
\end{proof}

\subsection{Proof of Proposition \ref{prop:ntk}}
\label{app:ntk_collapse}

\begin{proof}[Proof of Proposition \ref{prop:ntk}]
Under continuous-time gradient flow in the NTK regime,
\begin{equation}
    \frac{d}{dt}f_t(Z)=\Sigma(\bx-f_t(Z)).
\end{equation}
Let $\hat{\bv}_{NN}(t)=\bx-f_t(Z)$. Then
\begin{equation}
    \frac{d}{dt}\hat{\bv}_{NN}(t)=-\Sigma\hat{\bv}_{NN}(t).
\end{equation}
Expanding in the eigenbasis of $\Sigma$ gives
\begin{equation}
    \frac{d}{dt}
    \langle \hat{\bv}_{NN}(t),\phi_j\rangle
    =
    -\mu_j
    \langle \hat{\bv}_{NN}(t),\phi_j\rangle .
\end{equation}
Solving this scalar differential equation yields the stated expression. Since $\mu_j>0$ for every $j$, each coefficient converges to zero.
\end{proof}

\subsection{Proof of Proposition \ref{prop:leakage}}

\begin{proof}
Since $L$ is symmetric positive semidefinite, write
\begin{equation}
    L=\sum_{j=1}^n \mu_j\phi_j\phi_j^\top .
\end{equation}
Then
\begin{equation}
    S^2=(I+\lambda L)^{-2}
    =
    \sum_{j=1}^n
    (1+\lambda\mu_j)^{-2}
    \phi_j\phi_j^\top .
\end{equation}
Using $\E[\bv^\star(\bv^\star)^\top]=\Sigma_V$,
\begin{align}
    \E\|S\bv^\star\|_2^2
    &=
    \E[(\bv^\star)^\top S^2\bv^\star] \\
    &=
    \Tr(S^2\Sigma_V) \\
    &=
    \sum_{j=1}^n
    \frac{\phi_j^\top \Sigma_V \phi_j}
    {(1+\lambda\mu_j)^2}.
\end{align}
\end{proof}

\subsection{Projective and Spectral Proofs}
\label{app:projective_proofs}

\begin{proof}[Proof of Theorem \ref{thm:projective_frontier}]
Only centered directions matter.  Normalize $\tilde{\bx}$,
$\tilde{\bu}$, and the centered generated control to unit norm, and select
representatives of their projective directions whose inner products with
$\tilde{\bx}$ are nonnegative.  Let
\begin{equation}
    \theta=\arccos\rho_n,\qquad
    \varphi=\arcsin\sqrt{q_0}.
\end{equation}
The constraint $q_n(\hat{\bv})\ge q_0$ says that the acute projective angle
between $\tilde{\bx}$ and $\hat{\bv}$ is at least $\varphi$.  If
$\varphi\le\theta$, choosing
$\hat{\bv}\parallel\tilde{\bu}$ is feasible and gives $p_n=0$.

If $\varphi>\theta$, the closest feasible direction to $\tilde{\bu}$ lies in
$\operatorname{span}(\tilde{\bx},\tilde{\bu})$ at angle $\varphi$ from
$\tilde{\bx}$.  Its angle from $\tilde{\bu}$ is
$\varphi-\theta$, so the minimum projective error is
$\sin^2(\varphi-\theta)$.  Since
\begin{equation}
    \sin(\varphi-\theta)
    =
    \rho_n\sqrt{q_0}
    -
    \sqrt{1-\rho_n^2}\sqrt{1-q_0},
\end{equation}
and this expression is positive exactly when $\varphi>\theta$, the two cases
combine to give \eqref{eq:projective_frontier}.  The displayed directions
also prove attainment.
\end{proof}

\begin{proof}[Proof of Theorem \ref{thm:spectral_allocation}]
Let
\begin{equation}
    A=\sum_jw_j,\qquad
    B=\sum_jw_jr_j,\qquad
    C=\sum_jw_jr_j^2.
\end{equation}
Because $\tilde{\bx}^{\top}\hat{\bv}=B$ and
$\|\hat{\bv}\|_2^2=C$,
\begin{equation}
    n\kappa_n
    =
    A-\frac{B^2}{C}
    =
    \frac{AC-B^2}{C}.
\end{equation}
Pairing the $(i,j)$ and $(j,i)$ terms gives
\begin{equation}
    AC-B^2
    =
    \sum_{i<j}w_iw_j(r_i-r_j)^2,
\end{equation}
which proves \eqref{eq:spectral_relevance_identity}; division by $A$ gives
\eqref{eq:spectral_relevance_variance}.

Similarly, write
\begin{equation}
    T=\sum_jw_jt_j=\tilde{\bx}^{\top}\tilde{\bu},
    \qquad
    V=\sum_jw_jr_jt_j=\hat{\bv}^{\top}\tilde{\bu}.
\end{equation}
Then
\begin{equation}
    \tilde{\bx}^{\top}M_{\hat{\bv}}\tilde{\bu}
    =
    T-\frac{BV}{C}
    =
    \frac{TC-BV}{C},
\end{equation}
and direct pairwise expansion yields
\begin{equation}
    TC-BV
    =
    \sum_{i<j}
    w_iw_j(r_j-r_i)(t_ir_j-t_jr_i).
\end{equation}
Taking the ratio with the relevance identity and using
\eqref{eq:beta_decomposition} proves
\eqref{eq:spectral_bias_identity}.  Treatment-inactive components of
$\tilde{\bu}$ are orthogonal to both $\tilde{\bx}$ and $\hat{\bv}$ and
therefore do not enter this numerator.
\end{proof}

\begin{proof}[Proof of Proposition \ref{prop:monotone_barrier}]
For any nonzero $r\in\mathcal K$, let $R$ have this active spectral response.
The projection identity gives
\begin{equation}
    \|M_{R\tilde{\bx}}\bu\|_2^2
    =
    \|\bu_\perp\|_2^2+\|t\|_w^2
    -
    \frac{\langle r,t\rangle_w^2}{\|r\|_w^2}.
\end{equation}
The projection theorem for a closed convex cone, applied to both $t$ and
$-t$, gives the sign-safe support identity
\begin{equation}
    \sup_{r\in\mathcal K\setminus\{0\}}
    \frac{\langle r,t\rangle_w^2}{\|r\|_w^2}
    =
    \max\left\{
        \|\Pi_{\mathcal K}t\|_w^2,
        \|\Pi_{\mathcal K}(-t)\|_w^2
    \right\}.
\end{equation}
This form remains valid when either cone projection is zero. Substitution
into the preceding projection identity proves \eqref{eq:monotone_barrier}. A
complementary resolvent has response
$r(\mu)=\lambda\mu/(1+\lambda\mu)$, which is nonnegative, nondecreasing, and
constant on repeated-eigenvalue blocks.
\end{proof}

\begin{proof}[Proof of Proposition \ref{prop:conductance_escape}]
The centered eigenspace of the complete-graph Laplacian has repeated
eigenvalue $3/2$ under the mean-degree scaling in
\eqref{eq:scaled_laplacian}. Hence every scalar spectral response that produces a nonzero
residual applies one common gain to $\phi_g$ and $\phi_u$, giving
$R\tilde{\bx}=c\tilde{\bx}$ and
$\tilde{\bx}^{\top}M_{R\tilde{\bx}}\tilde{\bx}=0$.

For the cut graph, the mean weighted degree is $2/3$, so $L_1$ is
$3/2$ times the raw cut Laplacian. Hence
\begin{equation}
    L_1\phi_g=0,\qquad L_1\phi_u=3\phi_u.
\end{equation}
The complementary resolvent has gain zero at eigenvalue 0 and gain
$3\lambda/(1+3\lambda)$ at eigenvalue 3, which proves
\eqref{eq:conductance_escape}.

For the connected perturbation, give edges $(1,2)$ and $(2,3)$ weight
$\delta>0$ and edge $(1,3)$ weight one. Its raw Laplacian is
\begin{equation}
    L_\delta^{\mathrm{raw}}=
    \begin{pmatrix}
        1+\delta&-\delta&-1\\
        -\delta&2\delta&-\delta\\
        -1&-\delta&1+\delta
    \end{pmatrix},
\end{equation}
and its mean degree is $\bar d_\delta=(2+4\delta)/3$. Therefore the
scaled Laplacian $L_\delta=L_\delta^{\mathrm{raw}}/\bar d_\delta$ has
\begin{equation}
    \mu_g(\delta)=\frac{9\delta}{2+4\delta},\qquad
    \mu_u(\delta)=\frac{3(2+\delta)}{2+4\delta}
\end{equation}
on $\phi_g$ and $\phi_u$, respectively. If
\begin{equation}
    r_g=\frac{\lambda\mu_g(\delta)}{1+\lambda\mu_g(\delta)},
    \qquad
    r_u=\frac{\lambda\mu_u(\delta)}{1+\lambda\mu_u(\delta)},
\end{equation}
then the residual is $ar_g\phi_g+br_u\phi_u$.  Therefore
\begin{equation}
    \|M_{\hat{\bv}}\bu\|_2^2
    =
    \frac{a^2b^2r_g^2}{a^2r_g^2+b^2r_u^2}
    \longrightarrow0,
\end{equation}
while
\begin{equation}
    n\kappa_n
    =
    a^2+b^2-
    \frac{(a^2r_g+b^2r_u)^2}
    {a^2r_g^2+b^2r_u^2}
    \longrightarrow a^2.
\end{equation}
This proves the connected-graph limit.
\end{proof}

\subsection{Proof of Lemma \ref{lem:jump_separation}}

\begin{proof}
On the joint high-probability event in Assumptions \ref{ass:geometry} and
\ref{ass:pilot_threshold}, a local edge $(i,j)$ within one smooth region
satisfies
\begin{equation}
    |\tilde x_i-\tilde x_j|
    \le
    |g(T_i)-g(T_j)|+2r_{p,n}
    \le
    \rho_{g,n}+2r_{p,n}.
\end{equation}
Thus
\begin{equation}
    1\ge C_{ij}
    \ge
    \exp\left\{
    -\frac{(\rho_{g,n}+2r_{p,n})^2}{\gamma_n}
    \right\}
    \to 1 .
\end{equation}
For a local cross-jump edge, \eqref{eq:jump_gap} gives
\begin{equation}
    |\tilde x_i-\tilde x_j|
    \ge
    |g(T_i)-g(T_j)|-2r_{p,n}
    \ge
    \Delta-2\rho_{g,n}-2r_{p,n}.
\end{equation}
Therefore
\begin{equation}
    C_{ij}
    \le
    \exp\left\{
    -\frac{(\Delta-2\rho_{g,n}-2r_{p,n})^2}{\gamma_n}
    \right\}
    \to0 .
\end{equation}
The bounds are uniform over the stated edge sets on the same high-probability event.
\end{proof}

\subsection{Proof of Lemma \ref{lem:percentile_threshold}}

\begin{proof}
For within-region edges, \eqref{eq:percentile_within} gives
\begin{equation}
    C_{ij}
    =
    \exp\{-D_{ij}/\hat\gamma_n\}
    \ge
    e^{-c_\gamma}
\end{equation}
on the stated high-probability event. For cross-jump edges, \eqref{eq:percentile_gap} and \eqref{eq:percentile_jump} give
\begin{equation}
    C_{ij}
    =
    \exp\{-D_{ij}/\hat\gamma_n\}
    \le
    \exp\{-b_n/\hat\gamma_n\}
    \to0 .
\end{equation}
The two bounds are uniform over the corresponding edge sets.
\end{proof}

\subsection{Proof of Lemma \ref{lem:resolvent_leakage}}

\begin{proof}
The resolvent identity gives
\begin{equation}
    S_n-S_n^0
    =
    (I+\lambda_n L_n)^{-1}
    \lambda_n(L_n^0-L_n)
    (I+\lambda_n L_n^0)^{-1}.
\end{equation}
Since $L_n$ and $L_n^0$ are symmetric positive semidefinite, both resolvents have operator norm at most one. Hence
\begin{equation}
    \opnorm{S_n-S_n^0}
    \le
    \lambda_n\opnorm{L_n-L_n^0}
    =
    \eta_n .
\end{equation}
Multiplying by $n^{-1/2}\|\bx_n\|_2=O_p(1)$ gives \eqref{eq:resolvent_leakage_bound}.

For the degree bound, write $\Delta W=W_n-W_n^0$ and $\Delta D=D_n-D_n^0$. The operator norm of $\Delta D-\Delta W$ is bounded by a constant times $\max_i\sum_j|\Delta W_{ij}|$ by Gershgorin's theorem and the bound $\opnorm{\Delta W}\le \max_i\sum_j|\Delta W_{ij}|$ for symmetric matrices. The normalization gives
\begin{equation}
    L_n-L_n^0
    =
    \frac{D_n-W_n}{\bar d_n}
    -
    \frac{D_n^0-W_n^0}{\bar d_n^0}.
\end{equation}
The first difference is controlled by $\max_i\sum_j|\Delta W_{ij}|/\bar d_n^0$. The second is controlled by $|\bar d_n-\bar d_n^0|/\bar d_n^0$ and the operator-norm bound in \eqref{eq:mean_degree_scaling}. Condition \eqref{eq:mean_degree_scaling} gives $\opnorm{L_n-L_n^0}\le C\ell_n$.
\end{proof}

\subsection{Proof of Theorem \ref{thm:aihf_rate}}

\begin{proof}
By Proposition \ref{prop:decomp},
\begin{equation}
    \hat{\bv}_n-\bv_n^\star
    =
    \mathbf{g}_n-S_n\bx_n .
\end{equation}
Add and subtract $S_n^0\bx_n$:
\begin{equation}
    \hat{\bv}_n-\bv_n^\star
    =
    (I-S_n^0)\mathbf{g}_n
    -
    S_n^0\bv_n^\star
    +
    (S_n^0-S_n)\bx_n .
\end{equation}
Taking norms and dividing by $\sqrt n$ gives
\begin{align}
    n^{-1/2}\|\hat{\bv}_n-\bv_n^\star\|_2
    &\le
    n^{-1/2}\|(I-S_n^0)\mathbf{g}_n\|_2
    +
    n^{-1/2}\|S_n^0\bv_n^\star\|_2  \\
    &\quad+
    n^{-1/2}\|(S_n^0-S_n)\bx_n\|_2 .
\end{align}
The first term is controlled by \eqref{eq:oracle_graph_bias}, the second by \eqref{eq:residual_spectral_rate}, and the third by Lemma \ref{lem:resolvent_leakage}. This proves \eqref{eq:explicit_residual_rate}. The consistency statement follows from the displayed rate. Balancing $t_n^{s/2}$ with $n^{-1/2}t_n^{-m/4}$ gives $t_n\asymp n^{-2/(2s+m)}$ and rate $n^{-s/(2s+m)}$.
\end{proof}

\subsection{Proof of Corollary \ref{cor:explicit_beta_rate}}

\begin{proof}
Theorem \ref{thm:aihf_rate} gives
\begin{equation}
    n^{-1/2}\|\hat{\bv}_n-\bv_n^\star\|_2=O_p(r_n).
\end{equation}
Also,
\begin{equation}
    \begin{aligned}
    n^{-1/2}\|\hat{\bv}_n-\bu_n\|_2
    &\le n^{-1/2}\|\hat{\bv}_n-\bv_n^\star\|_2\\
    &\quad+n^{-1/2}\|\bv_n^\star-\bu_n\|_2\\
    &=O_p(r_n+\delta_{x,n}).
    \end{aligned}
\end{equation}
On the event $\kappa_n(\hat{\bv}_n)\ge\kappa_0$, Proposition \ref{prop:beta} bounds the generated-control contribution by a constant multiple of this term. This gives \eqref{eq:explicit_beta_centered_rate}. The last display follows when $\mathcal E_{\epsilon,n}=O_p(n^{-1/2})$.
\end{proof}

\subsection{Proof of Proposition \ref{prop:selector_oracle}}

\begin{proof}
Let $\hat h=\hat h_{\mathrm{obs}}$. On the event \eqref{eq:selector_uniform}, for any $h\in\Hcal_n^\kappa$,
\begin{align}
    R_n(\hat h)+\alpha J_n(\hat h)+\sigma_V^2
    &\le
    \widehat Q_{\mathrm{obs}}(\hat h)+\Delta_n \\
    &\le
    \widehat Q_{\mathrm{obs}}(h)+\Delta_n \\
    &\le
    R_n(h)+\alpha J_n(h)+\sigma_V^2+2\Delta_n .
\end{align}
Taking the infimum over $h\in\Hcal_n^\kappa$ and dropping the nonnegative term $\alpha J_n(\hat h)$ gives \eqref{eq:selector_oracle}.

It remains to justify \eqref{eq:selector_delta_explicit}. Work conditional on the graph-construction split and on $Z$, so the matrices $S_h$ are fixed. Set
\begin{equation}
    A_h=\frac{(I-S_h)^2}{(1-d_h)^2}.
\end{equation}
The degrees-of-freedom condition \eqref{eq:df_lower_bound} gives
\begin{equation}
    \opnorm{A_h}\le c_{\mathrm{df}}^{-2},
    \qquad
    \|A_h\|_{\mathrm F}\le c_{\mathrm{df}}^{-2}\sqrt n .
    \label{eq:Ah_bounds}
\end{equation}
Since $0\preceq S_h\preceq I$, the same type of bounds hold for $S_h^2$ and for $A_h-S_h^2$.

Write $\bx=\mathbf{g}+\boldsymbol{\eta}$. The part of
\[
    \widehat Q_{\mathrm{obs}}(h)
    -
    \{R_n(h)+\alpha J_n(h)+\sigma_V^2\}
\]
that is random after centering is
\begin{equation}
    \frac{1}{n}
    \left\{
    \boldsymbol{\eta}^\top B_h\boldsymbol{\eta}
    -
    \E[\boldsymbol{\eta}^\top B_h\boldsymbol{\eta}\mid Z]
    \right\}
    +
    \frac{2}{n}a_h^\top\boldsymbol{\eta},
    \label{eq:selector_random_decomp}
\end{equation}
where
\begin{equation}
    B_h=A_h-S_h^2,
    \qquad
    a_h=
    \left\{A_h+(I-S_h)S_h\right\}\mathbf{g}.
\end{equation}
The deterministic remainder is included in $b_{\mathrm{sel},n}$. The roughness term $J_n(h)$ cancels because it appears in the empirical criterion and in the comparison target with the same coefficient $\alpha$.

Define the centered quadratic form
\begin{equation}
    Q_h=
    \boldsymbol{\eta}^\top B_h\boldsymbol{\eta}
    -
    \E[\boldsymbol{\eta}^\top B_h\boldsymbol{\eta}\mid Z].
\end{equation}
By the sub-Gaussian quadratic-form inequality, for every $u>0$,
\begin{equation}
    \Pbb\left(
    |Q_h|>
    C\{\|B_h\|_{\mathrm F}\sqrt u+\opnorm{B_h}u\}
    \,\middle|\, Z
    \right)
    \le 2e^{-u}.
\end{equation}
Using \eqref{eq:Ah_bounds} gives a contribution bounded by
\begin{equation}
    C\left\{\sqrt{\frac{u}{n}}+\frac{u}{n}\right\}.
\end{equation}
The linear term satisfies
\begin{equation}
    \|a_h\|_2
    \le
    C\|\mathbf{g}\|_2
    \le
    C\sqrt n,
\end{equation}
so the sub-Gaussian linear-form bound gives
\begin{equation}
    \Pbb\left(
    \frac{2}{n}|a_h^\top\boldsymbol{\eta}|
    >
    C\sqrt{\frac{u}{n}}
    \,\middle|\, Z
    \right)
    \le 2e^{-u}.
\end{equation}
Take $u=\log(2|\Hcal_n|/\delta)$ and apply a union bound over
$\Hcal_n^\kappa$. Adding $b_{\mathrm{sel},n}$ proves
\eqref{eq:selector_delta_explicit}. If the trace is estimated, then
$0\preceq S_h\preceq I$ gives
\begin{equation}
    \frac{1}{n}\|(I-S_h)\bx\|_2^2
    \le
    \frac{1}{n}\|\bx\|_2^2
    \le C_x .
\end{equation}
The map $d\mapsto(1-d)^{-2}$ is Lipschitz on
$d\le1-c_{\mathrm{df}}/2$. Hence the perturbation of the complete GCV term
from replacing $d_h$ by $\widehat d_h$ is bounded by
$C C_x e_{\mathrm{tr},n}$. Finally,
$n^{-1}\|\bx\|_2^2=O_p(1)$ follows from
$n^{-1}\|\mathbf g\|_2^2\le C_g$ and the stated sub-Gaussian noise bound.
\end{proof}

\subsection{Derivation for Remark \ref{rem:orthogonal_inference}}

\begin{proof}
Use a first-order expansion of the cross-fitted sample moment around $(\beta_0,\eta_0)$:
\begin{equation}
    0
    =
    \frac{1}{n}\sum_{i=1}^n\psi(W_i;\beta_0,\eta_0)
    +
    G(\tilde\beta-\beta_0)
    +
    R_n .
\end{equation}
Cross-fitting removes first-order empirical-process terms from estimating $\eta_0$. Neyman orthogonality makes the deterministic first-order nuisance derivative zero. The assumed product-rate condition gives $R_n=o_p(n^{-1/2})$. Hence
\begin{equation}
    \sqrt n(\tilde\beta-\beta_0)
    =
    -G^{-1}\frac{1}{\sqrt n}\sum_{i=1}^n\psi(W_i;\beta_0,\eta_0)
    +o_p(1).
\end{equation}
The central limit theorem gives \eqref{eq:orthogonal_clt}. Consistency of the sandwich estimator follows from cross-fitted plug-in consistency and the law of large numbers.
\end{proof}

\section{Additional Experimental Details}
\label{app:exp_details}

The synthetic first stage begins with a one-dimensional latent instrument,
which is embedded in $d_Z$ observed coordinates through sinusoidal random
projections and Gaussian noise. Thus the intrinsic dimension is $m=1$ even
when the ambient dimension is $d_Z\in\{5,20,50\}$. From this common geometry,
the fractured and multi-fracture designs introduce one or several jumps; the
smooth design removes them; the weak-instrument design scales down the
systematic first-stage component; the correlated-residual design draws the
residual from a Gaussian process on the latent instrument; and the
high-dimensional nuisance design adds irrelevant ambient variation. Together
these designs test whether the graph recovers a low-dimensional structure
embedded in noisier, higher-dimensional observations.

For A-IHF, the affinity graph is a symmetric $K$-nearest-neighbor graph with
RBF weights. The default fixed configuration is $K=15$, $\tau=2$,
$\lambda=30$, $p=80$, and cutoff $10^{-6}$. The observational rule uses
$\alpha=0.05$, $\varepsilon=10^{-8}$, $c_\kappa=0.02$, and Hutchinson trace
estimation with 16 probes. These values are fixed across designs and seeds.
The graph-admissibility-filtered observational variant, abbreviated as
guarded observational A-IHF in tables, filters candidate graphs using
first-stage diagnostics before applying the observational score; in the tuned
suite it requires non-degenerate final degrees, a largest post-cut component
fraction of at least $0.5$, and normalized edge contrast $a_n(h)\le1$.

\subsection{Baseline Protocol and Seed Variation}
\label{sec:baseline_protocol}

Table \ref{tab:baseline_protocol} records the main tuned-baseline protocol. For control-function rows in nonlinear synthetic tables, the second stage is the same additive neural regressor for all generated controls in a given experiment. It uses separate one-hidden-layer ELU components \citep{clevert2016elu} for $X$ and the generated control, each with hidden width 64. The network is trained by full-batch Adam \citep{kingma2015adam} with learning rate $0.01$ and weight decay $10^{-4}$. The main nonlinear benchmark and semi-synthetic experiments use 500 stage-2 epochs; other diagnostics keep the epoch count fixed within each experiment. The epoch count is not selected by method.

\begin{table*}[t]
\centering
\small
\caption{\textbf{Baseline protocol used for the reported 54-cell benchmark.} All tuning uses only $(Z,X)$. Supervised learners use global-CV-tuned out-of-fold residuals rather than nested cross-fitting; graph methods are transductive smoothers with complexity chosen by graph GCV or the A-IHF observational rule. This distinction is retained in the released result metadata.}
\label{tab:baseline_protocol}
\begingroup
\scriptsize
\setlength{\tabcolsep}{3.5pt}
\renewcommand{\arraystretch}{1.08}
\resizebox{\textwidth}{!}{%
\begin{tabular}{p{0.15\textwidth}p{0.26\textwidth}p{0.27\textwidth}p{0.25\textwidth}}
\toprule
\textbf{Method} & \textbf{First-stage object} & \textbf{Tuning rule} & \textbf{Generated-residual protocol} \\
\midrule
Graph ridge CF & Isotropic graph-resolvent estimate of $g$ & Fixed default or graph GCV over $K$ and $\lambda$ & Transductive, full observed graph \\
Graph spectral CF & Truncated graph spectral estimate of $g$ & Graph GCV over spectral rank & Transductive, full observed graph \\
Series/KRR CF & Series or RBF-kernel estimate of $g$ & Five-fold first-stage CV & Five-fold out-of-fold predictions \\
RF/HistGBDT/XGBoost CF & Tree or boosting estimate of $g$ & Five-fold first-stage CV & Five-fold out-of-fold predictions \\
Fixed/CV Deep CF & ELU-network estimate of $g$ & Fixed regularization or held-out prediction loss & Five-fold out-of-fold predictions \\
Deep ensemble CF & Mean of three regularized neural estimates of $g$ & Fixed architecture and seeds & Mean of out-of-fold predictions \\
A-IHF & Pilot-adaptive anisotropic graph-resolvent estimate of $g$ & Fixed, observational, or guarded observational rule & Transductive, full observed graph \\
\bottomrule
\end{tabular}
}
\endgroup
\end{table*}

Table \ref{tab:main_seed_variation} reports seed variation for the main fractured cell in Table \ref{tab:main}. Values are means with standard deviations across ten seeds.

\begin{table}[ht]
\centering
\small
\caption{\textbf{Seed variation in the main fractured design.} Values are mean (standard deviation) over 10 seeds for $n=800$, $d_Z=50$.}
\label{tab:main_seed_variation}
\begingroup
\setlength{\tabcolsep}{4pt}
\renewcommand{\arraystretch}{1.05}
\resizebox{\linewidth}{!}{%
\begin{tabular}{lcc}
\toprule
\textbf{Method} & \textbf{Control corr.} $\uparrow$ & \textbf{Structural-response MSE} $\downarrow$ \\
\midrule
A-IHF, guarded observational & $0.948\ (0.011)$ & $1.732\ (0.844)$ \\
A-IHF, observational & $0.948\ (0.011)$ & $1.732\ (0.844)$ \\
A-IHF, fixed & $0.953\ (0.014)$ & $1.810\ (0.733)$ \\
Graph ridge CF & $0.853\ (0.023)$ & $2.163\ (1.085)$ \\
Graph spectral CF & $0.858\ (0.018)$ & $2.179\ (1.117)$ \\
Random forest CF & $0.827\ (0.025)$ & $2.277\ (1.151)$ \\
XGBoost CF & $0.793\ (0.025)$ & $2.603\ (1.296)$ \\
HistGBDT CF & $0.770\ (0.027)$ & $2.822\ (1.330)$ \\
Kernel ridge CF & $0.750\ (0.026)$ & $2.883\ (1.389)$ \\
Graph ridge CF, fixed & $0.727\ (0.037)$ & $2.948\ (1.266)$ \\
CV-tuned Deep CF & $0.764\ (0.031)$ & $3.153\ (1.625)$ \\
Series CF & $0.751\ (0.019)$ & $3.205\ (1.420)$ \\
Deep ensemble CF & $0.715\ (0.031)$ & $3.589\ (1.668)$ \\
Fixed Deep CF & $0.701\ (0.032)$ & $3.860\ (1.711)$ \\
\bottomrule
\end{tabular}
}
\endgroup
\end{table}

The linear alignment experiment keeps the same first-stage designs but uses
$Y_{\mathrm{lin}}=X+2.5U+\epsilon$ and ordinary least squares on
$(1,X,\hat V)$ in the second stage. It is designed around the coefficient
perturbation bound in Corollary \ref{cor:graph_causal}. The score-interval
diagnostic uses the same outcome and cross-fits only the linear second-stage
nuisance regressions. Because $Y_{\mathrm{lin}}$ and $X$ are residualized on
a generated control computed before those folds, the exercise does not
implement the fully cross-fitted construction of Remark
\ref{rem:orthogonal_inference}.

The split-selection check covers five designs at $n=800$, $d_Z=50$:
fractured, smooth, weak instrument, correlated residual, and
high-dimensional nuisance. It selects parameters on one split and refits the
transductive graph smoother on the full sample. This is a parameter-transfer
stability check, not the estimator of Proposition
\ref{prop:selector_oracle}. On the same five-design grid, the ablation study
compares fixed anisotropic A-IHF with isotropic diffusion, removal of the
pilot or hard cutoff, the full observational rule, removal of the relevance
screen, and removal of the roughness penalty. The meta-parameter study varies
$\alpha$ and $c_\kappa$ while holding the candidate graph family fixed. The
representation study rebuilds the graph from the observed first-stage
features, the latent instrument, principal components, and noisy variants.

For the matched projective audit, every regression vector is centered before
projection.  A nonzero generated-control direction is normalized before
forming the rank-one projector, which makes the numerical implementation
stable under extreme rescaling.  The audit records the untruncated slack
$p_n-F_{\rho_n}(q_n)$ and a separate nonnegative violation field; the largest
violation over 1,620 method-runs is zero at double precision, and no generated
control is degenerate.  The scale intervention produces 2,400 rows:
60 DGP--seed files, four controls including oracle $U$, five positive scales,
and raw versus canonicalized inputs.  It uses the same full-batch Adam
configuration as the main second stage.  Positive rescaling is not treated as
a hyperparameter and no outcome-based scale is selected.

The graph-certificate experiment reports structural leakage, residual attenuation, first-stage noise mismatch, and residualized treatment variation as defined in Proposition \ref{prop:diagnostic_certificate}. These terms use simulator information and are hidden in observational applications. The boundary stress tests use five designs, $n=800$, $d_Z=50$, and ten seeds. The instrument-strength path uses $\{0.2,0.4,0.6,0.8,1.0\}$. The irrelevant-dimension path appends \{$0$, $25$, $50$, $100$, $200$\} independent Gaussian coordinates to the observed first-stage feature representation. The representation-noise path adds isotropic Gaussian noise with scale $\{0,0.25,0.5,1,2\}$ and rescales the representation. All boundary tests use the same fixed and observational A-IHF configurations as the main $n=800$, $d_Z=50$ experiment.

The semi-synthetic real-$Z$ experiment uses covariates from the \texttt{diabetes}, \texttt{breast\_cancer}, and \texttt{digits} data sets distributed with \texttt{scikit-learn} \citep{pedregosa2011scikit}. Covariates are standardized before graph construction. Each run samples $n=400$ observations with replacement when needed. The simulated first stage is built from low-dimensional scores of the real covariate matrix: the fractured design applies a discontinuous transformation, the smooth design removes the discontinuity, and the weak-instrument design scales down the first-stage systematic component. The nonlinear response uses $U$ as the control component. The linear coefficient experiment in Table \ref{tab:realz_overall} uses $Y_{\mathrm{lin}}=\beta_0X+\gamma_0V^\star+\epsilon$. This design preserves access to $U$, $V^\star$, and $f_0$ for evaluation while replacing the synthetic first-stage feature cloud by empirical tabular geometry.

The four real-IV applications use publicly distributed benchmark data and pre-specified variable mappings. Card schooling uses \texttt{lwage} as the outcome, \texttt{educ} as treatment, \texttt{nearc4} as the excluded instrument, and experience, demographic, location, and region controls \citep{card1995using}. For the working-women subsample of the Mroz data, we follow the standard Wooldridge wage--education IV teaching specification: \texttt{lwage} is the outcome, \texttt{educ} is the treatment, parental education variables are the excluded instruments, and experience terms are controls \citep{mroz1987sensitivity,wooldridge2020introductory}. The 1995 Cigarette Demand cross-section uses log packs as outcome, log real price as treatment, real tax measures as instruments, and log real income as a control \citep{stock2007introduction}. The Social Insurance application uses the \texttt{causaldata} teaching extract associated with \citet{huntingtonklein2022effect}: insurance take-up is the outcome, prior friends' purchase behavior (\path{pre_takeup_rate}) is the treatment, the randomized first-round default condition is the instrument, and demographic plus village variables are controls \citep{cai2015social,huntingtonklein2022effect}. Continuous features are standardized using each training fold. Two-stage least squares is reported as a conventional reference. Ridge, graph ridge, raw A-IHF, and guarded A-IHF form out-of-fold generated controls for the linear control-function second stage. The guard uses the candidate-specific $K$ when evaluating \eqref{eq:edge_compatibility}; an inadmissible fold falls back to pre-specified ridge. Percentile intervals use 200 full-procedure bootstrap resamples: Social Insurance resamples natural-village clusters, and the other applications resample observations.

The edge-corruption calibration uses a one-dimensional piecewise-smooth first stage with $n=300$. A symmetric 15-nearest-neighbor graph is progressively rewired by replacing a pre-specified fraction of within-region edges with equally weighted cross-boundary edges. The corruption levels are
\begin{equation*}
\rho\in\{0,.10,.25,.40,.45,.50,.55,.60,.80,1\}.
\end{equation*}
Each of 50 Monte Carlo replications uses five-fold first-stage fitting. The graph construction, candidate selection, compatibility statistic, and ridge fallback use only the training portion of each fold. The experiment stores all 500 run-level records, the ten-level summary, and deterministic validation metadata.

The component-runtime experiment uses the fixed A-IHF configuration at $n=800$, $d_Z=50$ and decomposes first-stage graph time into affinity construction, pilot solve, conductance weighting, final solve, and Hutchinson trace estimation. The approximate-solver experiment uses the same A-IHF graph construction as the fixed configuration and compares exact sparse solves with conjugate-gradient solves for the final resolvent. The reported $\delta_n$ is computed relative to the exact sparse solution for the same graph. Solve times are wall-clock times for the final linear system only. Graph construction, pilot smoothing, trace estimation, and downstream regression are not included in those timings.

\subsection{Orthogonal Linear-Inference Check}
\label{sec:orthogonal_diag}

Remark \ref{rem:orthogonal_inference} gives a root-$n$ expansion under product-rate nuisance conditions, including cross-fitting of the generated-control nuisance. We report a deliberately weaker finite-sample score-interval diagnostic with the same linear outcome as Section \ref{sec:linear_alignment}. For each already-computed generated control, five-fold cross-fitting residualizes $Y_{\mathrm{lin}}$ and $X$ on that scalar control by a linear nuisance regression. Because the transductive generated control is computed before the folds, this diagnostic does not satisfy the premise of Remark \ref{rem:orthogonal_inference} and is not presented as a coverage guarantee. The interval uses the empirical score variance. The four real-IV applications instead use the separate bootstrap protocol described above.

\begin{table}[ht]
\centering
\small
\caption{\textbf{Orthogonal linear-inference check.} Entries report mean absolute error and empirical 95\% coverage over 10 seeds for $n=800$, $d_Z=50$. Oracle $U$ uses hidden simulator information and is included only as a reference.}
\label{tab:orthogonal_diag}
\begingroup
\setlength{\tabcolsep}{3pt}
\renewcommand{\arraystretch}{1.05}
\resizebox{\linewidth}{!}{%
\begin{tabular}{lcccccccc}
\toprule
\textbf{Design} &
\multicolumn{2}{c}{\textbf{Oracle $U$}} &
\multicolumn{2}{c}{\textbf{A-IHF obs.}} &
\multicolumn{2}{c}{\textbf{A-IHF fixed}} &
\multicolumn{2}{c}{\textbf{Linear CF}} \\
\cmidrule(lr){2-3}\cmidrule(lr){4-5}\cmidrule(lr){6-7}\cmidrule(lr){8-9}
& \textbf{Abs. error} & \textbf{Cov.} & \textbf{Abs. error} & \textbf{Cov.} & \textbf{Abs. error} & \textbf{Cov.} & \textbf{Abs. error} & \textbf{Cov.} \\
\midrule
Fractured & 0.003 & 1.00 & 0.030 & 0.60 & 0.053 & 0.30 & 0.047 & 0.60 \\
Smooth & 0.021 & 0.80 & 0.089 & 0.40 & 0.162 & 0.10 & 0.261 & 0.10 \\
Weak instrument & 0.008 & 1.00 & 0.117 & 0.70 & 0.270 & 0.10 & 0.276 & 0.00 \\
Correlated residual & 0.004 & 1.00 & 0.410 & 0.10 & 0.399 & 0.10 & 0.449 & 0.00 \\
High-dimensional nuisance & 0.004 & 1.00 & 0.252 & 0.00 & 0.420 & 0.00 & 0.100 & 0.40 \\
\midrule
Average & 0.008 & 0.96 & 0.180 & 0.36 & 0.261 & 0.12 & 0.227 & 0.22 \\
\bottomrule
\end{tabular}
}
\endgroup
\end{table}

The oracle row has small error and coverage close to the nominal level. Observational A-IHF reduces average absolute error relative to fixed A-IHF and to the linear control function. Coverage remains far below nominal in several designs, including correlated-residual and high-dimensional nuisance designs. This undercoverage is expected because the generated control is not fold-specific and the required generated-control and product-rate conditions are not established. All synthetic designs have latent dimension $m=1$, so this table does not diagnose the high-intrinsic-dimension restriction in Remark \ref{rem:dimension_inference}. It instead documents that cross-fitting only the second-stage regressions is insufficient for valid inference with a transductive generated control.

\subsection{Meta-Parameter Sensitivity}
\label{sec:meta_sensitivity_app}

The observational score contains two fixed scale constants, $\alpha$ and $c_\kappa$. We varied
\begin{equation}
    \begin{aligned}
    \alpha&\in\{0,0.01,0.05,0.10,0.20\},\\
    c_\kappa&\in\{0.005,0.01,0.02,0.05\}.
    \end{aligned}
\end{equation}
Table \ref{tab:meta_sensitivity} reports the fixed A-IHF default, the default observational rule, and the range of mean structural-response MSEs across the $20$ meta-parameter pairs. The candidate graph family is the same as in the main benchmark.

\begin{table}[ht]
\centering
\small
\caption{\textbf{Meta-parameter sensitivity.} Mean structural-response MSE over 10 seeds for $n=800$, $d_Z=50$. The last column gives the range across all tested $(\alpha,c_\kappa)$ pairs.}
\label{tab:meta_sensitivity}
\begingroup
\setlength{\tabcolsep}{4pt}
\renewcommand{\arraystretch}{1.05}
\resizebox{\linewidth}{!}{%
\begin{tabular}{lccc}
\toprule
\textbf{Design} & \textbf{A-IHF fixed} & \textbf{A-IHF obs. default} & \textbf{A-IHF obs. range} \\
\midrule
Fractured & 1.804 & \textbf{1.710} & 1.693--1.891 \\
High-dimensional nuisance & 2.793 & \textbf{2.240} & 2.186--2.743 \\
Smooth & \textbf{0.993} & 1.056 & 1.032--1.509 \\
Correlated residual & \textbf{5.042} & 5.312 & 4.911--5.532 \\
\bottomrule
\end{tabular}
}
\endgroup
\end{table}

The selected graph changes little across this experiment. The rule selects $p=70$ in every run. It selects $K=10$ in all high-dimensional nuisance and correlated-residual runs, and in most fractured and smooth runs. The main behavior is driven more by the candidate graph geometry than by fine adjustment of $\alpha$ or $c_\kappa$.

\subsection{Component Runtime Check}
\label{sec:runtime_diag}

We measure the main A-IHF graph components at $n=800$, $d_Z=50$, using the fixed configuration. The timings exclude downstream outcome regression and are reported only for the first-stage graph pipeline.

\begin{table}[ht]
\centering
\small
\caption{\textbf{Component runtime decomposition.} Mean over 10 seeds for fixed A-IHF at $n=800$, $d_Z=50$. Times are seconds. Percentages are shares of the listed component total.}
\label{tab:runtime_components}
\begingroup
\setlength{\tabcolsep}{4pt}
\renewcommand{\arraystretch}{1.05}
\resizebox{\linewidth}{!}{%
\begin{tabular}{lcccccc}
\toprule
\textbf{Design} & \textbf{Total} & \textbf{Graph} & \textbf{Pilot solve} & \textbf{Conductance} & \textbf{Final solve} & \textbf{Trace} \\
\midrule
Fractured & 0.0789 & 0.0469 (59.4\%) & 0.0040 (5.1\%) & 0.0023 (2.9\%) & 0.0036 (4.5\%) & 0.0222 (28.1\%) \\
High-dimensional nuisance & 0.1405 & 0.0406 (28.9\%) & 0.0110 (7.8\%) & 0.0023 (1.6\%) & 0.0109 (7.8\%) & 0.0758 (53.9\%) \\
Smooth & 0.0746 & 0.0409 (54.8\%) & 0.0038 (5.1\%) & 0.0020 (2.7\%) & 0.0043 (5.8\%) & 0.0235 (31.6\%) \\
Weak instrument & 0.0764 & 0.0448 (58.7\%) & 0.0038 (5.0\%) & 0.0020 (2.7\%) & 0.0036 (4.8\%) & 0.0220 (28.8\%) \\
Correlated residual & 0.0679 & 0.0429 (63.2\%) & 0.0043 (6.3\%) & 0.0020 (2.9\%) & 0.0037 (5.5\%) & 0.0150 (22.1\%) \\
\bottomrule
\end{tabular}
}
\endgroup
\end{table}

Table \ref{tab:runtime_components} decomposes runtime in the tested regime. Graph construction and Hutchinson trace estimation account for most of the fixed A-IHF time. The pilot and final sparse solves are small at this sample size. The large-$n$ solver experiment below examines only the final resolvent.

\subsection{Approximate Resolvent Check}
\label{sec:solver_diag}

The exact A-IHF implementation solves sparse linear systems for the pilot and final graph resolvents. Proposition \ref{prop:approx_solve} shows that an approximate solve affects the generated control through the empirical error
\begin{equation}
    \delta_n=n^{-1/2}\|(\tilde S_h-S_h)\bx\|_2 .
\end{equation}
We evaluate conjugate-gradient solves for the final resolvent on three designs with $d_Z=50$, $n\in\{800,3000,10000\}$, and five seeds. Table \ref{tab:solver} reports the largest sample size and tolerance $10^{-4}$. The times are solve-component times for the final resolvent, not end-to-end pipeline times.

\begin{table}[ht]
\centering
\small
\caption{\textbf{Approximate resolvent solves.} Mean over five seeds at $n=10000$, $d_Z=50$. CG uses relative tolerance $10^{-4}$.}
\label{tab:solver}
\begingroup
\setlength{\tabcolsep}{4pt}
\renewcommand{\arraystretch}{1.05}
\resizebox{\linewidth}{!}{%
\begin{tabular}{lccccc}
\toprule
\textbf{Design} & \textbf{Exact corr.} & \textbf{CG corr.} & $\boldsymbol{\delta_n}$ & \textbf{Exact time} & \textbf{CG time} \\
\midrule
Fractured & 0.96290 & 0.96290 & $4.36\times 10^{-5}$ & 1.115s & 0.0145s \\
Smooth & 0.98487 & 0.98487 & $2.79\times 10^{-5}$ & 1.139s & 0.0150s \\
High-dimensional nuisance & 0.69430 & 0.69430 & $3.32\times 10^{-5}$ & 6.867s & 0.0164s \\
\bottomrule
\end{tabular}
}
\endgroup
\end{table}

The generated residual is unchanged at the reported precision. Jacobi-preconditioned CG gives the same control correlations at the reported precision and smaller solve times in this experiment. Approximate resolvent solves can be used in larger graph instances when $\delta_n$ and $\kappa_n(\hat{\bv})$ are monitored.

\begin{figure}[ht]
\centering
\includegraphics[width=0.88\linewidth]{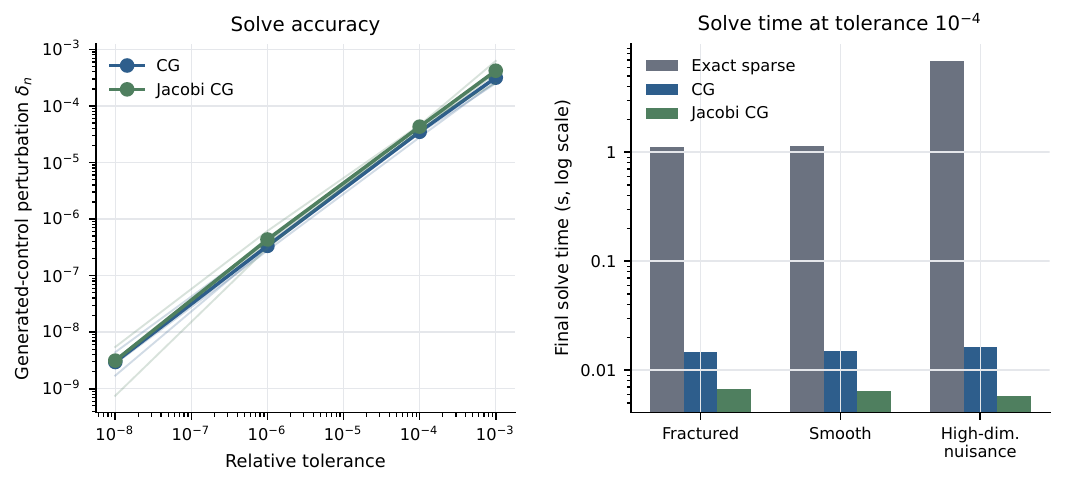}
\caption{\textbf{Approximate final-resolvent solves at $n=10000$.} The left panel reports the generated-control perturbation $\delta_n$ across iterative-solve tolerances. The right panel compares exact sparse, CG, and Jacobi-preconditioned CG solve times for the final resolvent. These are solve-component times, not end-to-end pipeline times.}
\label{fig:solver_diagnostic}
\end{figure}

\subsection{Evaluation Grid Note}
\label{sec:evaluation_grid_note}

The main benchmark evaluates the structural response on the empirical support of $X$ in each run. As $n$ changes, this support can widen. Cross-$n$ MSE comparisons combine estimation effects with support-width effects and are not used as empirical evidence for Proposition $\ref{prop:consistency_reduction}$. The consistency statement is a graph-sequence reduction, not a monotonicity claim about these finite benchmark tables.

\fi

\clearpage

\begin{center}
{\Large\bfseries Supplementary Material\par}
\vspace{0.5em}
{\large Learning Generated Controls under Fractured Geometry:
Projective Residualization and Variation-Allocation Frontiers\par}
\end{center}

\vspace{1em}

The following appendices provide the proofs, additional experimental details,
inference diagnostics, sensitivity analyses, runtime decomposition, and
approximate-solver study.

\appendix
\supplementtrue

\bibliographystyle{elsarticle-num-names}
\bibliography{aihf_references}

\end{document}